\documentclass{article}
\usepackage[preprint]{arxiv}
\usepackage[utf8]{inputenc} 
\usepackage[T1]{fontenc}    
\usepackage{url}            
\usepackage{booktabs}       
\usepackage{amsfonts}       
\usepackage{nicefrac}       
\usepackage{microtype}      
\usepackage{xcolor}         
\usepackage{wrapfig}
\usepackage{graphicx}
\usepackage[algo2e]{algorithm2e}
\usepackage{adjustbox}
\usepackage{comment}
\usepackage{enumitem}
\setlist{leftmargin=3mm}
\usepackage{multirow}
\usepackage{tcolorbox}
\tcbuselibrary{skins,breakable}
\usepackage{tikz}
\usetikzlibrary{calc,arrows.meta,decorations.pathreplacing,positioning,backgrounds,patterns,shapes.geometric}
\usepackage{listings}

\usepackage{amsmath,amsthm,amssymb,amsfonts}
\usepackage{algorithm}
\usepackage{algorithmic}

\usepackage{comment}
\usepackage{bm}
\usepackage{pifont}

\theoremstyle{plain}
\newtheorem{theorem}{Theorem}[]
\newtheorem{corollary}{Corollary}[]
\newtheorem{lemma}[]{Lemma}
\newtheorem{proposition}{Proposition}

\newtheorem{definition}{Definition}[]
\newtheorem{assumption}{Assumption}[]

\newtheorem{remark}{Remark}[]
\newtheorem{example}{Example}[]

\def\1{\mathbf{1}}

\def\eps{{\epsilon}}

\def\vzero{{\mathbf{0}}}
\def\vone{{\mathbf{1}}}
\def\vmu{{\mathbf{\mu}}}

\def\vb{{\mathbf{b}}}

\def\ve{{\mathbf{e}}}

\def\vt{{\mathbf{t}}}
\def\vu{{\mathbf{u}}}

\def\vy{{\mathbf{y}}}
\def\vz{{\mathbf{z}}}

\def\mI{{\mathbf{I}}}

\def\mU{{\mathbf{U}}}

\def\mSigma{{\mathbf{\Sigma}}}

\DeclareMathAlphabet{\mathsfit}{\encodingdefault}{\sfdefault}{m}{sl}
\SetMathAlphabet{\mathsfit}{bold}{\encodingdefault}{\sfdefault}{bx}{n}

\def\gA{{\mathcal{A}}}

\def\gD{{\mathcal{D}}}
\def\gE{{\mathcal{E}}}
\def\gF{{\mathcal{F}}}

\def\gJ{{\mathcal{J}}}

\def\gL{{\mathcal{L}}}
\def\gM{{\mathcal{M}}}
\def\gN{{\mathcal{N}}}

\def\gP{{\mathcal{P}}}

\def\gR{{\mathcal{R}}}

\def\gX{{\mathcal{X}}}
\def\gY{{\mathcal{Y}}}

\newcommand{\E}{\mathbb{E}}

\newcommand{\R}{\mathbb{R}}

\newcommand{\Var}{\mathrm{Var}}

\newcommand{\Cov}{\mathrm{Cov}}

\DeclareMathOperator*{\argmin}{arg\,min}

\usepackage{color, colortbl}
\usepackage{bbm}
\usepackage{bm}
\usepackage{hyperref}
\usepackage[capitalize,noabbrev]{cleveref}

\lstdefinestyle{mystyle}{
    backgroundcolor=\color{backcolour},   
    commentstyle=\color{codegreen},
    keywordstyle=\color{magenta},
    numberstyle=\tiny\color{codegray},
    stringstyle=\color{codepurple},
    basicstyle=\ttfamily\footnotesize,
    breakatwhitespace=false,         
    breaklines=true,                 
    captionpos=b,                    
    keepspaces=true,                 
    numbers=left,                    
    numbersep=5pt,                  
    showspaces=false,                
    showstringspaces=false,
    showtabs=false,                  
    tabsize=2,
    frame=single,
}
\usepackage[colorinlistoftodos,prependcaption]{todonotes}
\usepackage{soul}
\definecolor{LightCyan}{rgb}{0.9,1,1}
\definecolor{light-gray}{gray}{0.8}
\definecolor{codegreen}{rgb}{0,0.6,0}
\definecolor{codegray}{rgb}{0.5,0.5,0.5}
\definecolor{codepurple}{rgb}{0.58,0,0.82}
\definecolor{backcolour}{rgb}{0.95,0.95,0.92}

\PassOptionsToPackage{cmyk}{xcolor}
\usepackage{csquotes}
\usepackage{adjustbox}
\usepackage{lscape}
\usepackage{subcaption}

\Crefname{assumption}{Assumption}{Assumptions}

\definecolor{abstractboxbg}{gray}{0.94}
\newtcolorbox{abstractbox}{
  enhanced,
  colback=abstractboxbg,
  colframe=abstractboxbg,
  boxrule=0pt,
  arc=3mm,
  left=6mm, right=6mm, top=5mm, bottom=5mm,
  width=\textwidth,
}
\newenvironment{awsabstract}
  {\begin{center}\begin{abstractbox}%
   {\noindent\bfseries Abstract.}\par\vspace{0.4em}\noindent\ignorespaces}
  {\par\vspace{1.0em}%
   \noindent
   \begin{minipage}[b]{0.70\linewidth}
     \textbf{Correspondence:}~\fbox{\texttt{achanish@amazon.com}}
   \end{minipage}\hfill
   \begin{minipage}[b]{0.26\linewidth}
     \raggedleft\includegraphics[height=26pt]{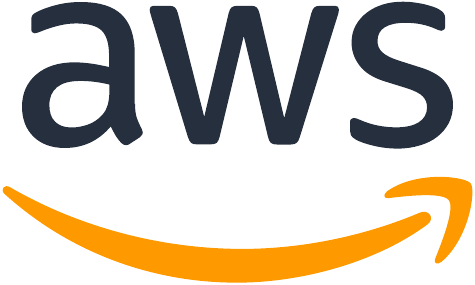}
   \end{minipage}%
   \end{abstractbox}\end{center}}

\title{\textsc{RoPoLL}: Robust Panel of LLM Judges}

\author{%
  Anish Acharya\thanks{Corresponding author.} \\
  Amazon Web Services \\
  \texttt{achanish@amazon.com} \\
  \And
  Kris W.\ Pan \\
  Amazon Web Services \\
  \texttt{kriswpan@amazon.com} \\
  \And
  Brian Verkhovsky \\
  Amazon Web Services \\
  \texttt{bverkhov@amazon.com} \\
}

\date{}

\begin{document}
\maketitle

\begin{awsabstract}
The LLM Jury, a \emph{Panel of LLM Evaluators}
(\textsc{PoLL})~\citep{verga2024replacing} reporting consensus
scores, has become a practical alternative to single judge LLM
evaluation, yet its statistical behaviour remains poorly
understood.
We formalize the LLM Jury setup under the Huber contamination model,
and show that \textsc{PoLL} incurs unbounded bias under any positive
contamination, regardless of jury size, whenever a single judge
fails in a biased, LLM-typical way (mode collapse, sycophancy,
safety refusal).
We frame the jury consensus problem as an instance of classical
robust mean estimation and
propose \textsc{RoPoLL} (\textbf{Ro}bust \textbf{P}anel
\textbf{o}f \textbf{LL}M-as-Judge), which preserves the
\textsc{PoLL} panel and substitutes the aggregation function with a robust mean estimator.
Among classical robust estimators, we instantiate \textsc{RoPoLL}
with the geometric median (GM), a tuning-free,
joint-distance-preserving mean estimator that yields the optimal finite-sample breakdown
point $1/2$.
We establish a finite-sample error bound and an
information-theoretic minimax lower bound that match on the
parametric rate $\sigma\sqrt{d/N}$ and differ on the breakdown
floor by a factor of $\sqrt{d}$ -- a
statistical-computational gap that polynomial-time \textsc{RoPoLL}
pays relative to the (intractable) Tukey halfspace median.
Across $13$ open-weight judges ($4$\,B--$675$\,B), three
reward-model benchmarks, and four corruption regimes at rates up
to $50\%$, \textsc{RoPoLL} dominates \textsc{PoLL} on every biased
corruption type: by $\approx 19\%$ on cross-dimensional attacks at
matched compute, and by orders of magnitude on heavy-tailed
Byzantine adversaries (whose unbounded first moments make any
breakdown-positive aggregator beat averaging unconditionally).
A $3$-judge \textsc{RoPoLL} committee at $38$\,B beats
Mistral-Large-3 ($675$\,B) by $1.31\times$ on HelpSteer-2 under
$30\%$ \texttt{bimodal-random} corruption -- an $18\times$ parameter
advantage with strictly better accuracy.
A Noisy-GT control confirms the premium is paid against
\emph{biased} contamination, not benign Gaussian imprecision (where
\textsc{PoLL} is statistically optimal).
Overall, we establish that robust aggregation of a small, diverse committee is a
parameter-efficient and statistically principled alternative to
scaling a single large LLM-as-judge.
\end{awsabstract}

\tableofcontents


\section{Introduction}
\label{sec:intro}

Reliable evaluation remains the bottleneck in aligning Large
Language Models (LLMs).
Human evaluation, while the gold standard, does not scale to the
iterative development cycles that modern alignment pipelines
demand.
The field has therefore converged on the \emph{LLM-as-a-Judge}
paradigm \citep{zheng2023judging}, in which another LLM
(typically a frontier model) acts as a referee, scoring outputs
along one or more quality attributes.
Subsequent work has trained open judges to match this behaviour
\citep{kim2024prometheus} and standardised rubric-based evaluation
protocols
\citep{li2023alpacaeval,dubois2024alpacafarm,ye2024flask}.
A single judge, however, is a single point of statistical
failure.
The systematic biases its backbone exhibits, e.g., position, verbosity,
self-enhancement, sycophancy, and refusal artefacts, are by now
well documented
\citep{wang2023large,panickssery2024llm,saito2023verbosity,stureborg2024large};
they propagate uncorrected to every score, and the cost-quality
profile of the resulting evaluation is fixed to that of the single
model.

A natural remedy is to evaluate by committee.
The \emph{LLM Jury}, instantiated by the \emph{Panel of LLM
Evaluators} (\textsc{PoLL}) of \citet{verga2024replacing},
ensembles smaller, diverse, cheaper backbones and reports the
arithmetic mean of their scores as the consensus---sufficient, in
their experiments, to match or exceed a single large judge.
Related multi-model evaluators include peer-rank discussion
\citep{li2024prd}, multi-agent debate \citep{chan2024chateval},
and deeper/wider judge networks \citep{zhang2024wider};
these vary the panel structure but inherit \textsc{PoLL}'s
aggregation rule.
\textsc{PoLL} is the optimal aggregator precisely when judge errors
are light-tailed and centered on the truth, in which case averaging
$N$ judges contracts the variance at the parametric rate $1/N$
(Proposition~\ref{prop:variance_reduction},
\S\ref{sec:variance_reduction});
\Cref{fig:p5_jury_vs_individual_clean} (\S\ref{sec:exp_clean_baseline})
shows the clean-baseline parameter-efficiency this delivers
empirically.

\paragraph{The problem: Byzantine failures, not Gaussian noise.}
Real LLM judges fail in ways that are nothing like Gaussian noise.
A judge that produces malformed JSON triggers a parser fallback
to the all-zeros score, dropping a single observation onto the
boundary of the score space.
A judge with sycophancy bias rates every response near the maximum,
flattening genuine quality differences.
A judge that handles one attribute well may catastrophically
mis-score another, producing a vector that is plausible per axis
yet jointly anomalous.
A judge whose parser hallucinates can emit values entirely outside
the bounded score scale.
These four failure modes---\emph{mode collapse}, \emph{sycophancy},
\emph{cross-attribute confusion}, and \emph{heavy-tailed
hallucination}---are all \emph{biased point masses far from the
truth}, not symmetric perturbations of it, and each occurs in real
deployments at non-trivial rates: in our corpus, parser-failure
alone reaches $33\%$ on the smallest judge (Gemma-4B) for
HelpSteer\,3 multilingual prompts, with mean rates of $3.4\%$ on
HelpSteer\,3 and $0.6\%$ on HelpSteer\,2 across the $13$-judge
panel (\Cref{fig:corpus_failure_rates}, \S\ref{sec:huber}).

This is the regime the classical robust-statistics literature
\citep{huber1964robust,tukey1960survey,small1990survey,vardi2000multivariate,
minsker2015geometric,lugosi2019sub}
and Byzantine-robust optimisation literature
\citep{blanchard2017machine,yin2018byzantine,elmhamdi2018hidden,
acharya2022robust,acharya2025geometric}
identify as the wrong regime for \textsc{PoLL}-style aggregation.
The Huber $\eps$-contamination model
(Assumption~\ref{asm:huber}) admits all four failure modes as
specific instantiations of the contamination distribution $Q_{i}$
(\texttt{zeros}, \texttt{inverted}, \texttt{bimodal-random}, and
\texttt{cauchy-far}; mapped explicitly in \S\ref{sec:huber} and
evaluated in \S\ref{sec:exp_cauchy}--\ref{sec:exp_bounded}), and
a direct calculation
(Proposition~\ref{prop:mean_bias}, \S\ref{sec:mean_fragility})
shows that under \emph{any} positive contamination rate
\textsc{PoLL}'s conditional bias grows linearly with the corruption
shift and is unbounded over the corruption class, regardless of
$N$:
the $1/N$ variance reduction that motivates juries cannot rescue
an aggregator whose bias is itself unbounded.

\begin{figure*}[t]
\centering
\begin{subfigure}{0.32\textwidth}
    \centering
    \includegraphics[width=\textwidth]{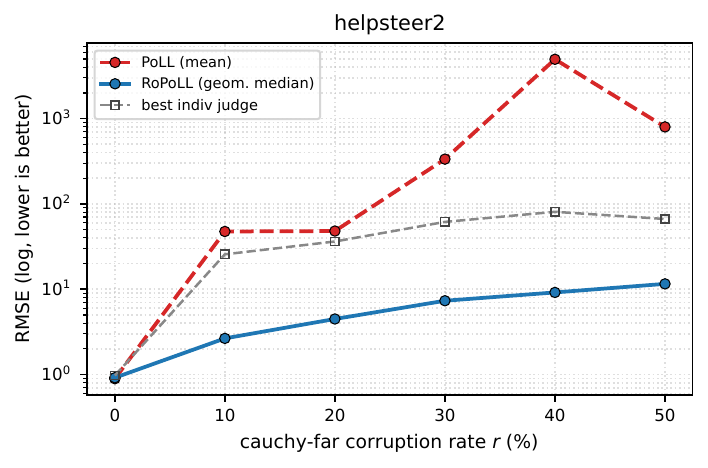}
    \caption{HelpSteer\,2}
\end{subfigure}\hfill
\begin{subfigure}{0.32\textwidth}
    \centering
    \includegraphics[width=\textwidth]{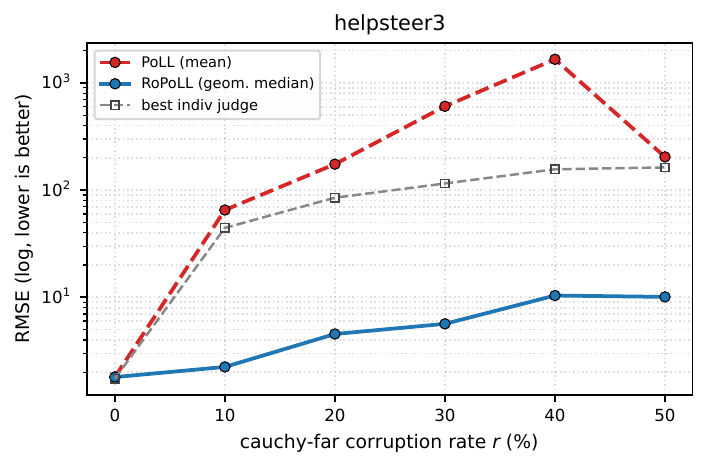}
    \caption{HelpSteer\,3}
\end{subfigure}\hfill
\begin{subfigure}{0.32\textwidth}
    \centering
    \includegraphics[width=\textwidth]{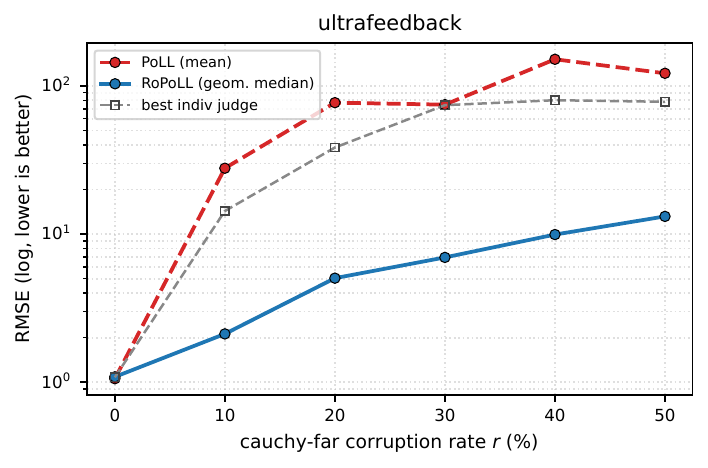}
    \caption{UltraFeedback}
\end{subfigure}
\caption{\textbf{\textsc{PoLL} vs.\ \textsc{RoPoLL} under heavy-tailed
\texttt{cauchy-far} corruption.}
RMSE vs.\ per-case corruption rate $r$ (log $y$-axis) for the
\textsc{Medium} jury ($N{=}3$, ${\approx}89$\,B), with the best single
open-weight judge as a gray dashed reference; coordinate-wise
\textsc{Median} is competitive with \textsc{RoPoLL} here and is
omitted (full three-method comparison in
\Cref{fig:p7_method_contrast}). Each corrupted slot is drawn as
$\hat{y} = y^{\star} + 10 + 2(s_{\max}{-}s_{\min})\,T$ with $T$
component-wise standard Cauchy: a biased heavy-tailed Byzantine
attack with undefined mean and variance, instantiating the
adversarial choice in Proposition~\ref{prop:mean_bias}.}
\label{fig:p1_cauchy_hero}
\end{figure*}

\paragraph{Overview of our approach.}
We propose \textsc{RoPoLL} (\textbf{Ro}bust \textbf{P}anel \textbf{o}f
\textbf{LL}M-as-Judge), a drop-in replacement for the
arithmetic-mean aggregation step of \textsc{PoLL} with a robust
mean estimator.
Among classical candidates---the coordinate-wise median (CoMed),
the trimmed mean, and the geometric median (GM)---only GM is
simultaneously \emph{tuning-free} (no contamination-rate hyperparameter
unlike the trimmed mean), \emph{joint-distance preserving} (operates
on Euclidean distance over the full score vector unlike CoMed,
which decouples coordinates and misses cross-attribute structure
of Example~\ref{ex:cross_dim}), and attains the optimal $1/2$
breakdown point (Definition~\ref{def:breakdown},
Proposition~\ref{prop:gm_properties});
the comparison is developed in detail in \S\ref{sec:choose_estimator}.
We instantiate \textsc{RoPoLL} with the geometric median
(Definition~\ref{def:gm}), computed via the modified Weiszfeld
iteration (Algorithm~\ref{alg:ropoll}, \S\ref{sec:weiszfeld}) at
$O(Nd\log(1/\eps))$ per query.
CoMed and the trimmed mean serve as empirical baselines in
\S\ref{sec:benchmark}.

\paragraph{Contributions.}
\begin{itemize}[leftmargin=*,topsep=2pt,itemsep=2pt]
    \item \textbf{Formalisation.}
    We give the first formal treatment of LLM jury aggregation as a
    robust mean-estimation problem
    (\S\ref{sec:problem_setup}):
    we model the LLM-as-Judge pipeline as a Markov kernel
    (Definition~\ref{def:judge}), define the LLM Jury
    (Definition~\ref{def:jury}), and characterise judge failures
    as Byzantine faults under the Huber contamination model
    (Assumption~\ref{asm:huber}).
    Proposition~\ref{prop:mean_bias} shows that \textsc{PoLL}
    admits unbounded bias under this model.

    \item \textbf{Algorithm and theory.}
    We propose \textsc{RoPoLL} (\S\ref{sec:methodology}) and
    establish its theoretical guarantees (\S\ref{sec:theory}):
    a finite-sample upper bound
    $\|\hat{\vy}_{\mathrm{GM}}-\vy^{\star}\|_{2} \leq C_{\alpha+\beta}\rho$
    with explicit absolute constants
    (Theorem~\ref{thm:ropoll_bound}),
    a correlated-jury extension under the equicorrelated-indicator
    hypothesis (Lemma~\ref{lem:correlated_jury}, with empirical
    indicator-correlation $\bar\gamma_{W} \in [0.45, 0.53]$
    measured on our judge panels,
    \Cref{sec:appendix_gamma_W}), and an information-theoretic
    minimax lower bound (Theorem~\ref{thm:minimax_body}) that
    matches on the parametric rate $\sigma\sqrt{d/N}$ and differs
    on the breakdown floor by a $\sqrt{d}$
    statistical--computational gap, attributed to GM's
    polynomial-time tractability relative to the (intractable)
    Tukey halfspace median.

    \item \textbf{Large-scale empirical validation.}
    We evaluate $13$ open-weight LLM judges spanning four model-size
    tiers ($4$\,B--$675$\,B parameters) on three benchmarks with
    complementary ground-truth sources: HelpSteer\,2
    \citep{wang2024helpsteer2}, HelpSteer\,3
    \citep{wang-etal-2025-helpsteer3}, and UltraFeedback
    \citep{cui2024ultrafeedback}.
    Under systematic adversarial injection at contamination rates
    up to $50\%$ (\S\ref{sec:benchmark}), \textsc{RoPoLL}
    outperforms \textsc{PoLL} by up to three orders of magnitude on
    biased heavy-tailed (\texttt{cauchy-far},
    \Cref{fig:p1_cauchy_hero}) and cross-dimensional
    (\texttt{bimodal-random},
    \Cref{fig:p2_bimodal_ropoll_hero}) attacks;
    a $3$-judge \textsc{RoPoLL} committee at $38$\,B total
    parameters beats Mistral-Large-3 ($675$\,B) by $1.31\times$
    on HelpSteer\,2 under $30\%$ \texttt{bimodal-random} corruption
    (an $18\times$ parameter advantage,
    \S\ref{sec:exp_bimodal}).
    A Noisy-GT control (\S\ref{sec:exp_noisy_gt}) rules out the
    obvious confound that the \textsc{RoPoLL} premium is paid
    against benign Gaussian imprecision rather than against biased
    contamination.

    \item \textbf{Open release of the judge-output corpus.}%
    {\renewcommand{\thefootnote}{$\dagger$}\footnote{Dataset released at \href{https://github.com/aws/RoPoLL}{https://github.com/aws/RoPoLL}.}}
    We release the full $13$-judge $\times$ three-benchmark output
    corpus---approximately $28\mathrm{K}$ scored
    $(\text{judge}, \text{sample})$ cells of parsed attribute scores,
    per-call latencies, and reference labels underlying every figure
    in \S\ref{sec:benchmark} (\S\ref{sec:exp_corpus}).
    To our knowledge this is the first standardised corpus of
    LLM-jury outputs;
    follow-up work on judge calibration, alternative aggregators,
    or new corruption families can be benchmarked against this
    fixed substrate without re-running the inference cost.
\end{itemize}

\paragraph{Paper organisation.}
\S\ref{sec:problem_setup} formalises the problem setup, including
the LLM Jury (Definitions~\ref{def:judge}--\ref{def:jury}), the
Huber contamination model (Assumption~\ref{asm:huber}, with the
empirical natural-failure-rate calibration of
\Cref{fig:corpus_failure_rates}), and the unbounded-bias result
for \textsc{PoLL} (Proposition~\ref{prop:mean_bias}).
\S\ref{sec:methodology} develops the \textsc{RoPoLL} methodology:
the choice of geometric median over coordinate-wise median and
trimmed mean (\S\ref{sec:choose_estimator}), structural
properties of GM (Proposition~\ref{prop:gm_properties}), and the
Weiszfeld iteration (Algorithm~\ref{alg:ropoll}).
\S\ref{sec:theory} states the finite-sample upper bound
(Theorem~\ref{thm:ropoll_bound}), its correlated-jury extension
(Lemma~\ref{lem:correlated_jury}), and the matching minimax lower
bound (Theorem~\ref{thm:minimax_body}).
\S\ref{sec:benchmark} presents the benchmark evaluation organised
by corruption type;
\S\ref{sec:related} situates the work in the LLM-as-Judge,
robust-statistics, and Byzantine-distributed-learning literatures;
\S\ref{sec:discussion} concludes with scope, limitations, and
follow-up directions.
The released corpus and its inter-judge correlation structure,
including the empirical $\bar\gamma_{W}$ measurement, are documented
in \S\ref{sec:benchmark}.
Full proofs and a 2D synthetic visualisation gallery are deferred to
Appendices~\ref{sec:appendix_theory}--\ref{sec:simulation}.

\section{Related Work}
\label{sec:related}

{\bf LLM-as-Judge evaluation and per-judge biases.}
The LLM-as-Judge paradigm was established by \citet{zheng2023judging}
(MT-Bench, Chatbot Arena), demonstrating that strong models such as
GPT-4 can serve as reliable proxies for human annotators.
Subsequent work has extended the paradigm along several axes:
open-source judges with fine-grained rubrics
\citep{kim2024prometheus};
automated frameworks for instruction-following models
\citep{li2023alpacaeval,dubois2024alpacafarm};
and skill-level evaluation \citep{ye2024flask}.
A parallel literature documents systematic biases of single
judges---position, verbosity, self-enhancement, sycophancy, and
prompt-format sensitivity
\citep{wang2023large,panickssery2024llm,saito2023verbosity,stureborg2024large}.
These findings motivate the use of diverse judge panels but treat
each judge in isolation;
no prior work analyzes the \emph{aggregation} step or its failure
modes.

{\bf Jury and panel evaluation.}
\citet{verga2024replacing} introduced the Panel of LLM Evaluators
(\textsc{PoLL}), our direct predecessor: a diverse committee of
smaller backbones aggregated by the arithmetic mean.
Their work established the practical value of LLM juries but did not
analyze robustness;
the mean aggregator is used without justification, and no failure
modes are considered.
\citet{zhang2024wider} studied how panel width and depth affect
evaluation fairness, again without robustness guarantees.
The key gap across this literature is the absence of any analysis of
catastrophic failure modes or formal robustness properties of the
aggregation rule.
Our Proposition~\ref{prop:mean_bias} closes this gap: under any
positive contamination rate \textsc{PoLL}
\citep{verga2024replacing} admits unbounded bias regardless of $N$.

{\bf Multi-agent debate and structured aggregation.}
A distinct family of multi-judge methods produces aggregated
judgments through structured \emph{interaction} rather than
independent scoring.
\citet{li2024prd} propose peer-rank discussion among judges, in
which each judge sees others' scores and updates its own;
\citet{chan2024chateval} propose multi-agent debate, in which judges
argue over a verdict before consensus.
These methods change the joint distribution of
$(\hat{\vy}_{1}, \ldots, \hat{\vy}_{N})$---they introduce dependence
by design, breaking Assumption~\ref{asm:independence}---and trade
independence for deliberation-driven error reduction.
Whether they exhibit the same Byzantine-failure mode as \textsc{PoLL}
is an open question.
The Huber-contamination analysis of this paper does not directly
apply to such interactive aggregators, but the corruption-class
diagnosis (point masses far from the truth) likely transfers,
suggesting robust extensions of debate-based aggregation as a future
direction.
Majority voting in mathematical reasoning
\citep{cobbe2021training} is a related but coarser ensemble
technique on binary correctness;
the analogue of Proposition~\ref{prop:mean_bias} for vote-based
aggregation on $\{0, 1\}$ outputs is the standard $\alpha < 1/2$
Byzantine threshold.

{\bf Calibration as a complementary paradigm.}
A separate line of work removes judge bias \emph{at the source} via
per-judge calibration on a labeled validation slice
\citep{zheng2023judging}.
Calibration assumes a stationary, recoverable bias and trades
worst-case guarantees for average-case efficiency;
\textsc{RoPoLL} assumes nothing on the corruption distribution and
pays a constant-factor insurance premium to bound the worst case.
The two are complementary: \textsc{RoPoLL} can aggregate calibrated
scores, and the calibration-RoPoLL composition---together with
extensions to heterogeneous, correlated, and dependent juries---is
left to future work.

{\bf Robust statistics and the geometric median.}
The Huber contamination model \citep{huber1964robust} and the
breakdown point \citep{tukey1960survey} are the classical framework
for estimation under arbitrary corruption.
The geometric median attains the optimal $1/2$ breakdown for any
translation-equivariant estimator
\citep{lopuhaa1991breakdown,small1990survey,vardi2000multivariate};
in high dimensions, \citet{minsker2015geometric} established
sub-Gaussian concentration for the geometric median of
means---the result Theorem~\ref{thm:ropoll_bound} adapts to
contaminated juries---and \citet{lugosi2019sub} developed
sub-Gaussian mean estimators with optimal dimension dependence.
Recent applications to ML pipelines include block-coordinate GM
descent for robust training \citep{acharya2022robust} and GM
Matching for robust subset selection
\citep{acharya2025geometric};
\citet{acharya2025robust} surveys robust learning from noisy data.
Our setting differs from this literature on three axes:
\emph{(i)~low dimension} ($d \in \{4, 5\}$ evaluation attributes,
so the $\sqrt{d/N}$ rate is dominated by constants and the
$1/(1-2\alpha)$ contamination factor is the load-bearing
dependence);
\emph{(ii)~structured contamination} ($Q_{i}$ arises from specific
LLM failure modes---parser fallback, sycophancy, refusals,
cross-attribute confusion---which inform the four empirical
corruption types in \S\ref{sec:benchmark});
and \emph{(iii)~heterogeneous workers} (per-judge
$\sigma_{i}, \alpha_{i}$ vary across the panel, outside the i.i.d.\
regime that the classical robust-statistics literature targets).
Among broader alternatives in the robust-aggregation toolbox, the
\emph{half-space (Tukey) median} attains the optimal breakdown
$1/2$ in any dimension but is NP-hard to compute and prohibitive at
$d \geq 5$ \citep{small1990survey};
\emph{median of means} \citep{lugosi2019sub} targets heavy-tailed
data rather than Huber contamination concentrated in a minority of
judges;
the geometric median's tuning-free $1/2$ breakdown, joint-distance
objective, and $O(Nd\log(1/\eps))$ cost make it the right default
for the small-$d$, small-$N$, heterogeneous-worker, one-shot regime
that LLM juries occupy.
A systematic empirical comparison against the broader family is
left to future work (\S\ref{sec:discussion}).

{\bf Byzantine-robust distributed learning.}
The connection between robust aggregation and Byzantine fault
tolerance has been worked out in distributed optimization:
Krum \citep{blanchard2017machine}, coordinate-wise median and
trimmed mean as gradient aggregators \citep{yin2018byzantine}, and
Bulyan \citep{elmhamdi2018hidden}.
This literature targets $N$ from tens to thousands of workers, with
adversarial perturbations composed across thousands of training
rounds.
The LLM-jury setting shares the mathematical structure but differs
operationally on three axes:
\emph{(a)~small $N$} (juries operate at $N \in \{3,\dots,13\}$
where every judge is materially expensive, requiring tight
finite-sample guarantees);
\emph{(b)~per-sample heterogeneity} (the contamination indicator
$Z_{i}$ is conditional on the prompt-response pair $x$, not per
round);
\emph{(c)~no iterative learning loop} (LLM-jury aggregation is
one-shot at evaluation time, so the per-instance bias bound matters
directly rather than its cumulative effect across rounds).
These differences explain why our analysis emphasizes finite-sample
distribution-free guarantees over the corruption class
(Theorem~\ref{thm:ropoll_bound});
the heterogeneity of the worker pool, judge correlation, and
explicit dependence (in debate-based methods) are left to future
work and have no direct analogue in the Byzantine
distributed-learning literature.

\medskip
\noindent
To our knowledge, we are the first to formalize LLM jury
aggregation as a robust estimation problem, prove finite-sample
contamination guarantees in this setting, and evaluate robustness
systematically against both natural and adversarial judge failures
at scale.

\section{Problem Setup}
%
\label{sec:problem_setup}

We evaluate a system agent $\gM : \gP \to \gR$ that maps prompts to
responses.
For each evaluation instance $x = (p, r) \in \gX \triangleq \gP \times \gR$
the goal is to estimate a vector of attribute scores describing how
good the response $r$ is for the prompt $p$.

\subsection{System Agent and Reward Space}
\label{sec:reward}

Let $\gP$ and $\gR$ denote the spaces of admissible natural-language
prompts and responses, and define the \emph{instance space}
$\gX \triangleq \gP \times \gR$.
The model under evaluation is the \emph{system agent}
$\gM : \gP \to \gR$, $p \mapsto \gM(p)$.
Given a prompt $p$, the realized response is $r = \gM(p)$ and the
evaluation instance is $(p, r) \in \gX$.
Any stochasticity in the underlying generation procedure is immaterial
for the development below, which is carried out conditional on the
realized pair $(p, r)$.

\begin{definition}[Reward]
\label{def:reward}
Fix $d \in \mathbb{N}$ and write $[d] \triangleq \{1, \ldots, d\}$.
For each $k \in [d]$, let $\gY^{(k)}$ be a measurable space encoding
the admissible judgements for attribute $k$.
The \emph{reward space} is the Cartesian product
$\gY \triangleq \prod_{k=1}^{d} \gY^{(k)}$, and a \emph{reward}
associated with $(p, r) \in \gX$ is a vector
$\vy = (y^{(1)}, \ldots, y^{(d)}) \in \gY$
where $y^{(k)} \in \gY^{(k)}$ records the judgement
of $r$ on attribute $k$.
\end{definition}

Definition~\ref{def:reward} does not impose a common structure across
the coordinate spaces, and typical instantiations include bounded
scalars ($\gY^{(k)} = [0, K_{k}]$), categorical or ordinal labels
($\gY^{(k)} = \{c_{1}, \ldots, c_{L}\}$), and free-form text
($\gY^{(k)} = \gR$).
For ease of exposition we specialize throughout the paper to the
\emph{homogeneous bounded-scalar} setting: there exists $K > 0$ with
$\gY^{(k)} = [0, K]$ for every $k \in [d]$, so
\begin{equation}
\label{eq:scalar_reward_space}
    \gY \;=\; [0, K]^{d} \;\subset\; \R^{d}.
\end{equation}

\begin{assumption}[Latent Reward Functional]
\label{asm:latent_reward}
Under~\eqref{eq:scalar_reward_space} there exists a measurable map
\begin{equation}
    \vy^{\star} : \gX \to [0, K]^{d},
    \qquad
    \vy^{\star}(x)
    = \big(y^{\star(1)}(x), \ldots, y^{\star(d)}(x)\big),
\end{equation}
called the \emph{latent reward functional}, such that $\vy^{\star}(x)$
is the canonical attribute-wise assessment of response $r$ to prompt
$p$ under the reference evaluation protocol
(Definition~\ref{def:reference_protocol}).
The components $y^{\star(k)}(x) \in [0, K]$ are unobservable.
\end{assumption}

\subsection{Reference Protocol, Rubric, and Parser}
\label{sec:pipeline}

Because $\vy^{\star}$ is unobservable, evaluation must proceed through
an observable reference protocol.

\begin{definition}[Reference Protocol]
\label{def:reference_protocol}
A \emph{reference protocol} is a Markov kernel from $\gX$ to
$[0, K]^{d}$
\citep{billingsley1995probability,dudley2002real,kallenberg2002foundations}:
$\mathcal{A} : \gX \rightsquigarrow [0, K]^{d}$,
meaning that for each $x \in \gX$, $\mathcal{A}(\cdot \mid x)$ is a
probability measure on $[0, K]^{d}$ and, for each Borel set
$B \subseteq [0, K]^{d}$, the map $x \mapsto \mathcal{A}(B \mid x)$ is
measurable.
We interpret $\mathcal{A}(\cdot \mid x)$ as the distribution of the
reference label assigned to $x$.
Given evaluation instances $x_{1}, \ldots, x_{M} \in \gX$ with
$x_{j} = (p_{j}, r_{j})$, the corresponding benchmark dataset is
\begin{equation}
\label{eq:benchmark_dataset}
    \gD = \{(x_{j}, \vy_{j}^{\mathrm{ref}})\}_{j=1}^{M},
    \qquad
    \vy_{j}^{\mathrm{ref}} \sim \mathcal{A}(\cdot \mid x_{j}).
\end{equation}
The protocol $\mathcal{A}$ may encode expert human annotation, an
aggregation of multiple human judgements, or a designated reference
model.
In the noiseless idealization
$\mathcal{A}(\cdot \mid x) = \delta_{\vy^{\star}(x)}$, so
$\vy_{j}^{\mathrm{ref}} = \vy^{\star}(x_{j})$ for every $j \in [M]$.
The reference labels are used only for evaluation and are not
available to the predictors under study.
\end{definition}

\begin{definition}[Rubric]
\label{def:rubric}
A \emph{rubric} is a natural-language specification $\rho \in \gP$
that fixes:
(i) the collection of $d$ evaluation attributes and their semantics;
(ii) the score range $[0, K]$ associated with each attribute;
and (iii) the output schema from which scores are extracted.
Associated with $\rho$ is a deterministic encoding map
$\operatorname{enc}_{\rho} : \gX \to \gP$, which serializes an
evaluation instance $x = (p, r)$ into the prompt presented to the
judging model.
\end{definition}

\begin{definition}[LLM-As-Judge]
\label{def:judge}
An \emph{LLM judge} is a triplet $f = (\gM_{f},\, \rho,\, \phi)$,
where (i) $\gM_{f} : \gP \rightsquigarrow \gR$ is a backbone language
model viewed as a Markov kernel from prompts to raw textual outputs;
(ii) $\rho$ is a rubric (Definition~\ref{def:rubric});
and (iii) $\phi : \gR \to \R^{d}$ is a measurable deterministic
parser that extracts a score vector from the raw text.
For an evaluation instance $x \in \gX$, the induced pipeline is
\begin{equation}
\label{eq:judge_pipeline}
    x \;\xrightarrow{\;\operatorname{enc}_{\rho}\;}\; \operatorname{enc}_{\rho}(x)
    \;\xrightarrow{\;\gM_{f}\;}\; T_{f}
    \;\xrightarrow{\;\phi\;}\; \hat{\vy}_{f}(x) \;\in\; \R^{d},
\end{equation}
where $T_{f} \sim \gM_{f}(\,\cdot \mid \operatorname{enc}_{\rho}(x))$
and $\hat{\vy}_{f}(x) = \phi(T_{f})$.
Equivalently, $f$ induces a Markov kernel
$f : \gX \rightsquigarrow \R^{d}$ via
$f(B \mid x)
= \gM_{f}\!\big(\{t \in \gR : \phi(t) \in B\} \mid \operatorname{enc}_{\rho}(x)\big)$
for every Borel set $B \subseteq \R^{d}$.
\end{definition}

\begin{remark}[Operational Stochasticity]
\label{rem:stochasticity}
Even under deterministic API settings (temperature $\tau = 0$), the
induced law $f(\cdot \mid x)$ need not be degenerate.
Effective randomness arises from non-determinism in the inference
stack, sensitivity to prompt serialization, and post-processing
branching;
modeling the judge at the level of the induced kernel $f$ absorbs
these effects without committing to a particular source of
randomness.
\end{remark}

\begin{remark}[Parser-Induced Atoms]
\label{rem:parser}
The parser $\phi$ is part of the estimator, not merely an
implementation detail.
In practice, $\phi$ maps malformed outputs, refusals, or missing
fields to a fixed fallback vector such as $\vzero$, so the law of
$\hat{\vy}_{f}(x)$ may contain non-trivial point masses even when
$T_{f}$ has a diffuse generation law.
This is the formal counterpart of the \emph{mode collapse} failure
mode discussed in \S\ref{sec:intro} and is the mechanism by which
backbone-level Gaussian noise becomes parser-level Huber
contamination.
\end{remark}

\subsection{LLM Jury and Aggregation Function}
\label{sec:jury}

\begin{definition}[\bf LLM Jury]
\label{def:jury}
A \emph{jury} is a finite collection of $N$ LLM judges
$\gJ = \{f_{1}, \ldots, f_{N}\}$ sharing a common rubric $\rho$ and
parser $\phi$ but employing distinct backbones
$\{\gM_{f_{i}}\}_{i=1}^{N}$.
On instance $x$, the jury produces score vectors
$\{\hat{\vy}_{1}, \ldots, \hat{\vy}_{N}\} \subset \R^{d}$.
\end{definition}

\begin{definition}[\bf Aggregation Function]
\label{def:aggregation}
An \emph{aggregation function} is a measurable map
$\gA : (\R^{d})^{N} \to \R^{d}$ producing a consensus estimate
$\hat{\vy}_{\mathrm{agg}} = \gA(\hat{\vy}_{1}, \ldots, \hat{\vy}_{N})$.
The objective is to minimize
$\|\hat{\vy}_{\mathrm{agg}} - \vy^{\star}\|_{2}$ uniformly over the
evaluation distribution.
\end{definition}

The central question of this work is: \emph{which $\gA$ remains
accurate when judges fail in arbitrary, possibly adversarial ways?}
To answer it formally, we adopt the classical contamination model
from robust statistics.

\subsection{Huber Contamination Model and Companion Assumptions}
\label{sec:huber}

\begin{assumption}[\bf Huber $\eps$-Contamination Model]
\label{asm:huber}
Each judge $f_{i} \in \gJ$ has a contamination rate
$\alpha_{i} \in [0, 1)$, and the conditional law of $\hat{\vy}_{i}$
given $\vy^{\star}$ is the mixture
\begin{equation}
\label{eq:huber}
    \hat{\vy}_{i} \,\sim\, (1 - \alpha_{i})\, P_{i} + \alpha_{i}\, Q_{i},
\end{equation}
where the \emph{competent} component $P_{i}$ is unbiased, with
$\E_{P_{i}}[\hat{\vy}_{i}] = \vy^{\star}$ and finite second moment,
and the \emph{corruption} component $Q_{i}$ is an arbitrary
distribution on $\R^{d}$.
At each evaluation an indicator
$Z_{i} \sim \mathrm{Bernoulli}(\alpha_{i})$ selects which component
generates $\hat{\vy}_{i}$.
\end{assumption}

\paragraph{Concrete instantiations of $Q_{i}$.}
The unrestricted $Q_{i}$ admits, as special cases, every LLM-judge
failure mode reported in the single-judge bias literature:
\emph{mode collapse} ($Q_{i} = \delta_{\vzero}$, the parser-fallback
vector emitted on malformed JSON or safety refusals);
\emph{sycophancy} ($Q_{i} = \delta_{K \cdot \vone}$, near-maximum
scores assigned indiscriminately,
\citealp{wang2023large,stureborg2024large});
\emph{anti-correlated Byzantine attacks}
($Q_{i} = \delta_{K\vone - \vy^{\star}}$, mirror-image scores);
\emph{cross-attribute confusion} ($Q_{i} = \mathrm{Unif}\{0,K\}^{d}$,
each coordinate plausible per axis but jointly anomalous,
matching the cross-dimensional failure mode of
Example~\ref{ex:cross_dim});
and \emph{heavy-tailed adversaries} ($Q_{i}$ Cauchy or otherwise
unbounded, modelling parser hallucinations of arbitrarily large
scores).
The four synthetic regimes evaluated in \S\ref{sec:benchmark}
(\texttt{zeros}, \texttt{inverted}, \texttt{bimodal-random},
\texttt{cauchy-far}) instantiate these four $Q_{i}$ choices
respectively;
$\alpha_{i}$ encodes the per-judge unreliability and is expected to
decrease with backbone capacity.

\paragraph{Empirical grounding.}
Naturally-occurring parser failures (the
$Q_{i} = \delta_{\vzero}$ instantiation above) are not hypothetical:
across our $13$-judge $\times$ benchmark grid
(\Cref{fig:corpus_failure_rates}),
mean failure rates span $0.59\%$ on HelpSteer\,2 and $3.38\%$ on
HelpSteer\,3, with the smallest judge (Gemma-4B) failing on $33\%$ of
HS\,3 multilingual signed-preference samples.
The deployment regime is therefore dataset-dependent across one to
two orders of magnitude in $\alpha$, and the
distribution-free contamination class
$\{Q_{i}\}$ is the right object of study.

\begin{figure*}[t]
\centering
\includegraphics[width=0.95\textwidth]{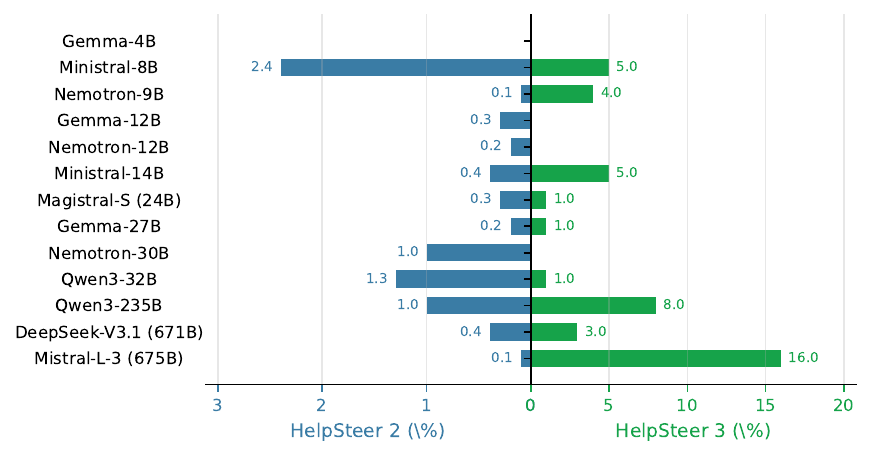}
\caption{\textbf{Naturally-occurring parser-failure rates motivate
the contamination model.}
Horizontal bars per judge (sorted by parameter count, top = smallest)
restricted to the $13$-judge pool common to both benchmarks
(Claude-Opus/Sonnet/Haiku-4.5 are HS\,3-only and excluded here for
panel alignment; their HS\,3 statistics appear in
Table~\ref{tab:corpus_stats}).
The natural failure regime is \emph{dataset-dependent}: $0.59\%$
mean on HelpSteer\,2 and $3.38\%$ mean on HelpSteer\,3---with the
smallest judge (Gemma-4B) failing on $33\%$ of HS\,3 multilingual
signed-preference samples (full $16$-judge pool).
Each parser-failure event is a Dirac mass at the fallback vector
$\vzero$, instantiating $Q = \delta_{\vzero}$ in
Assumption~\ref{asm:huber} (mode collapse).
Naturally-occurring rates already span $0\%$ to $33\%$, motivating
the synthetic sweep $r \in [0\%, 50\%]$ studied in
\S\ref{sec:benchmark}, which covers this natural regime and
stress-tests beyond.}
\label{fig:corpus_failure_rates}
\end{figure*}

\begin{assumption}[Conditional Independence]
\label{asm:independence}
Conditioned on $\vy^{\star}$, the judge outputs
$\hat{\vy}_{1}, \ldots, \hat{\vy}_{N}$ are mutually independent.
\end{assumption}

\begin{remark}[i.i.d.\ as baseline; correlated extension]
\label{rem:correlated_extension}
Assumption~\ref{asm:independence} is the standard i.i.d.\ baseline
of robust statistics.
Real LLM juries trained on overlapping corpora violate this:
inter-judge correlation
$\bar\gamma \in [0.3, 0.7]$ is typical
(\Cref{fig:corpus_judge_corr,sec:appendix_dataset_correlation}).
Lemma~\ref{lem:correlated_jury} (\S\ref{sec:upper_bound})
extends Theorem~\ref{thm:ropoll_bound} to the equicorrelated case
and shows that the breakdown structure ($C_{\alpha+\beta}$ and
$\rho$) is unchanged;
only the high-probability event weakens, from
$1-\exp(-N\beta^{2}/2)$ (Hoeffding under independence) to
$1 - 1/(\beta^{2}N_{\mathrm{eff}})$ (Chebyshev under correlation),
with $N_{\mathrm{eff}} = N/(1+(N-1)\bar\gamma_{W})$.
\end{remark}

\begin{assumption}[Sub-Gaussian Competent Noise]
\label{asm:subgaussian}
For each judge $f_{i}$, the competent component $P_{i}$ of
Assumption~\ref{asm:huber} is $\sigma_{i}^{2}$-sub-Gaussian:
for all $\vu \in \mathbb{S}^{d-1}$,
\begin{equation}
\label{eq:subgaussian}
    \E_{P_{i}}\!\Big[\exp\!\big(\lambda\, \vu^{\top}(\hat{\vy}_{i} - \vy^{\star})\big)\Big]
    \;\leq\; \exp\!\big(\lambda^{2} \sigma_{i}^{2} / 2\big),
    \qquad \forall\, \lambda \in \R.
\end{equation}
The parameter $\sigma_{i}^{2}$ is the per-judge \emph{skill} parameter.
\end{assumption}

\begin{remark}[Sub-Gaussian Is Strictly Weaker than Gaussian]
\label{rem:subg_is_weaker}
Assumption~\ref{asm:subgaussian} is automatically satisfied by any
distribution supported on a bounded set, and in particular by every
score distribution arising from a parser with codomain $[0, K]^{d}$
restricted to its \emph{competent} regime.
The boundedness of the response scale therefore makes the
sub-Gaussian assumption non-restrictive in our setting.
\end{remark}

\begin{assumption}[\bf Minority Corruption]
\label{asm:minority}
The \emph{effective contamination fraction}
\begin{equation}
\label{eq:alpha_global}
    \alpha \;\triangleq\; \frac{1}{N}\sum_{i=1}^{N}\alpha_{i}
\end{equation}
satisfies $\alpha < 1/2$.
\end{assumption}

\begin{remark}[Tightness of Minority Corruption]
\label{rem:minority_tight}
The threshold $\alpha < 1/2$ is {\bf information-theoretically tight}:
if corrupted judges form a majority they can simulate any target law
and $\vy^{\star}$ becomes unidentifiable without further structure on
$\{Q_{i}\}$.
\end{remark}

\subsection{Observation Model and Variance Reduction}
\label{sec:variance_reduction}

Collecting the assumptions above, the complete observation model for
a single evaluation instance is
\begin{equation}
\label{eq:full_model}
\boxed{
    \hat{\vy}_{i}
    \;=\;
    (1 - Z_{i})\,(\vy^{\star} + \bm{\epsilon}_{i})
    \;+\;
    Z_{i}\, \bm{\eta}_{i},
    \qquad
    i = 1, \ldots, N,
}
\end{equation}
where $Z_{i} \sim \mathrm{Bernoulli}(\alpha_{i})$ are independent
latent corruption indicators,
$\bm{\epsilon}_{i} \sim P_{i} - \vy^{\star}$ is zero-mean and
$\sigma_{i}^{2}$-sub-Gaussian (Assumption~\ref{asm:subgaussian}), and
$\bm{\eta}_{i} \sim Q_{i}$ is the arbitrary corruption noise,
independent of $\bm{\epsilon}_{i}$ and $Z_{i}$.
The statistician observes only
$\{\hat{\vy}_{1}, \ldots, \hat{\vy}_{N}\}$ and has no access to
$\{Z_{i}\}$ or $\{Q_{i}\}$.

The canonical jury aggregator is the arithmetic mean adopted by
\textsc{PoLL}~\citep{verga2024replacing}:
\begin{equation}
\label{eq:mean_aggregation}
    \hat{\vy}_{\mathrm{mean}}
    \;\triangleq\;
    \frac{1}{N}\sum_{i=1}^{N}\hat{\vy}_{i}.
\end{equation}
On clean juries the mean enjoys the parametric variance-reduction
rate.

\begin{proposition}[\bf Variance Reduction for the Clean Jury]
\label{prop:variance_reduction}
Assume $\alpha_{i} = 0$ for all $i \in [N]$, so every judge operates
in the competent regime.
Then
$\E[\hat{\vy}_{\mathrm{mean}} \mid \vy^{\star}] = \vy^{\star}$ and
\begin{equation}
\label{eq:mean_cov_general}
    \Cov\!\left(\hat{\vy}_{\mathrm{mean}} \mid \vy^{\star}\right)
    \;=\;
    \frac{1}{N^{2}}\sum_{i=1}^{N}\sum_{j=1}^{N}
    \Cov\!\left(\hat{\vy}_{i}, \hat{\vy}_{j} \mid \vy^{\star}\right).
\end{equation}
Under Assumption~\ref{asm:independence}, the off-diagonal terms
vanish and
\begin{equation}
\label{eq:mean_cov_indep}
    \Cov\!\left(\hat{\vy}_{\mathrm{mean}} \mid \vy^{\star}\right)
    \;=\;
    \frac{1}{N^{2}}\sum_{i=1}^{N}\mSigma_{i},
    \qquad
    \E\!\left[\left\|\hat{\vy}_{\mathrm{mean}} - \vy^{\star}\right\|_{2}^{2}
        \,\middle|\, \vy^{\star}\right]
    \;=\;
    \frac{1}{N^{2}}\sum_{i=1}^{N}\operatorname{tr}(\mSigma_{i}).
\end{equation}
If $\mSigma_{i} \preceq \sigma^{2}\mI_{d}$ uniformly, then
$\E\!\left[\left\|\hat{\vy}_{\mathrm{mean}} - \vy^{\star}\right\|_{2}^{2}
\,\middle|\, \vy^{\star}\right] \leq d\sigma^{2}/N$.
\end{proposition}

The proof is a direct application of linearity of expectation and
bilinearity of covariance and is given in
Appendix~\ref{sec:appendix_variance_proof}.

\begin{corollary}[Effective Jury Size Under Correlation]
\label{cor:effective_jury_size}
Assume $\alpha_{i} = 0$ for all $i$ and that there exist
$\mSigma \succeq 0$ and $\gamma \in [-1/(N-1), 1]$ with
$\Cov(\hat{\vy}_{i} \mid \vy^{\star}) = \mSigma$ and
$\Cov(\hat{\vy}_{i}, \hat{\vy}_{j} \mid \vy^{\star}) = \gamma\,\mSigma$
for $i \neq j$.
Then
$\Cov(\hat{\vy}_{\mathrm{mean}} \mid \vy^{\star})
= \tfrac{1 + (N-1)\gamma}{N}\,\mSigma$
and
$\E[\|\hat{\vy}_{\mathrm{mean}} - \vy^{\star}\|_{2}^{2}\mid\vy^{\star}]
= \tfrac{1 + (N-1)\gamma}{N}\,\operatorname{tr}(\mSigma)$.
Defining the \emph{effective jury size}
$N_{\mathrm{eff}} \triangleq N / (1 + (N-1)\gamma)$, the MSE rate is
$\operatorname{tr}(\mSigma) / N_{\mathrm{eff}}$.
\end{corollary}

A controlled synthetic validation
(\Cref{fig:tv_gamma} of \Cref{fig:theory_validation}, \S\ref{sec:upper_bound})
confirms that empirical MSE on an equicorrelated Gaussian jury
matches the closed-form prediction
$\tfrac{1 + (N-1)\gamma}{N}\,d\sigma^{2}$ across the full range
$\gamma \in [0, 0.95]$.
The corollary has a direct implication for jury design: for any
$\gamma > 0$ the effective jury size $N_{\mathrm{eff}}$ saturates at
$1/\gamma$, so adding more judges past $N \approx 1/\gamma$ buys
essentially nothing.
With $\gamma$ in the moderate range $[0.3, 0.5]$ characteristic of
diverse but non-orthogonal LLM backbones, $N_{\mathrm{eff}}$
saturates already at $N \approx 2$--$3$, motivating the three-judge
committees throughout \S\ref{sec:benchmark}.

\subsection{Fragility of \textsc{PoLL}}
\label{sec:mean_fragility}

The next result shows that the $1/N$ variance-reduction rate of
Proposition~\ref{prop:variance_reduction} is irrelevant the moment
any contamination is present.

\begin{proposition}[\bf Unbounded Bias of \textsc{PoLL}]
\label{prop:mean_bias}
Under Assumption~\ref{asm:huber}, suppose each $Q_{i}$ has finite
first moment $\vmu_{i}^{Q} \triangleq \E_{Q_{i}}[\hat{\vy}_{i}] \in \R^{d}$.
Then
\begin{equation}
\label{eq:mean_bias_formula}
    \E\!\left[\hat{\vy}_{\mathrm{mean}} \mid \vy^{\star}\right]
    \;=\;
    \vy^{\star}
    \;+\;
    \frac{1}{N}\sum_{i=1}^{N}\alpha_{i}\!\left(\vmu_{i}^{Q} - \vy^{\star}\right),
\end{equation}
and for any $\alpha > 0$ the conditional bias
$\E[\hat{\vy}_{\mathrm{mean}} \mid \vy^{\star}] - \vy^{\star}$ cannot
be uniformly bounded under Assumption~\ref{asm:huber}, regardless of
$N$.
\end{proposition}

\begin{proof}[Proof sketch]
The bias formula~\eqref{eq:mean_bias_formula} is immediate from
linearity of expectation and the per-judge identity
$\E[\hat{\vy}_{i} \mid \vy^{\star}]
= (1 - \alpha_{i})\vy^{\star} + \alpha_{i}\vmu_{i}^{Q}$.
For unboundedness, fix any $B > 0$ and any index $i_{0}$ with
$\alpha_{i_{0}} > 0$ (which exists since $\alpha > 0$).
Choose $Q_{i_{0}} = \delta_{\vy^{\star} + (NB/\alpha_{i_{0}})\,\ve_{1}}$
and $Q_{i} = P_{i}$ for $i \neq i_{0}$.
Then~\eqref{eq:mean_bias_formula} reduces to $B\,\ve_{1}$, so the bias
has Euclidean norm $B$;
since $B$ is arbitrary, no constant depending only on
$(\alpha, N, d, \sigma)$ can bound the bias uniformly over $\{Q_{i}\}$.
The full proof is in Appendix~\ref{sec:appendix_mean_bias_proof}.
\end{proof}

Proposition~\ref{prop:mean_bias} is the {\bf central impossibility}
motivating \textsc{RoPoLL}: variance reduction over $N$ judges is
irrelevant when the bias of the aggregator is unbounded over the
corruption class.
The construction in the proof scales the corruption mean linearly
with $N$, exactly cancelling the $1/N$ averaging---so increasing the
jury size cannot fix the problem.
We therefore seek an aggregator that simultaneously
(i) matches the $O(\sigma\sqrt{d/N})$ rate of the mean in the clean
case and (ii) has bounded error under arbitrary contamination with
$\alpha < 1/2$.
The geometric median, introduced in \S\ref{sec:methodology}, achieves
both.

\section{Robust Panel of LLM Judges}
\label{sec:methodology}
%

Proposition~\ref{prop:mean_bias} forces us to abandon the arithmetic
mean: under contamination its bias is unbounded over the corruption
class regardless of jury size $N$.
We therefore propose \textsc{RoPoLL}, a drop-in replacement for the
\textsc{PoLL} aggregation step that swaps the arithmetic mean for a
robust mean estimator.
The framework is agnostic to the choice of estimator;
we instantiate it with the geometric median, motivated below.

\subsection{Choosing the Robust Estimator}
\label{sec:choose_estimator}

Three classical robust mean estimators are natural candidates: the
\emph{coordinate-wise median} (CoMed), the \emph{trimmed mean}, and
the \emph{geometric median} (GM).

\paragraph{Coordinate-wise median.}
The coordinate-wise median applies the univariate median per
dimension, solving the separable problem
\begin{equation}
\label{eq:coord_median}
    \hat{\vy}_{\mathrm{Med}}
    \;=\; \argmin_{\vz \in \R^{d}}
        \sum_{i=1}^{N} \|\vz - \hat{\vy}_{i}\|_{1},
\end{equation}
whose $k$-th coordinate is the univariate median
$\mathrm{med}_i(\hat{y}_i^{(k)})$.
The geometric median operates on joint Euclidean distance instead.
The distinction matters when corruptions are structured across
dimensions, as the following example illustrates.

\begin{example}[Cross-Dimensional Corruption]
\label{ex:cross_dim}
Consider a jury evaluating on two attributes with ground truth
$\vy^{\star} = (2.5, 2.5)$ on $[0, 5]^{2}$.
Suppose a corrupted judge outputs
$\hat{\vy}_{\mathrm{corr}} = (0, 5)$.
Each coordinate individually lies in the plausible range $[0, 5]$,
so the coordinate-wise median treats this as unremarkable per axis.
The joint displacement
$\|\hat{\vy}_{\mathrm{corr}} - \vy^{\star}\|_{2} = \sqrt{12.5}
\approx 3.54$ is large, however, and the geometric median correctly
downweights the point (\Cref{fig:diag_crossdim}).
\end{example}

%
%
%
\begin{figure}[t]
\centering
\begin{tikzpicture}[
    every node/.style={font=\small},
    competent/.style={circle, fill=blue!55, draw=blue!75!black,
                       line width=0.3pt, inner sep=1.6pt},
    corrupted/.style={red!75!black, mark size=3pt,
                       line width=1pt},
    truthstar/.style={star, star points=5, star point ratio=2.4,
                       fill=yellow!85!orange, draw=black,
                       line width=0.4pt, minimum size=10pt, inner sep=0pt},
    radlabel/.style={font=\small, inner sep=1.3pt, fill=white,
                     fill opacity=0.92, text opacity=1, rounded corners=1pt},
    coordtick/.style={red!70!black, line width=1.0pt},
]
\def\Kbox{5.0}
\fill[blue!4] (0,0) rectangle (\Kbox, \Kbox);
\draw[gray!75!black, line width=0.7pt] (0,0) rectangle (\Kbox, \Kbox);

\foreach \v in {0, 1, 2, 3, 4, 5} {
    \draw[gray!60!black, line width=0.4pt] (\v, 0) -- (\v, -0.10);
    \draw[gray!60!black, line width=0.4pt] (0, \v) -- (-0.10, \v);
}
\node[font=\footnotesize, anchor=north] at (0,    -0.18) {$0$};
\node[font=\footnotesize, anchor=north] at (\Kbox,-0.18) {$K$};
\node[font=\footnotesize, anchor=east]  at (-0.18, 0)    {$0$};
\node[font=\footnotesize, anchor=east]  at (-0.18, \Kbox){$K$};
\node[font=\footnotesize, anchor=north] at (\Kbox/2, -0.45) {Attribute 1};
\node[font=\footnotesize, anchor=south, rotate=90]
    at (-0.55, \Kbox/2) {Attribute 2};

\coordinate (yst) at (2.5, 2.5);

\foreach \x/\y in {2.10/2.35, 2.85/2.20, 2.50/2.85} {
    \node[competent] at (\x, \y) {};
}

\coordinate (ycorr) at (0, \Kbox);
\node[corrupted] at (ycorr) {\pgfuseplotmark{x}};

\draw[->, line width=0.65pt, red!55!black, dashed]
    (yst) -- (ycorr);
\node[radlabel, text=red!55!black, anchor=center]
    at ($(yst)!0.32!(ycorr)+(0.95,-0.05)$)
    {$\|\hat{\vy}_{\mathrm{corr}}\!-\!\vy^{\star}\|_{2}\!=\!\sqrt{12.5}\!\approx\!3.54$};

\draw[coordtick] (0, -0.04) -- (0, -0.22);
\node[font=\footnotesize, text=red!70!black, anchor=north, align=center]
    at (0.0, -0.78) {$\hat{y}_{\mathrm{corr}}^{(1)}\!=\!0$};
\draw[coordtick] (-0.04, \Kbox) -- (-0.22, \Kbox);
\node[font=\footnotesize, text=red!70!black, anchor=east, align=right]
    at (-1.05, \Kbox) {$\hat{y}_{\mathrm{corr}}^{(2)}\!=\!K$};

\draw[gray!50!black, line width=0.7pt, dashed]
    (yst) -- (2.5, 0);
\draw[gray!50!black, line width=0.7pt, dashed]
    (yst) -- (0, 2.5);

\node[truthstar] (yststar) at (yst) {};
\node[font=\footnotesize, anchor=west, xshift=5pt] at (yststar) {$\vy^{\star}=(2.5,2.5)$};

\node[font=\footnotesize, text=red!60!black, align=center,
      fill=white, fill opacity=0.92, text opacity=1,
      inner sep=2pt, rounded corners=1pt]
    at (3.7, 0.65)
    {\textbf{CoMed:} both coords\\in $[0,K]$, no anomaly};

\node[font=\footnotesize, text=blue!60!black, align=center,
      fill=white, fill opacity=0.92, text opacity=1,
      inner sep=2pt, rounded corners=1pt]
    at (4.00, 4.45)
    {\textbf{GM:} joint $\ell_{2}$\\displacement large};

\node[radlabel, text=red!75!black, anchor=south]
    at ($(ycorr)+(0.50,0.30)$) {$\hat{\vy}_{\mathrm{corr}}=(0,K)$};

\end{tikzpicture}
\caption{\textbf{Cross-dimensional corruption (Example~\ref{ex:cross_dim}).}
Three competent judges (blue dots) cluster around the truth
$\vy^{\star}=(2.5,2.5)$ in the score box $[0,K]^{2}$ with $K=5$.
A corrupted judge outputs $\hat{\vy}_{\mathrm{corr}}=(0,K)$:
\emph{each coordinate individually} lies in the plausible range
$[0,K]$ (red axis ticks), so any \emph{coordinate-wise} estimator
sees nothing anomalous on either axis.
Jointly, however, the corrupted vector lies at $\ell_{2}$ distance
$\sqrt{12.5}\approx 3.54$ from $\vy^{\star}$ (red dashed arrow), and
the geometric median's joint-distance objective downweights it.
This is the qualitative reason \textsc{RoPoLL} uses GM rather than
CoMed; the empirical analogue at scale is the
\texttt{bimodal-random} sweep of \S\ref{sec:exp_bimodal}.}
\label{fig:diag_crossdim}
\end{figure}
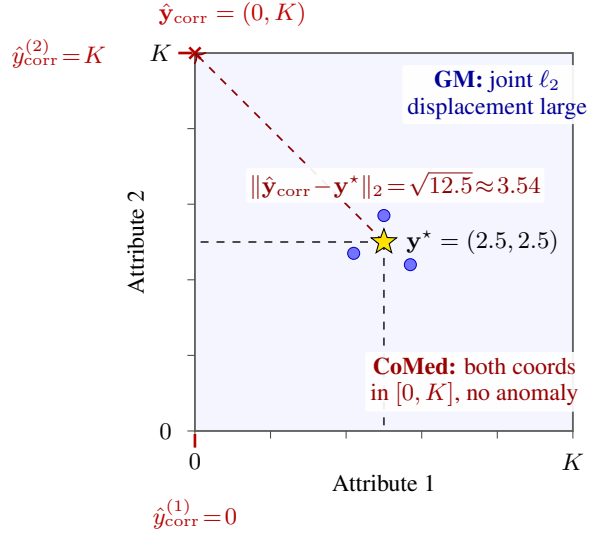

The picture in Example~\ref{ex:cross_dim} extends to a class of
corruptions that are bounded per-coordinate but jointly anomalous:
random vertices of the score hypercube $\{0, K\}^{d}$, mixtures of
extreme corner values, or any corruption whose marginals look
plausible but whose joint structure is not.
The empirical \texttt{bimodal-random} class
(\S\ref{sec:exp_bimodal}), in which each corrupted slot has each
coordinate independently drawn from $\{0, K\}$ with equal
probability, is the canonical instance: per-coordinate the marginal
is $\frac{1}{2}(\delta_{0}+\delta_{K})$, indistinguishable from
plausible scoring; jointly, the corrupted vector sits at a random
hypercube corner, far from $\vy^{\star}$ in $\ell_{2}$.
Because each coordinate-wise estimator must commit per coordinate
without seeing the joint pattern, its per-coordinate bias is
$\Omega(\alpha)$ on this class (a one-dimensional Le~Cam
two-point argument under the symmetric corruption marginal,
\citealp{huber1964robust}, Thm.~5.1), and the per-coordinate
errors compose in $\ell_{2}$ with a $\sqrt{d}$ factor;
\Cref{fig:p2_bimodal_ropoll_hero} (and the
\S\ref{sec:exp_bimodal} sweep) is the empirical analogue.

\paragraph{A note on the geometric median's optimality.}
The geometric median we use is the polynomial-time robust default
at the small jury sizes $N \leq 5$ characteristic of LLM panels;
it is \emph{not} the theoretically optimal estimator at large $N$.
Geometric median-of-means
\citep{lugosi2019sub,hopkins2020mean} achieves the parametric rate
$\sigma\sqrt{d/N}$ \emph{and} a tighter breakdown-floor scaling of
order $\sigma\sqrt{\alpha}$ by aggregating block means before
applying the geometric median.
At $N \leq 5$, however, MoM degenerates: $K = N$ blocks of size 1
gives plain GM, and any coarser blocking lacks per-block
concentration.
We therefore use plain GM throughout.

\paragraph{Trimmed mean.}
The $\beta$-trimmed mean discards the $\beta$ fraction of points
farthest from the sample mean and averages the remainder.
It requires choosing $\beta$, which in turn requires knowledge of the
contamination rate $\alpha$:
if $\beta < \alpha$, corrupted points survive trimming and the bias
of Proposition~\ref{prop:mean_bias} returns;
if $\beta > \alpha$, competent points are discarded, inflating
variance.

\paragraph{Geometric median.}
The geometric median attains the same $1/2$ breakdown as the trimmed
mean, operates on \emph{joint} Euclidean distance, and is
\emph{tuning-free}---it requires no knowledge of the contamination
rate.
We therefore instantiate \textsc{RoPoLL} with the geometric median;
empirical comparisons against CoMed and the trimmed mean are
reported in \S\ref{sec:benchmark}, with the head-to-head against the
trimmed mean on heavy-tailed corruption shown in
\Cref{fig:p1_cauchy_hero}.

\subsection{The Geometric Median: Definition and Properties}
\label{sec:gm_definition}

\begin{definition}[\textsc{RoPoLL} via Geometric Median]
\label{def:gm}
Given jury outputs $\hat{\vy}_{1}, \ldots, \hat{\vy}_{N} \in \R^{d}$,
the \textsc{RoPoLL} estimate of $\vy^{\star}$ is
\begin{equation}
\label{eq:gm}
    \hat{\vy}_{\mathrm{GM}}
    \;\triangleq\;
    \argmin_{\vz \in \R^{d}} \sum_{i=1}^{N} \|\vz - \hat{\vy}_{i}\|_{2}.
\end{equation}
\end{definition}

The geometric median has a long history in location theory dating to
Fermat (1643) and Weber (1909);
its modern robustness analysis is due to
\citet{lopuhaa1991breakdown,small1990survey,vardi2000multivariate,
acharya2022robust,acharya2025geometric}.
We collect the structural properties we will use.

\begin{definition}[Finite-Sample Breakdown Point]
\label{def:breakdown}
For an estimator $T : (\R^{d})^{N} \to \R^{d}$ and a sample
$\hat{\vy}_{1:N} \in (\R^{d})^{N}$, the \emph{finite-sample breakdown
point} of $T$ at $\hat{\vy}_{1:N}$ is the smallest fraction $m/N$
such that there exists a corrupted sample $\hat{\vy}'_{1:N}$
differing from $\hat{\vy}_{1:N}$ in at most $m$ coordinates for which
$\|T(\hat{\vy}'_{1:N}) - T(\hat{\vy}_{1:N})\|_{2}$ can be made
arbitrarily large \citep{lopuhaa1991breakdown}.
\end{definition}

\begin{proposition}[Properties of the Geometric Median]
\label{prop:gm_properties}
Let $\hat{\vy}_{1}, \ldots, \hat{\vy}_{N} \in \R^{d}$ with $N \geq 1$,
and let $F(\vz) = \sum_{i=1}^{N} \|\vz - \hat{\vy}_{i}\|_{2}$.
\begin{enumerate}
    \item \textbf{Existence.} $F$ is continuous, convex, and
    coercive, so a minimizer exists.
    \item \textbf{Uniqueness.} If
    $\hat{\vy}_{1}, \ldots, \hat{\vy}_{N}$ are not collinear, then $F$
    is strictly convex and the minimizer is unique.
    \item \textbf{Affine equivariance.} For any orthogonal
    $\mU \in \R^{d \times d}$ and translation $\vb \in \R^{d}$,
    $\mathrm{GM}(\mU \hat{\vy}_{i} + \vb)
    = \mU\,\mathrm{GM}(\hat{\vy}_{i}) + \vb$.
    \item \textbf{Breakdown point.} The finite-sample breakdown point
    is $\eps^{\star} = \lceil N/2 \rceil / N \to 1/2$ as
    $N \to \infty$, which is optimal among translation-equivariant
    estimators \citep{lopuhaa1991breakdown}.
\end{enumerate}
\end{proposition}

\begin{proof}[Proof sketch]
Existence follows from continuity, convexity, and coercivity of $F$
via Weierstrass.
Strict convexity holds whenever some data point lies off any given
line, which is the non-collinearity hypothesis.
Affine equivariance is a direct calculation using
$\|\mU\vu\|_{2} = \|\vu\|_{2}$.
For the breakdown point, a subgradient argument shows that fewer
than $\lceil N/2 \rceil$ corrupted points cannot dominate the
competent unit-vector sum at infinity; tightness comes from placing
$\lceil N/2 \rceil$ points at a divergent location.
The full proof is in Appendix~\ref{sec:appendix_gm_properties}.
\end{proof}

\begin{remark}[Discrete breakdown threshold at small $N$]
\label{rem:small_N_breakdown}
The asymptotic $1/2$ breakdown is attained as $N \to \infty$;
at finite $N$ the discrete breakdown threshold is
$\lceil N/2 \rceil/N$, strictly above $1/2$ for odd $N$.
At the practical $N = 3$ used throughout \S\ref{sec:benchmark},
the threshold is $2/3$, but the \emph{integer} cutoff is
``one corrupted of three''---two corrupted out of three
breaks the geometric median regardless of $\alpha$.
The empirical sweep up to per-cell rate $r = 50\%$ thus operates
at this discrete cutoff:
in expectation $1.5$ of $3$ judges are corrupted at $r = 50\%$, so
$\sim 50\%$ of cells have one corrupted judge (within breakdown)
and $\sim 50\%$ have two or more (at or beyond breakdown).
The corruption-class dependence visible in
\Cref{fig:p7_method_contrast} reflects this discrete structure:
under symmetric mean-preserving $Q_{i}$ (\texttt{zeros},
\texttt{inverted}) the cells beyond breakdown still average
out, while under
biased $Q_{i}$ (\texttt{bimodal-random}, \texttt{cauchy-far}) they
do not, and the gap to \textsc{PoLL} grows accordingly.
\end{remark}

\subsection{The Weiszfeld Iteration}
\label{sec:weiszfeld}

The geometric median has no closed form for $d \geq
2$~\citep{bajaj1988algebraic}; we compute it via the modified
Weiszfeld iteration \citep{weiszfeld1937point,vardi2000multivariate}.

\paragraph{Derivation.}
At any non-data point $\vz \neq \hat{\vy}_{i}$ (for all $i$), the
gradient of $F(\vz) = \sum_{i} \|\vz - \hat{\vy}_{i}\|_{2}$ is
$\nabla F(\vz)
= \sum_{i=1}^{N} (\vz - \hat{\vy}_{i})/\|\vz - \hat{\vy}_{i}\|_{2}$.
Setting $\nabla F(\vz) = \vzero$ and rearranging gives the fixed-point
\begin{equation}
\label{eq:weiszfeld_fixedpoint}
    \vz
    \;=\;
    \frac{\sum_{i=1}^{N} \hat{\vy}_{i} / \|\vz - \hat{\vy}_{i}\|_{2}}
         {\sum_{i=1}^{N} 1 / \|\vz - \hat{\vy}_{i}\|_{2}}.
\end{equation}
A modified weight
$1 / \max(\|\vz - \hat{\vy}_{i}\|_{2}, \eta)$ for small stability
parameter $\eta > 0$ handles the singularity when the iterate
coincides with a data point \citep{vardi2000multivariate}, yielding
Algorithm~\ref{alg:ropoll}.

\begin{algorithm}[t]
\caption{\textsc{\bf RoPoLL}}
\label{alg:ropoll}
\begin{algorithmic}[1]
\REQUIRE Jury scores $\hat{\vy}_{1}, \ldots, \hat{\vy}_{N} \in \R^{d}$;
    tolerance $\eps > 0$; stability $\eta > 0$
\STATE $\vz^{(0)} \leftarrow \frac{1}{N}\sum_{i=1}^{N} \hat{\vy}_{i}$
\FOR{$t = 0, 1, 2, \ldots$}
    \STATE $w_{i}^{(t)} \leftarrow
        1 / \max(\|\vz^{(t)} - \hat{\vy}_{i}\|_{2},\, \eta)$
        \quad for each $i$
    \STATE $\vz^{(t+1)} \leftarrow
        \sum_{i} w_{i}^{(t)} \hat{\vy}_{i}
        \,\big/\,
        \sum_{i} w_{i}^{(t)}$
    \IF{$\|\vz^{(t+1)} - \vz^{(t)}\|_{2} < \eps$}
        \STATE \textbf{break}
    \ENDIF
\ENDFOR
\RETURN $\hat{\vy}_{\mathrm{GM}} \leftarrow \vz^{(t+1)}$
\end{algorithmic}
\end{algorithm}

Each iteration is a reweighted mean in which points far from the
current consensus receive small weights, automatically downweighting
corrupted judges.

\paragraph{Convergence and cost.}
\citet{vardi2000multivariate} prove that the modified Weiszfeld
iteration converges to the unique geometric median at a linear rate
whenever the data are not collinear:
$\|\vz^{(t)} - \hat{\vy}_{\mathrm{GM}}\|_{2}
\leq \rho^{t} \|\vz^{(0)} - \hat{\vy}_{\mathrm{GM}}\|_{2}$
for some $\rho \in (0, 1)$.
The number of iterations to reach tolerance $\eps$ is therefore
$O(\log(1/\eps))$, and each iteration costs $O(Nd)$, giving total
cost $O(Nd\log(1/\eps))$.
For a typical LLM jury ($N \leq 20$, $d \leq 5$, $\eps = 10^{-8}$)
this is microseconds on a modern processor---negligible relative to
the seconds of GPU time per judge inference.
A full convergence analysis is in
Appendix~\ref{sec:appendix_weiszfeld}.

%
%
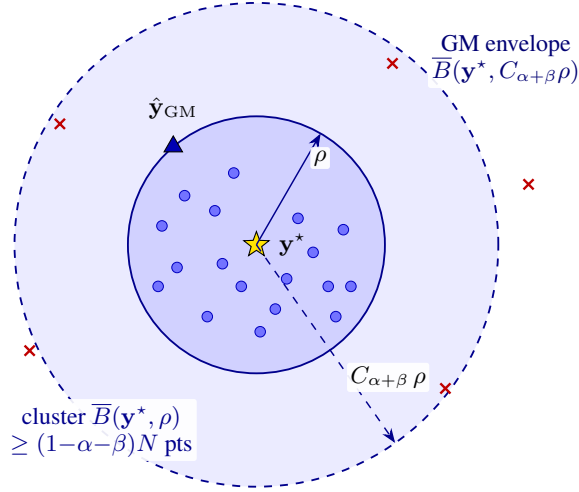
\begin{figure}[t]
\centering
\begin{tikzpicture}[
    every node/.style={font=\small},
    competent/.style={circle, fill=blue!55, draw=blue!75!black,
                       line width=0.3pt, inner sep=1.4pt},
    corrupted/.style={red!75!black, mark size=2.7pt,
                       line width=0.9pt},
    truthstar/.style={star, star points=5, star point ratio=2.4,
                       fill=yellow!85!orange, draw=black,
                       line width=0.4pt, minimum size=10pt, inner sep=0pt},
    gmtri/.style={regular polygon, regular polygon sides=3,
                  fill=blue!70!black, draw=black, line width=0.4pt,
                  minimum size=8pt, inner sep=0pt},
    radlabel/.style={font=\small, inner sep=1.3pt, fill=white,
                      fill opacity=0.92, text opacity=1, rounded corners=1pt},
]
\def\Rinner{1.7}
\def\Router{3.2}
\coordinate (ystar) at (0,0);
\fill[blue!8]   (ystar) circle (\Router);
\fill[blue!18]  (ystar) circle (\Rinner);
\draw[blue!55!black, line width=0.7pt]            (ystar) circle (\Rinner);
\draw[blue!55!black, line width=0.7pt, dashed]    (ystar) circle (\Router);

\foreach \x/\y in {
    -0.55/0.45, 0.75/-0.10, -1.05/-0.30, 0.25/-0.85,
    -0.30/0.95, 0.95/-0.55, -1.25/0.25, -0.65/-0.95,
    1.15/0.20, 0.40/-0.45, -0.20/-0.55, 1.25/-0.55,
    -0.95/0.65, -1.30/-0.55, 0.55/0.35, 0.05/-1.15,
    -0.45/-0.25, 1.05/-0.95}
    \node[competent] at (\x,\y) {};

\node[corrupted] at (-2.6,1.6) {\pgfuseplotmark{x}};
\node[corrupted] at (2.5,-1.9) {\pgfuseplotmark{x}};
\node[corrupted] at (3.6,0.8)  {\pgfuseplotmark{x}};
\node[corrupted] at (-3.0,-1.4){\pgfuseplotmark{x}};
\node[corrupted] at (1.8,2.4)  {\pgfuseplotmark{x}};

\node[truthstar] (yst) at (ystar) {};
\node[gmtri]    (gm) at (-1.10,1.30) {};

\node[font=\footnotesize, anchor=west, xshift=5pt] at (yst) {$\vy^{\star}$};
\node[font=\footnotesize, anchor=south, yshift=2.5pt] at (gm.north)
    {$\hat{\vy}_{\mathrm{GM}}$};

\draw[-Stealth, line width=0.55pt, blue!55!black]
    (ystar) -- ($(ystar)+(60:\Rinner)$);
\node[radlabel] at ($(ystar)+(60:0.78*\Rinner)+(0.18,0.0)$) {$\rho$};

\draw[-Stealth, line width=0.55pt, blue!55!black, dashed]
    (ystar) -- ($(ystar)+(-55:\Router)$);
\node[radlabel] at ($(ystar)+(-55:0.65*\Router)+(0.55,-0.05)$)
    {$C_{\alpha+\beta}\,\rho$};

\node[font=\footnotesize, blue!55!black, align=center,
      fill=white, fill opacity=0.92, text opacity=1,
      inner sep=1.5pt, rounded corners=1pt]
    at (-2.05,-2.45) {cluster $\overline{B}(\vy^{\star},\rho)$\\$\geq(1{-}\alpha{-}\beta)N$ pts};

\node[font=\footnotesize, blue!45!black, align=center,
      fill=white, fill opacity=0.92, text opacity=1,
      inner sep=1.5pt, rounded corners=1pt]
    at (3.30,2.45) {GM envelope\\$\overline{B}(\vy^{\star},C_{\alpha+\beta}\rho)$};

\end{tikzpicture}
\caption{\textbf{Geometry of Theorem~\ref{thm:ropoll_bound}.}
Lemma~\ref{lem:cluster_radius} guarantees that at least
$(1{-}\alpha{-}\beta)N$ judge outputs (blue dots) lie inside the
\emph{cluster ball} $\overline{B}(\vy^{\star}, \rho)$ of
sub-Gaussian radius $\rho$ (solid disk).
Lemma~\ref{lem:gm_breakdown} then forces the geometric median
$\hat{\vy}_{\mathrm{GM}}$ (blue triangle) to lie inside the
\emph{GM envelope} $\overline{B}(\vy^{\star}, C_{\alpha+\beta}\rho)$
(dashed disk),
\emph{regardless of where the remaining $(\alpha{+}\beta)N$
corrupted points (red $\times$) are placed}---this is the
distribution-free breakdown property of the geometric median.
The two-step composition is exactly Theorem~\ref{thm:ropoll_bound}.}
\label{fig:diag_thm1}
\end{figure}

%
%
%
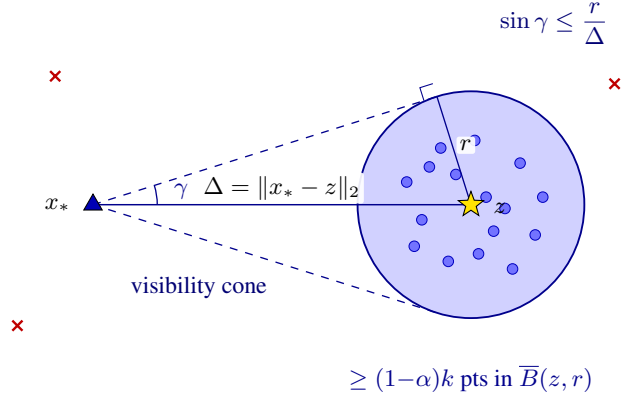
\begin{figure}[t]
\centering
\begin{tikzpicture}[
    every node/.style={font=\small},
    competent/.style={circle, fill=blue!55, draw=blue!75!black,
                       line width=0.3pt, inner sep=1.4pt},
    corrupted/.style={red!75!black, mark size=2.7pt,
                       line width=0.9pt},
    truthstar/.style={star, star points=5, star point ratio=2.4,
                       fill=yellow!85!orange, draw=black,
                       line width=0.4pt, minimum size=10pt, inner sep=0pt},
    gmtri/.style={regular polygon, regular polygon sides=3,
                  fill=blue!70!black, draw=black, line width=0.4pt,
                  minimum size=8pt, inner sep=0pt},
    radlabel/.style={font=\small, inner sep=1.3pt, fill=white,
                      fill opacity=0.92, text opacity=1, rounded corners=1pt},
    rightangle/.style={blue!55!black, line width=0.4pt},
]
\def\rball{1.5}
\coordinate (z)    at (2.5, 0);
\coordinate (xs)   at (-2.5, 0);
\coordinate (tup)  at (2.05,  1.43);
\coordinate (tdn)  at (2.05, -1.43);

\fill[blue!18] (z) circle (\rball);
\draw[blue!55!black, line width=0.7pt] (z) circle (\rball);

\foreach \dx/\dy in {
    -0.20/0.40, 0.45/-0.05, -0.65/-0.20, 0.10/-0.65,
    0.65/0.55, -0.40/0.75, 0.85/-0.40, -0.85/0.30,
    0.20/0.10, -0.30/-0.75, 0.95/0.10, 0.05/0.85,
    -0.55/0.50, 0.55/-0.85, -0.75/-0.55, 0.30/-0.35}
    \node[competent] at ($(z)+(\dx,\dy)$) {};

\node[corrupted] at (-3.5, -1.6) {\pgfuseplotmark{x}};
\node[corrupted] at (-3.0,  1.7) {\pgfuseplotmark{x}};
\node[corrupted] at ( 4.4,  1.6) {\pgfuseplotmark{x}};

\draw[line width=0.55pt, blue!55!black]
    (xs) -- (z);

\draw[line width=0.5pt, blue!55!black, dashed] (xs) -- (tup);
\draw[line width=0.5pt, blue!55!black, dashed] (xs) -- (tdn);

\draw[line width=0.5pt, blue!55!black] (z) -- (tup);

\draw[rightangle]
    ($(tup)+(-0.954*0.18,-0.300*0.18)$) --
    ($(tup)+(-0.954*0.18,-0.300*0.18)+(-0.30*0.18,0.953*0.18)$) --
    ($(tup)+(-0.30*0.18,0.953*0.18)$);

\draw[blue!55!black, line width=0.5pt]
    ($(xs)+(0:0.85)$) arc (0:17.46:0.85);

\node[truthstar] (zst) at (z) {};
\node[font=\footnotesize, anchor=west, xshift=5pt, yshift=-1pt] at (zst) {$z$};

\node[gmtri] (xsm) at (xs) {};
\node[font=\footnotesize, anchor=east, xshift=-5pt] at (xsm) {$x_{*}$};

\node[radlabel] at ($(xs)!0.5!(z)+(0,0.22)$) {$\Delta = \|x_{*}-z\|_{2}$};

\node[radlabel] at ($(z)!0.55!(tup)+(0.18,0.0)$) {$r$};

\node[font=\small, blue!55!black] at ($(xs)+(8.7:1.18)$) {$\gamma$};

\node[radlabel, text=blue!45!black, align=center]
    at (3.6, 2.4) {$\sin\gamma \leq \dfrac{r}{\Delta}$};

\node[font=\footnotesize, blue!55!black, align=center,
      fill=white, fill opacity=0.92, text opacity=1,
      inner sep=1.5pt, rounded corners=1pt]
    at (2.5,-2.30) {$\geq(1{-}\alpha)k$ pts in $\overline{B}(z,r)$};

\node[font=\footnotesize, blue!45!black, align=center,
      fill=white, fill opacity=0.92, text opacity=1,
      inner sep=1.5pt, rounded corners=1pt]
    at (-1.10,-1.10) {visibility cone};

\end{tikzpicture}
\caption{\textbf{Geometric core of Lemma~\ref{lem:gm_breakdown}.}
The contradiction hypothesis $\Delta \triangleq \|x_{*}-z\|_{2} > C_{\alpha}r$
places $x_{*}$ \emph{outside} the cluster ball $\overline{B}(z, r)$,
so the ball subtends a cone of half-angle $\gamma$ at $x_{*}$ with
$\sin\gamma = r/\Delta$ (right-angle at the tangent point shown).
Every cluster point $x_{j} \in \overline{B}(z, r)$ lies inside this
cone, hence makes angle $\gamma_{j} \leq \gamma$ with the central
ray $x_{*}\to z$, so
$\cos\gamma_{j} \geq \sqrt{1 - r^{2}/\Delta^{2}} = \alpha/(1-\alpha)$
when $\Delta > C_{\alpha}r$ with $C_{\alpha}=(1-\alpha)/\sqrt{1-2\alpha}$.
Summing this lower bound over the $(1{-}\alpha)k$ cluster indices
forces the directional derivative
$DF(x_{*}; z-x_{*})$ to be strictly negative, contradicting the
first-order optimality of the geometric median---hence
$\Delta \leq C_{\alpha}r$.}
\label{fig:diag_lem1}
\end{figure}

%
%
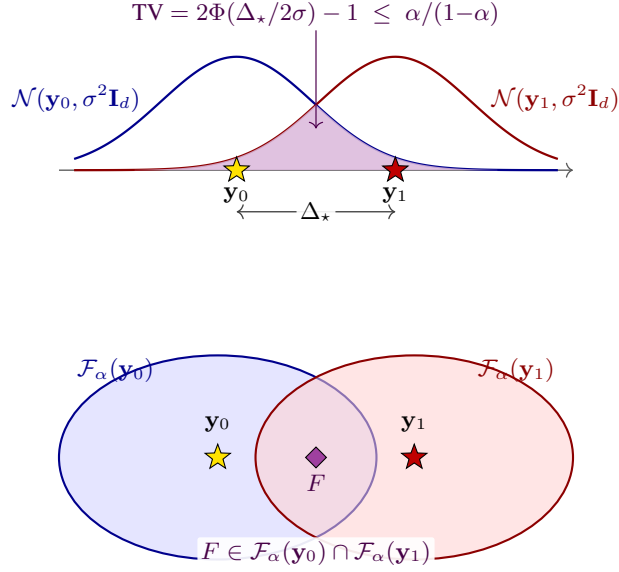
\begin{figure}[t]
\centering
\begin{tikzpicture}[
    every node/.style={font=\small},
    huberA/.style={draw=blue!55!black, line width=0.7pt,
                   fill=blue!18, fill opacity=0.55},
    huberB/.style={draw=red!55!black, line width=0.7pt,
                   fill=red!18,  fill opacity=0.55},
    truthstar/.style={star, star points=5, star point ratio=2.4,
                       fill=yellow!85!orange, draw=black,
                       line width=0.4pt, minimum size=10pt, inner sep=0pt},
    truthstarB/.style={star, star points=5, star point ratio=2.4,
                       fill=red!75!black, draw=black,
                       line width=0.4pt, minimum size=10pt, inner sep=0pt},
    commonpt/.style={diamond, fill=violet!75, draw=black,
                     line width=0.4pt, minimum size=8pt, inner sep=0pt},
    radlabel/.style={font=\small, inner sep=1.3pt, fill=white,
                     fill opacity=0.92, text opacity=1, rounded corners=1pt},
]

\begin{scope}[yshift=2.2cm]
    \def\sigA{1.0}      
    \def\Dstar{1.05}    
    \def\peakh{1.5}     
    \coordinate (yzero) at (-\Dstar, 0);
    \coordinate (yone)  at ( \Dstar, 0);

    \draw[->, line width=0.5pt, gray!70!black]
        (-3.4, 0) -- (3.4, 0) node[right, font=\footnotesize, black] {};

    \draw[blue!55!black, line width=0.85pt, smooth, samples=80,
          domain=-3.2:3.2]
        plot ({\x}, {\peakh*exp(-(\x-(-\Dstar))*(\x-(-\Dstar))/(2*\sigA*\sigA))});
    \draw[red!55!black, line width=0.85pt, smooth, samples=80,
          domain=-3.2:3.2]
        plot ({\x}, {\peakh*exp(-(\x-\Dstar)*(\x-\Dstar)/(2*\sigA*\sigA))});

    \fill[violet!28, opacity=0.85]
        plot[smooth, samples=40, domain=-2.4:0]
            ({\x}, {\peakh*exp(-(\x-\Dstar)*(\x-\Dstar)/(2*\sigA*\sigA))})
        -- plot[smooth, samples=40, domain=0:2.4]
            ({\x}, {\peakh*exp(-(\x-(-\Dstar))*(\x-(-\Dstar))/(2*\sigA*\sigA))})
        -- (2.4, 0) -- (-2.4, 0) -- cycle;

    \node[truthstar]  at (yzero) {};
    \node[truthstarB] at (yone)  {};
    \node[font=\footnotesize, anchor=north, yshift=-3pt] at (yzero) {$\vy_{0}$};
    \node[font=\footnotesize, anchor=north, yshift=-3pt] at (yone)  {$\vy_{1}$};

    \draw[<->, line width=0.5pt, gray!50!black]
        ($(yzero)+(0,-0.55)$) -- ($(yone)+(0,-0.55)$);
    \node[radlabel] at (0, -0.55) {$\Delta_{\star}$};

    \node[font=\footnotesize, text=blue!55!black, anchor=east]
        at (-2.2, 0.95) {$\gN(\vy_{0}, \sigma^{2}\mI_{d})$};
    \node[font=\footnotesize, text=red!55!black, anchor=west]
        at ( 2.2, 0.95) {$\gN(\vy_{1}, \sigma^{2}\mI_{d})$};

    \node[radlabel, text=violet!55!black, anchor=south]
        (tvlbl) at (0.0, 1.85)
        {TV $= 2\Phi(\Delta_{\star}/2\sigma)-1 \;\leq\; \alpha/(1{-}\alpha)$};
    \draw[->, line width=0.45pt, violet!60!black]
        (tvlbl.south) -- (0.0, 0.55);
\end{scope}

\begin{scope}[yshift=-1.6cm]
    \coordinate (cA) at (-1.3, 0);
    \coordinate (cB) at ( 1.3, 0);

    \draw[huberA, smooth, line width=0.8pt]
        (cA) ellipse [x radius=2.1, y radius=1.35];
    \draw[huberB, smooth, line width=0.8pt]
        (cB) ellipse [x radius=2.1, y radius=1.35];

    \node[commonpt] (Fpt) at (0, 0) {};
    \node[font=\footnotesize, anchor=north, yshift=-3pt, text=violet!55!black]
        at (Fpt) {$F$};

    \node[truthstar]  at (cA) {};
    \node[truthstarB] at (cB) {};
    \node[font=\footnotesize, anchor=south, yshift=2pt] at (-1.3, 0.15) {$\vy_{0}$};
    \node[font=\footnotesize, anchor=south, yshift=2pt] at ( 1.3, 0.15) {$\vy_{1}$};

    \node[font=\footnotesize, blue!55!black]
        at (-2.65, 1.15) {$\gF_{\alpha}(\vy_{0})$};
    \node[font=\footnotesize, red!55!black]
        at ( 2.65, 1.15) {$\gF_{\alpha}(\vy_{1})$};

    \node[radlabel, text=violet!55!black, align=center, anchor=north]
        at (0, -1.05) {$F \in \gF_{\alpha}(\vy_{0}) \cap \gF_{\alpha}(\vy_{1})$};
\end{scope}

\end{tikzpicture}
\caption{\textbf{Modulus of continuity for Theorem~\ref{thm:minimax_body}.}
\textbf{Top:}
total variation between two equal-covariance Gaussians at separation
$\Delta_{\star}$ is $2\Phi(\Delta_{\star}/2\sigma) - 1$
(\emph{dimension-free}; depends only on the line through the two
centers).
The overlap is shaded; when the overlap mass exceeds
$\alpha/(1{-}\alpha)$, the two Huber neighborhoods touch.
\textbf{Bottom:}
the contamination class
$\gF_{\alpha}(\vy) = \{(1-\alpha)\gN(\vy, \sigma^{2}\mI_{d}) + \alpha Q\}$
is depicted as a cloud of distributions around each center;
under the threshold above, a single distribution
$F \in \gF_{\alpha}(\vy_{0}) \cap \gF_{\alpha}(\vy_{1})$
is consistent with \emph{both} truths.
No estimator can distinguish $\vy_{0}$ from $\vy_{1}$ on observations
drawn from $F$, hence Le Cam's two-point bound forces minimax error
$\geq \Delta_{\star}/4 \geq \tfrac{\sqrt{2\pi}}{4}\,\sigma\alpha/(1{-}\alpha)$,
\emph{independent of $N$}.}
\label{fig:diag_thm2}
\end{figure}

\begin{table}[t]
\centering
\small
\setlength{\tabcolsep}{6pt}
\renewcommand{\arraystretch}{1.10}
\begin{tabular}{@{}l p{0.72\linewidth}@{}}
\toprule
\textbf{Result} &
\textbf{One-line statement} \\
\midrule
\multicolumn{2}{l}{\textit{\cref{sec:problem_setup} \;Problem Setup}} \\
\midrule
~\cref{asm:huber}                & Huber $\eps$-contamination model. \\
~\cref{asm:independence}         & Conditional independence of judges. \\
~\cref{asm:subgaussian}          & $\sigma$-sub-Gaussian competent noise. \\
~\cref{asm:minority}             & Minority corruption $\alpha < 1/2$. \\
~\cref{prop:variance_reduction} & Clean-jury MSE $\operatorname{tr}(\mSigma)/N$. \\
~\cref{cor:effective_jury_size}  & Effective jury size $N_{\mathrm{eff}}{=}N/(1{+}(N{-}1)\gamma)$. \\
~\cref{prop:mean_bias}          & \textsc{PoLL} bias unbounded for any $\alpha > 0$. \\
\midrule
\multicolumn{2}{l}{\textit{\cref{sec:methodology} \;Robust Panel of LLM Judges}} \\
\midrule
~\cref{ex:cross_dim}              & Cross-dimensional corruption: per-coord plausible, jointly anomalous. \\
~\cref{def:gm}                   & \textsc{RoPoLL} via geometric median. \\
~\cref{def:breakdown}            & Finite-sample breakdown point. \\
~\cref{prop:gm_properties}      & GM existence, uniqueness, equivariance, breakdown $1/2$. \\
~\cref{alg:ropoll}               & Modified Weiszfeld iteration: $O(Nd\log(1/\eps))$. \\
\midrule
\multicolumn{2}{l}{\textit{\cref{sec:theory} \;Theoretical Guarantees}} \\
\midrule
~\cref{lem:gm_breakdown}         & GM within $C_{\alpha}r$ if $(1{-}\alpha)$ fraction of points are within $r$. \\
~\cref{lem:cluster_radius}       & Sub-Gaussian cluster radius $\rho{=}\sigma(C_{1}\sqrt{d}{+}\sqrt{\log(1/\beta)/c})$. \\
~\cref{thm:ropoll_bound}         & \textsc{RoPoLL} error $\leq C_{\alpha+\beta}\rho$ w.p.\ $1{-}\exp({-}N\beta^{2}/2)$. \\
~\cref{lem:correlated_jury}      & Same bound under equicorrelated juries, w.p.\ $1{-}1/(\beta^{2}N_{\mathrm{eff}})$. \\
~\cref{thm:minimax_body}         & Minimax lower bound $\Omega(\sigma(\sqrt{d/N}{+}\alpha/(1{-}\alpha)))$. \\
\bottomrule
\end{tabular}
\vspace{1em}
\caption{\centering \textbf{Roadmap of formal results.}
Each row links to the result's full statement (clickable
reference); proofs are deferred to the \cref{sec:appendix_theory}}
\label{tab:theory_roadmap}
\end{table}

\section{Theoretical Guarantees}
\label{sec:theory}
\subsection{Finite-Sample Error Bound}
\label{sec:upper_bound}
\begin{figure}[t]
\centering
\begin{subfigure}[t]{0.32\textwidth}
    \centering
    \includegraphics[width=\linewidth]{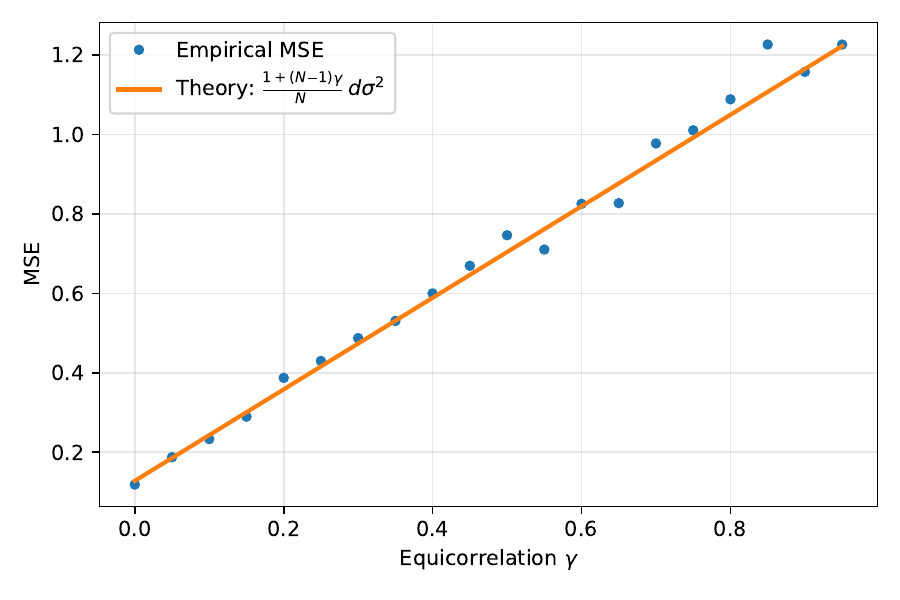}
    \caption{\cref{cor:effective_jury_size}.}
    \label{fig:tv_gamma}
\end{subfigure}\hfill
\begin{subfigure}[t]{0.32\textwidth}
    \centering
    \includegraphics[width=\linewidth]{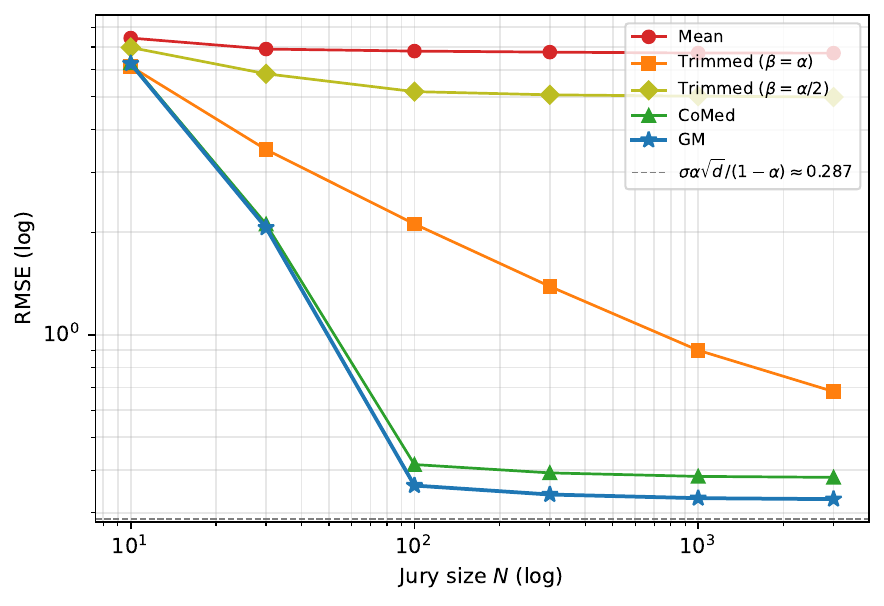}
    \caption{\cref{thm:ropoll_bound}.}
    \label{fig:tv_n_sweep}
\end{subfigure}\hfill
\begin{subfigure}[t]{0.32\textwidth}
    \centering
    \includegraphics[width=\linewidth]{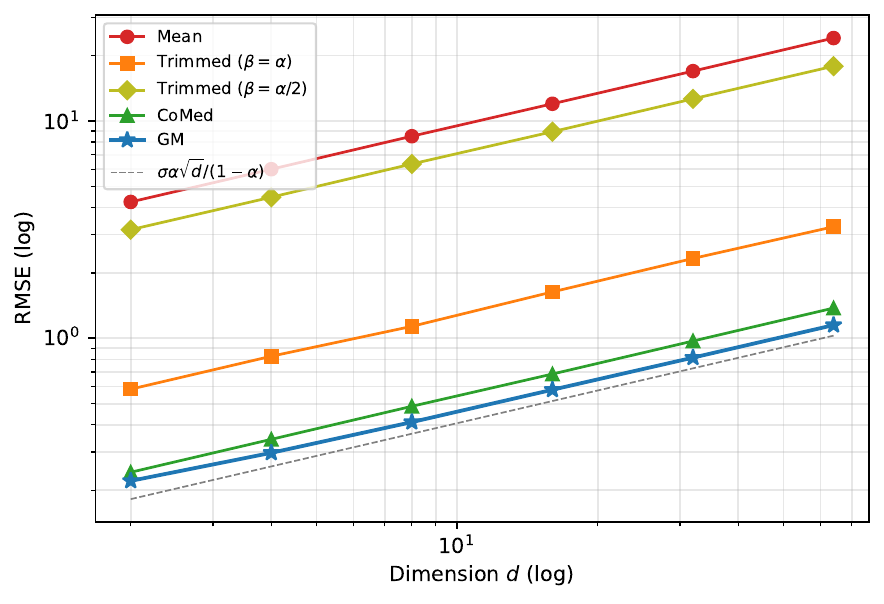}
    \caption{\cref{lem:cluster_radius}.}
    \label{fig:tv_d_sweep}
\end{subfigure}
\caption{\textbf{Theory validation.}
\textbf{(a)} Empirical MSE of the arithmetic mean for $N{=}10$
equicorrelated clean Gaussian judges matches the closed-form
prediction $\tfrac{1+(N{-}1)\gamma}{N}\,d\sigma^{2}$
(Corollary~\ref{cor:effective_jury_size}) across the full range
$\gamma \in [0, 0.95]$;
the effective jury size $N_{\mathrm{eff}} = N/(1+(N{-}1)\gamma)$
saturates at $1/\gamma$, motivating the three-judge committees of
\S\ref{sec:benchmark}.
\textbf{(b)} Under worst-case Huber contamination
($\alpha{=}0.3$, $\sigma{=}0.3$, $d{=}5$, Dirac corruption at
$\vy^{\star}+10\,\vone$), the geometric median converges to the
predicted breakdown floor $\sigma\alpha\sqrt{d}/(1{-}\alpha) \approx 0.287$
(gray dashed) as $N$ grows, matching
Theorem~\ref{thm:ropoll_bound};
the arithmetic mean and under-trimmed mean
($\beta{=}\alpha/2$) plateau above the floor, confirming
Proposition~\ref{prop:mean_bias}.
\textbf{(c)} Holding $N{=}1000$ fixed and sweeping the dimension,
the geometric median tracks the predicted $\sqrt{d}$ scaling of the
cluster radius (Lemma~\ref{lem:cluster_radius}) to within an
absolute constant.}
\label{fig:theory_validation}
\end{figure}

We bound $\|\hat{\vy}_{\mathrm{GM}} - \vy^{\star}\|_{2}$ under our
Huber model in two steps:
a deterministic geometric lemma about the geometric median
(Lemma~\ref{lem:gm_breakdown}, due to
\citet{minsker2015geometric}), and a probabilistic lemma that
controls the sub-Gaussian cluster radius
(Lemma~\ref{lem:cluster_radius}).
Plugging the cluster radius into the geometric lemma yields
Theorem~\ref{thm:ropoll_bound}, our finite-sample upper bound on
\textsc{RoPoLL}.
Full proofs are in
Appendix~\ref{sec:appendix_theory}.

\begin{lemma}[\bf Geometric Breakdown of GM]
\label{lem:gm_breakdown}
\emph{(\citet{minsker2015geometric}, Lemma~2.1; building on
\citet{lopuhaa1991breakdown}.)}
Let $x_{1}, \ldots, x_{k} \in \R^{d}$ and let $x_{*}$ be any
minimizer of $z \mapsto \sum_{j=1}^{k} \|z - x_{j}\|_{2}$
(a geometric median).
Fix $\alpha \in (0, 1/2)$, $r > 0$, and $z \in \R^{d}$.
If $|\{j : \|x_{j} - z\|_{2} \leq r\}| \geq (1-\alpha)k$, then
\begin{equation}
\label{eq:gm_breakdown}
    \big\|x_{*} - z\big\|_{2} \;\leq\; C_{\alpha}\, r,
    \qquad
    C_{\alpha} \;\triangleq\; \frac{1-\alpha}{\sqrt{1-2\alpha}}.
\end{equation}
\end{lemma}

This is purely deterministic: a multiplicative bound between the
geometric median and any target $z$ in terms of how concentrated the
inputs are around $z$.
The constant $C_{\alpha}$ is sharp and diverges as
$\alpha \to 1/2$, matching the breakdown point of GM.
\Cref{fig:diag_lem1} illustrates the geometric core of the proof:
under the contradiction hypothesis $\Delta = \|x_{*}-z\|_{2} > C_{\alpha}r$,
the cluster ball subtends a narrow cone at $x_{*}$, forcing the
$(1{-}\alpha)k$ cluster points to lie inside it---and a balance
of unit-vector subgradients (the first-order optimality condition
for the geometric median) produces the contradiction.

\begin{lemma}[\bf Sub-Gaussian Cluster Radius]
\label{lem:cluster_radius}
Under Assumptions~\ref{asm:huber}--\ref{asm:minority}, let
$\beta \in (0,\, 1/2 - \alpha)$ be a slack parameter.
With probability at least $1 - \exp(-N\beta^{2}/2)$,
\begin{equation}
\label{eq:cluster_radius}
    \big|\big\{i \in [N] : \|\hat{\vy}_{i} - \vy^{\star}\|_{2} \leq \rho\big\}\big|
    \;\geq\; (1 - \alpha - \beta)\,N,
\end{equation}
where the cluster radius is
\begin{equation}
\label{eq:cluster_radius_rho}
    \rho \;=\; \sigma\!\left(C_{1}\sqrt{d} \;+\; \sqrt{\frac{1}{c}\,\log\!\frac{2(1-\alpha)}{\beta}}\right),
\end{equation}
and $C_{1}, c > 0$ are absolute constants (from the
sub-Gaussian-norm tail bound, Step~1 of the proof in
Appendix~\ref{sec:lem_cluster_radius}).
\end{lemma}

The slack $\beta$ controls a trade-off: a larger $\beta$ permits a
smaller cluster radius $\rho$ (since fewer competent samples need to
be inside) but augments the effective contamination threshold from
$\alpha$ to $\alpha + \beta$ in the geometric step.

\begin{theorem}[\bf \textsc{RoPoLL} Breakdown Bound under Huber Contamination]
\label{thm:ropoll_bound}
Under Assumptions~\ref{asm:huber}--\ref{asm:minority}, fix any slack
$\beta \in (0,\, 1/2 - \alpha)$.
With probability at least $1 - \exp(-N\beta^{2}/2)$,
\begin{equation}
\label{eq:ropoll_bound}
    \big\|\hat{\vy}_{\mathrm{GM}} - \vy^{\star}\big\|_{2}
    \;\leq\;
    \underbrace{\frac{1 - \alpha - \beta}{\sqrt{1 - 2\alpha - 2\beta}}}_{C_{\alpha+\beta}}
    \;\cdot\;
    \underbrace{\sigma\!\left(C_{1}\sqrt{d} \;+\; \sqrt{\frac{1}{c}\log\!\frac{2(1-\alpha)}{\beta}}\right)}_{\rho}.
\end{equation}
\end{theorem}

The proof is a one-line combination: applying
Lemma~\ref{lem:gm_breakdown} with $k = N$, $z = \vy^{\star}$,
$r = \rho$, and effective threshold $\alpha' = \alpha + \beta < 1/2$,
the count bound from Lemma~\ref{lem:cluster_radius} is exactly the
hypothesis of Lemma~\ref{lem:gm_breakdown}, so the geometric lemma
gives $\|\hat{\vy}_{\mathrm{GM}} - \vy^{\star}\|_{2} \leq C_{\alpha+\beta}\,\rho$.
Full details are in Appendix~\ref{sec:thm_ropoll_huber_breakdown}.

\paragraph{Interpretation.}
The bound has two components.
The geometric constant $C_{\alpha+\beta}$ depends only on the
contamination rate (and slack); it diverges as
$\alpha + \beta \to 1/2$, encoding the breakdown point.
The cluster radius $\rho$ depends only on the noise scale $\sigma$
and the dimension $d$;
it does \emph{not} shrink with $N$.
This reflects the breakdown-point character of plain GM: under
arbitrary $Q$ in the Huber class, the asymptotic-$N$ floor is set
by the cluster radius, not by sample averaging.
The bound is distribution-free over the corruption class
$\{Q_{i}\}$.
\Cref{fig:diag_thm1} illustrates the two-step geometry.
For comparison with the matching minimax lower bound
(Theorem~\ref{thm:minimax_body}), see
\S\ref{sec:minimax_body}.

\paragraph{Synthetic validation.}
\Cref{fig:theory_validation} validates the i.i.d.\ theory on
controlled synthetic data:
panel~(a) confirms the closed-form clean-jury MSE of
Corollary~\ref{cor:effective_jury_size};
panel~(b) shows the geometric median converging to the predicted
breakdown floor $\sigma\alpha\sqrt{d}/(1-\alpha)$ as $N$ grows under
worst-case Huber contamination, while \textsc{PoLL} stays above
the floor (Proposition~\ref{prop:mean_bias});
panel~(c) confirms the $\sqrt{d}$ scaling of the cluster radius
$\rho$ (Lemma~\ref{lem:cluster_radius}) at fixed $N$.

\paragraph{Beyond i.i.d.: equicorrelated juries.}
Theorem~\ref{thm:ropoll_bound} assumes conditional independence
(Assumption~\ref{asm:independence}).
Real LLM juries trained on overlapping corpora violate this:
inter-judge correlation $\bar\gamma \in [0.3, 0.7]$ is typical
(\Cref{fig:corpus_judge_corr,sec:appendix_dataset_correlation}).
We close §\ref{sec:upper_bound} by extending
Theorem~\ref{thm:ropoll_bound} to this regime: the breakdown
structure ($C_{\alpha+\beta}$ and $\rho$) survives unchanged;
only the high-probability event weakens, from exponential in $N$
to polynomial in the \emph{effective jury size}
$N_{\mathrm{eff}} = N/(1+(N-1)\bar\gamma_{W})$ familiar from
Corollary~\ref{cor:effective_jury_size}.

\begin{lemma}[\bf \textsc{RoPoLL} under Equicorrelated Juries]
\label{lem:correlated_jury}
Replace Assumption~\ref{asm:independence} with the weaker
\emph{equicorrelated-indicator} assumption: for the cluster
indicators
$W_{i} \triangleq \mathbb{1}\{Z_{i}=0,\ \|\hat{\vy}_{i}-\vy^{\star}\|_{2} \leq \rho_{p}\}$
of Lemma~\ref{lem:cluster_radius},
$\Cov(W_{i}, W_{j}) \leq \bar\gamma_{W}\sqrt{\Var(W_{i})\Var(W_{j})}$
for all $i \neq j$, with $\bar\gamma_{W} \in [0, 1]$.
Under Assumptions~\ref{asm:huber}, \ref{asm:subgaussian},
\ref{asm:minority} and the equicorrelated-indicator assumption, fix
any slack $\beta \in (0, 1/2 - \alpha)$.
With probability at least
\begin{equation}
\label{eq:correlated_prob}
    1 - \frac{1}{\beta^{2}\,N_{\mathrm{eff}}},
    \qquad
    N_{\mathrm{eff}}
    \;\triangleq\;
    \frac{N}{1+(N-1)\bar\gamma_{W}},
\end{equation}
the bound
$\|\hat{\vy}_{\mathrm{GM}} - \vy^{\star}\|_{2} \leq C_{\alpha+\beta}\,\rho$
of Theorem~\ref{thm:ropoll_bound} holds, with $C_{\alpha+\beta}$ and
$\rho$ unchanged.
At $\bar\gamma_{W}=0$, the equicorrelated assumption reduces to
independence and Theorem~\ref{thm:ropoll_bound}'s exponential bound
$\exp(-N\beta^{2}/2)$ recovers (which is strictly tighter than
\eqref{eq:correlated_prob}).
\end{lemma}

The proof replaces the Hoeffding step of
Lemma~\ref{lem:cluster_radius} with a Chebyshev bound on
$\sum_{i}W_{i}$ under the bounded-covariance hypothesis;
the cluster radius $\rho$ and the geometric constant $C_{\alpha+\beta}$
are unchanged because Lemma~\ref{lem:gm_breakdown} is purely
deterministic.
The hypothesis is on the indicator correlation $\bar\gamma_{W}$,
which we estimate directly from the empirical co-incidence of
cluster events $\{W_{i} = 1\}$ on our $13$-judge experimental
panels (\Cref{sec:appendix_gamma_W,tab:empirical_gamma_W}):
across three cluster-radius calibrations
(50th/70th/90th percentiles of pooled per-sample deviations),
$\bar\gamma_{W} \in [0.45, 0.53]$ on HelpSteer-2 and
UltraFeedback, in line with the inter-judge \emph{score}
correlations $\bar\gamma \in [0.49, 0.71]$ of
\Cref{fig:corpus_judge_corr}.
At $N{=}3$ this gives $N_{\mathrm{eff}} \in [1.45, 1.58]$;
the breakdown floor and geometric constant are unaffected, and the
high-probability event remains non-trivial for the moderate slack
$\beta$ used in practice.
Full details are in Appendix~\ref{sec:lem_correlated_jury}.

\subsection{Minimax Lower Bound}
\label{sec:minimax_body}
We close the section with a matching information-theoretic minimax
lower bound that exposes a $\sqrt{d}$
statistical--computational gap: \textsc{RoPoLL} is rate-optimal up
to a dimensional constant on the breakdown floor, the price for
$O(Nd\log(1/\eps))$ tractability via the Weiszfeld iteration.
No estimator can do better on the parametric variance term;
the breakdown floor, in turn, is matched up to a dimensional
constant only by the (intractable) Tukey halfspace median.

\begin{theorem}[\bf Minimax Lower Bound]
\label{thm:minimax_body}
Under the same assumptions as Theorem~\ref{thm:ropoll_bound},
\begin{equation}
\label{eq:minimax_body}
    \inf_{\hat{\vy}}\;\sup_{F \in \gF_{\alpha,\sigma}}\;
    \E_{F}\!\left[\|\hat{\vy} - \vy^{\star}\|_{2}\right]
    \;\geq\;
    c\sigma\!\left(\sqrt{d/N} \;+\; \frac{\alpha}{1-\alpha}\right),
\end{equation}
where the infimum is over all measurable estimators of the form
$\hat{\vy}(\hat{\vy}_{1}, \ldots, \hat{\vy}_{N})$, $\gF_{\alpha,\sigma}$
is the class of joint distributions consistent with
Assumptions~\ref{asm:huber}, \ref{asm:independence},
\ref{asm:subgaussian}, and \ref{asm:minority}, and $c > 0$ is a
universal constant.
\end{theorem}

\begin{proof}[Proof sketch]
We use Le Cam's two-point method.
For the $\sqrt{d/N}$ term, set $\alpha = 0$ and pit two clean
Gaussian hypotheses $\vy_{0} = \vzero$ vs.\
$\vy_{1} = \Delta\,\ve_{1}$ with $\Delta = c_{1}\sigma\sqrt{d/N}$;
the standard Pinsker-plus-tensorisation calculation gives
$\mathrm{TV}(F_{0}^{\otimes N}, F_{1}^{\otimes N}) \leq 1/2$.
For the $\alpha/(1-\alpha)$ term, we exploit the
\emph{modulus of continuity} of the Huber neighborhood:
the contamination class
$\gF_{\alpha}(\vy) = \{(1-\alpha)\gN(\vy, \sigma^{2}\mI_{d}) + \alpha Q\}$
contains a common distribution at two centers $\vy_{0}, \vy_{1}$
whenever $\|\gN(\vy_{0}, \sigma^{2}\mI_{d}) - \gN(\vy_{1}, \sigma^{2}\mI_{d})\|_{\mathrm{TV}} \leq \alpha/(1-\alpha)$,
since both $(1-\alpha)\gN(\vy_{0}, \sigma^{2}\mI_{d})$ and
$(1-\alpha)\gN(\vy_{1}, \sigma^{2}\mI_{d})$ are then dominated
componentwise by a single probability measure.
TV between two equal-covariance Gaussians at separation $\Delta$ is
$2\Phi(\Delta/(2\sigma)) - 1$ (\emph{dimension-free}), so the largest
indistinguishable separation is
$\Delta_{\star} = 2\sigma\Phi^{-1}\!\left(\tfrac{1}{2} + \tfrac{\alpha}{2(1-\alpha)}\right) \geq \sqrt{2\pi}\,\sigma\alpha/(1-\alpha)$.
Le Cam then gives error at least $\Delta_{\star}/4$.
The full proof is in
Appendix~\ref{sec:appendix_minimax}.
\end{proof}

\Cref{fig:diag_thm2} illustrates the modulus-of-continuity
construction: when the two Gaussians at $\vy_{0}$ and $\vy_{1}$ are
TV-close enough, their Huber neighborhoods overlap at a common
distribution $F$ that is consistent with both truths---making the
two centers indistinguishable from any number of i.i.d.\ samples.

\paragraph{Comparison with the upper bound.}
Theorems~\ref{thm:ropoll_bound} and~\ref{thm:minimax_body} match
exactly on the parametric rate $\sigma\sqrt{d/N}$.
On the breakdown floor they differ by a $\sqrt{d}$ factor: the upper
bound scales as $C_{\alpha}\sigma\sqrt{d}$ while the lower bound
scales as $\sigma\alpha/(1-\alpha)$.
This is not a slack in the analysis but a real
statistical--computational gap.
The minimax-optimal estimator on the breakdown floor is the
\emph{Tukey halfspace median}
\citep{tukey1975mathematics,donoho1992breakdown}, whose exact
computation is NP-hard for $d \geq 3$
\citep{johnson1978densest,aloupis2006geometric};
the smoothed-depth estimator of \citet{chen2018robust} matches the
$\sigma\alpha$ floor in sub-exponential time.
The geometric median is the polynomial-time alternative: it shares
the optimal $1/2$ breakdown point but pays a $\sqrt{d}$ price for
$O(Nd\log(1/\eps))$ tractability via the Weiszfeld iteration
(\S\ref{sec:weiszfeld}).
For LLM juries the trade is favourable: $d$ is small (1--5 in our
benchmarks) so the $\sqrt{d}$ overhead is at most $\sim 2.2\times$,
and at $N = 3$ the variance term $\sigma\sqrt{d/N}$ dominates the
breakdown floor on every regime we test
(\S\ref{sec:exp_method_contrast}).
\textsc{RoPoLL} is therefore an \emph{efficient} robust estimator,
matching the minimax breakdown point at the small price of a
dimensional constant on the contamination floor.

\paragraph{Scope of the i.i.d.\ assumptions.}
Theorems~\ref{thm:ropoll_bound} and~\ref{thm:minimax_body} are
stated under the i.i.d.\ baseline (Asm~\ref{asm:independence},
identical $\sigma_{i}$).
Lemma~\ref{lem:correlated_jury} relaxes independence to
equicorrelation, covering the
inter-judge correlation $\bar\gamma \in [0.3, 0.7]$ measured in our
experiments.
Two further deviations remain out of scope: per-judge
\emph{heterogeneity} ($\sigma_{i}, \alpha_{i}$ varying across the
$4$--$675$\,B parameter range) and \emph{explicit dependence} by
design (peer-rank discussion \citep{li2024prd}, multi-agent debate
\citep{chan2024chateval}, judge networks \citep{zhang2024wider}).
The empirical evaluation in \S\ref{sec:benchmark} holds these axes
fixed and isolates the effect of contamination type and rate;
the heterogeneous-jury extension---together with side information
such as per-judge calibration on a labeled validation slice---is
the subject of follow-up work.

\section{Experiments}
%
\label{sec:benchmark}

We evaluate \textsc{RoPoLL} against \textsc{PoLL} (the arithmetic-mean
baseline of \citet{verga2024replacing}) and the coordinate-wise
\textsc{Median} on three reward-model benchmarks under a per-case
corruption pipeline that exposes the corruption-type dependence
predicted by Theorem~\ref{thm:ropoll_bound} and
Example~\ref{ex:cross_dim}.

\begin{figure*}[t]
\centering
\includegraphics[width=0.95\textwidth]{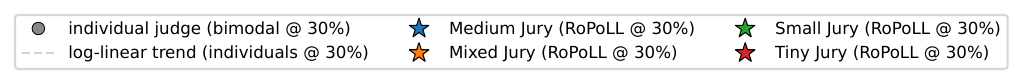}\\[2pt]
\begin{subfigure}{0.33\textwidth}
    \centering
    \includegraphics[width=\textwidth]{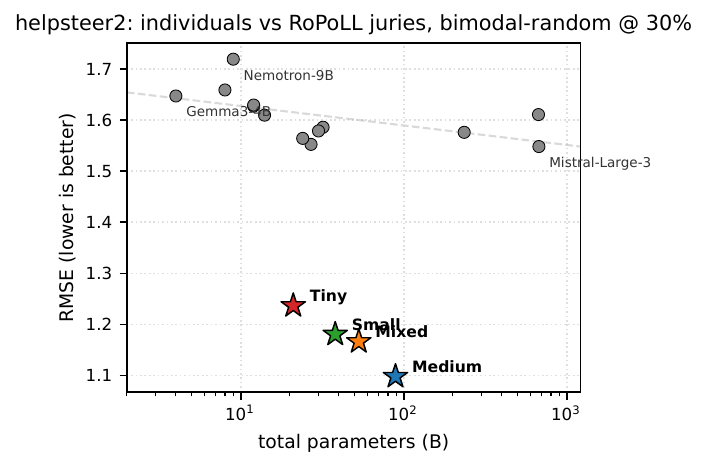}
    \caption{HelpSteer\,2}
\end{subfigure}\hfill
\begin{subfigure}{0.33\textwidth}
    \centering
    \includegraphics[width=\textwidth]{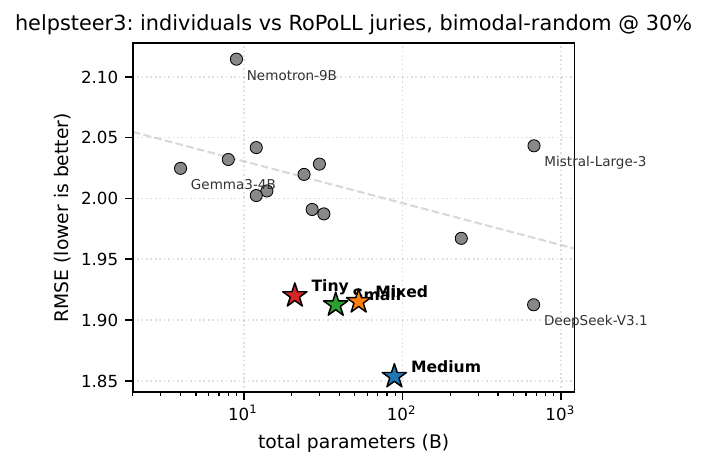}
    \caption{HelpSteer\,3}
\end{subfigure}\hfill
\begin{subfigure}{0.33\textwidth}
    \centering
    \includegraphics[width=\textwidth]{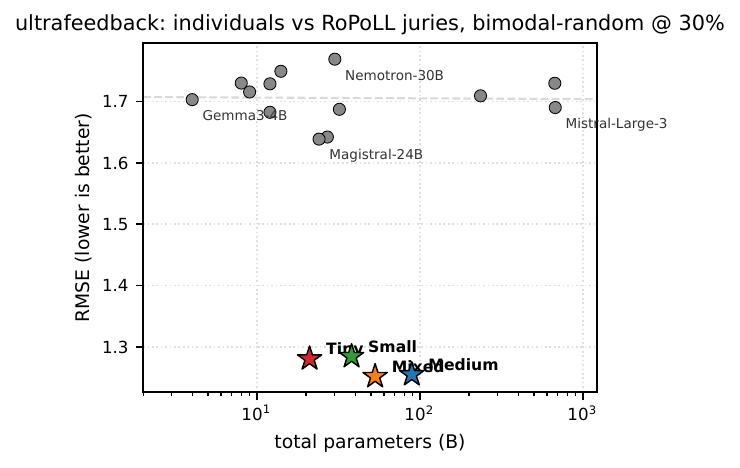}
    \caption{UltraFeedback}
\end{subfigure}
\caption{\textbf{Parameter efficiency of \textsc{RoPoLL} juries vs.\
individual judges under \texttt{bimodal-random} corruption at
$r = 30\%$.}
RMSE vs.\ parameter count (log scale) for each dataset; gray circles
are the 13 individual open-weight judges (four anchors labelled),
dashed line is their log-linear scaling fit, and coloured stars mark
the four \textsc{RoPoLL} juries (Medium/Mixed/Small/Tiny) at their
aggregate parameter budget---all evaluated under identical $30\%$
per-case corruption. \texttt{bimodal-random} drives each coordinate of
a corrupted score independently to an extremum, instantiating the
cross-dimensional failure mode of Example~\ref{ex:cross_dim};
clean-baseline and \texttt{zeros} counterparts are in
\Cref{fig:p5_jury_vs_individual_clean,fig:p3_hero_rmse_zeros}.}
\label{fig:p2_bimodal_ropoll_hero}
\end{figure*}

\subsection{Setup}
\label{sec:benchmark_setup}

\paragraph{Datasets.}
We use three popular reward model benchmarks with complementary
ground-truth sources.
\textbf{HelpSteer\,2}~\citep{wang2024helpsteer2} contributes
$1{,}000$ samples drawn uniformly at random from the validation
split, each rated on a $0$--$4$ Likert scale across five attributes
(helpfulness, correctness, coherence, complexity, verbosity) by
trained human annotators.
\textbf{HelpSteer\,3}~\citep{wang-etal-2025-helpsteer3} contributes
its full $2{,}017$-sample multilingual validation split;
the native chosen-vs-rejected preference is converted to a scalar
\texttt{overall\_preference} target on $[-4, 4]$ by re-scoring both
responses on the HelpSteer\,2 rubric and taking the signed
difference.
\textbf{UltraFeedback}~\citep{cui2024ultrafeedback} contributes
$1{,}000$ samples scored on a $1$--$5$ scale across four attributes
(helpfulness, honesty, instruction following, truthfulness) using
GPT-4 as the reference annotator.

\paragraph{Judges and juries.}
We score every sample with $13$ open-weight judges spanning
$4$--$675$\,B parameters at temperature $0$ under a shared structured
rubric: Mistral-Large-3 ($675$\,B), DeepSeek-V3.1 ($671$\,B),
Qwen3-235B, Qwen3-32B, Nemotron-30B, Gemma-27B, Magistral-Small
($24$\,B), Ministral-14B, Gemma-12B, Nemotron-12B, Nemotron-9B,
Ministral-8B, and Gemma-4B.
From these we curate four three-judge committees that trade size
against compute:
\textsc{Medium}~$\approx 89$\,B (Qwen3-32B, Nemotron-30B, Gemma-27B),
\textsc{Mixed}~$\approx 53$\,B (Qwen3-32B, Gemma-12B, Nemotron-9B),
\textsc{Small}~$\approx 38$\,B (Ministral-14B, Gemma-12B,
Nemotron-12B), and
\textsc{Tiny}~$\approx 21$\,B (Nemotron-9B, Ministral-8B, Gemma-4B).
The choice of $N = 3$ is not arbitrary: under the equicorrelated jury
model of Corollary~\ref{cor:effective_jury_size} the effective jury
size $N_{\mathrm{eff}} = N/(1 + (N-1)\gamma)$ saturates at $1/\gamma$
as $N \to \infty$, so for the moderate inter-judge correlation
$\gamma \in [0.3, 0.5]$ characteristic of diverse but non-orthogonal
LLM backbones, $N_{\mathrm{eff}}$ saturates already at
$N \approx 2$--$3$
(\Cref{fig:tv_gamma} of \Cref{fig:theory_validation},
\S\ref{sec:upper_bound})---a prediction corroborated by the
empirical diminishing-returns knee at $N = 3$ in
\Cref{fig:p6_ablation_N_rmse}.
We compare \textsc{PoLL} (arithmetic mean), the coordinate-wise
\textsc{Median}, and \textsc{RoPoLL} (Algorithm~\ref{alg:ropoll}) on
these committees;
all three operate on the same three (possibly corrupted) score
vectors per sample.

\paragraph{Per-case corruption protocol.}
Rather than injecting extra adversarial judges into a fixed pool, we
hold the jury size at three and corrupt individual (sample, judge)
\emph{cells} at a per-case rate
$r \in \{0\%, 10\%, 20\%, 30\%, 40\%, 50\%\}$, matching the realistic
failure pattern in which a judge occasionally emits a bad score on a
specific input.
The sweep range is calibrated against the natural-failure
characterization of \Cref{fig:corpus_failure_rates}
(\S\ref{sec:huber}): naturally-occurring rates span $0.59\%$ on
HelpSteer\,2 to $33\%$ on HelpSteer\,3 multilingual, so $r \in [0, 50\%]$
covers the natural regime and stress-tests beyond.
We consider four corruption types covering distinct adversarial
regimes:
\emph{(i)} \texttt{zeros}, where every corrupted slot is replaced by
$\vzero$ (the parser-failure fallback);
\emph{(ii)} \texttt{inverted}, where corrupted slots are replaced by
$K \cdot \vone - \vy^{\star}$ (the worst-case anti-correlated
Byzantine attack);
\emph{(iii)} \texttt{bimodal-random}, where each coordinate of the
corrupted slot is independently set to $0$ or $K$ with equal
probability (the cross-dimensional failure mode of
Example~\ref{ex:cross_dim});
and \emph{(iv)} \texttt{cauchy-far}, where each corrupted slot is
drawn as
$\vy^{\star} + 10 + 2(s_{\max}\!-\!s_{\min})\!\cdot\!\vt$ with $\vt$
component-wise standard Cauchy (a biased heavy-tailed Byzantine
attack with undefined mean and variance).
Pre-existing parser failures from real judges are dropped at
$r = 0\%$ and replaced with adversarial vectors at $r > 0\%$, so the
effective observed corruption rate at $r > 0$ is $f + (1 - f)\,r$
with $f$ the naturally occurring parser-failure rate.
We report RMSE against the reference labels;
per-judge calibration breakdowns (MAE, mean bias) on the
UltraFeedback rubric dimensions are in
Appendix~\ref{sec:appendix_per_model}.

\subsection{Heavy-Tailed Corruption}
\label{sec:exp_cauchy}

The \texttt{cauchy-far} attack is the empirical analogue of the
adversarial choice in Proposition~\ref{prop:mean_bias}: each
corrupted slot has an unbounded first moment, and a single
contaminated judge can in principle drag \textsc{PoLL} by an
arbitrary amount.
The teaser (\Cref{fig:p1_cauchy_hero}, queued in
\S\ref{sec:intro}) confirms this empirically.
On the \textsc{Medium} jury, \textsc{PoLL}'s RMSE exceeds
\textsc{RoPoLL}'s by one to three orders of magnitude at every
$r \geq 10\%$ on all three benchmarks.
The largest gap occurs on HelpSteer\,2 at $r = 40\%$, where
\textsc{PoLL}'s RMSE reaches $\approx 4{,}951$ while \textsc{RoPoLL}
holds at $\approx 9.2$---a ratio of $\approx 540\times$.
The coordinate-wise \textsc{Median} is competitive with
\textsc{RoPoLL} here (full three-method comparison in
\Cref{fig:p7_method_contrast}, \S\ref{sec:exp_method_contrast}):
under heavy-tailed Byzantine attacks \emph{any} robust aggregator
beats the mean, exactly as the theory predicts.

\subsection{Cross-Dimensional Corruption}
\label{sec:exp_bimodal}
The \texttt{bimodal-random} attack drives each coordinate of a
corrupted score independently to an extremum---each corrupted score
is plausible per coordinate but jointly anomalous, the failure mode
predicted by Example~\ref{ex:cross_dim}.
\Cref{fig:p2_bimodal_ropoll_hero} plots, for each benchmark, the
$13$ individual judges (gray circles) against their parameter count
alongside the four \textsc{RoPoLL} committees (coloured stars) at
their aggregate parameter count, both evaluated under identical
$30\%$ \texttt{bimodal-random} corruption.
On HelpSteer\,2 and UltraFeedback, \emph{all four} \textsc{RoPoLL}
committees sit visibly below the individual-at-30\% scaling trend.
The headline number: on HelpSteer\,2, the \textsc{Small} committee
at $38$\,B reaches RMSE $= 1.18$, beating Mistral-Large-3's $1.55$
at $675$\,B---a $1.31\times$ accuracy advantage at $18\times$ fewer
parameters.
On the harder HelpSteer\,3 (signed-preference target), the
\textsc{Medium} committee at $89$\,B still matches DeepSeek-V3.1
($671$\,B) at RMSE $= 1.85$.

\paragraph{Compute-matched comparison.}
The single-judge-vs-committee framing above understates the case
for robust aggregation, because \emph{any} 3-judge committee
($114$\,B--$267$\,B forward-pass compute) beats a single $675$\,B
judge at evaluation cost.
The fair compute-matched comparison is \textsc{RoPoLL} versus
\textsc{PoLL} on the same committee.
On the \textsc{Small} committee at $30\%$ \texttt{bimodal-random},
\textsc{RoPoLL} (RMSE $1.18$) beats \textsc{PoLL} (RMSE $\approx 1.45$)
on the same three judges by $\approx 19\%$ at \emph{identical}
inference cost (\Cref{fig:p7_method_contrast}, full grid):
the parameter-efficiency advantage holds because the geometric
median's joint-distance objective extracts more signal from the
same forward passes than \textsc{PoLL}.
\emph{Robust aggregation, not the ensemble itself, is what
delivers the win.}

\subsection{Bounded Mean-Preserving Corruptions: Zeros and Inverted}
\label{sec:exp_bounded}

\begin{figure}[t]
\centering
\includegraphics[width=0.95\textwidth]{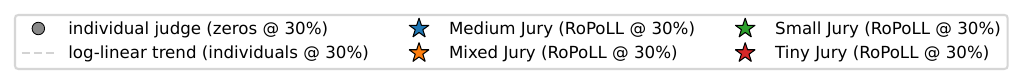}\\[2pt]
\begin{subfigure}{0.33\textwidth}
    \centering
    \includegraphics[width=\textwidth]{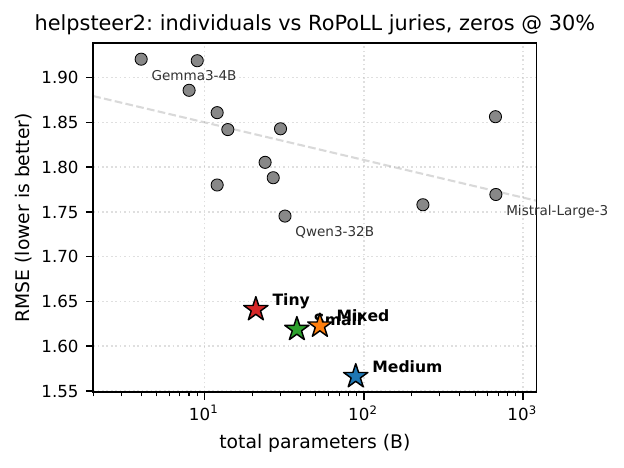}
    \caption{HelpSteer\,2}
\end{subfigure}\hfill
\begin{subfigure}{0.33\textwidth}
    \centering
    \includegraphics[width=\textwidth]{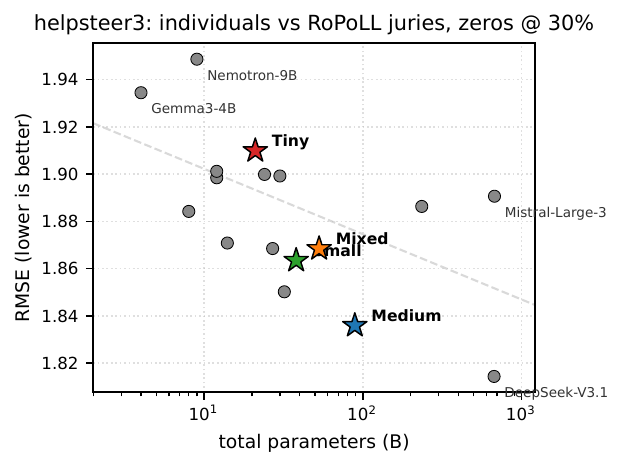}
    \caption{HelpSteer\,3}
\end{subfigure}\hfill
\begin{subfigure}{0.33\textwidth}
    \centering
    \includegraphics[width=\textwidth]{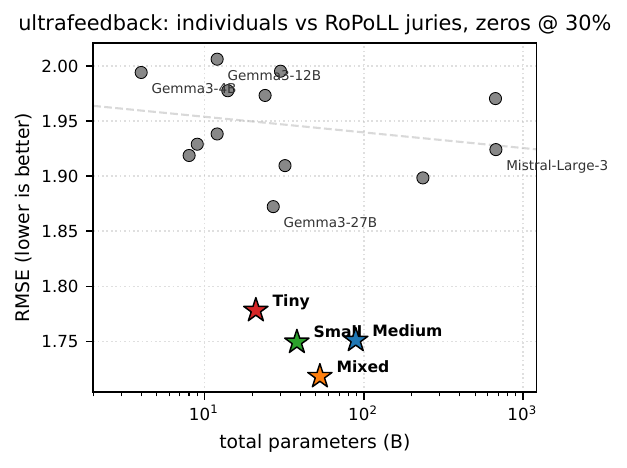}
    \caption{UltraFeedback}
\end{subfigure}
\caption{\textbf{Parameter efficiency of \textsc{RoPoLL} juries vs.\
individual judges under \texttt{zeros} corruption at $r = 30\%$.}
RMSE vs.\ parameter count (log scale) for each dataset; gray circles
are the 13 individual open-weight judges (four anchors labelled),
dashed line is their log-linear scaling fit, and coloured stars mark
the four \textsc{RoPoLL} juries at their aggregate parameter budget,
all under identical $30\%$ per-case corruption. \texttt{zeros}
replaces each corrupted slot with the parser-fallback vector
$\vzero$; the direct \textsc{RoPoLL} vs.\ \textsc{PoLL} contrast is
deferred to \Cref{fig:p7_method_contrast}.}
\label{fig:p3_hero_rmse_zeros}
\end{figure}

The \texttt{zeros} and \texttt{inverted} attacks place corrupted
scores at fixed bounded points on the score scale.
On a bounded scale \textsc{PoLL} is \emph{mean-preserving} under
uniform-rate corruption when the corrupted point happens to lie at
the scale midpoint---an empirical accident, not a property of the
mean as an estimator---which makes these the regimes where
\textsc{RoPoLL} and \textsc{PoLL} should be hardest to separate.
\Cref{fig:p3_hero_rmse_zeros} plots the parameter-efficiency view
under $30\%$ \texttt{zeros} corruption on the same axes as
\Cref{fig:p2_bimodal_ropoll_hero}.
The headline: even in this favourable-to-the-mean regime, all four
\textsc{RoPoLL} committees sit at or below the individual-at-30\%
scaling line on every benchmark, and the gap to \textsc{PoLL}
remains positive but small ($\leq 0.3$ RMSE for the
\textsc{Medium} jury across the full corruption sweep---see
\Cref{fig:p7_method_contrast}).
\textsc{RoPoLL} is therefore not a universal replacement for the
mean: when the corruption is bounded and happens to be
mean-preserving, the two are within an insurance premium of each
other.
The argument for \textsc{RoPoLL} as a default is that the practitioner
does not get to choose which regime the next corruption falls into.

\subsection{Clean-Baseline Parameter Efficiency}
\label{sec:exp_clean_baseline}

\begin{figure}[t]
\centering
\includegraphics[width=0.9\textwidth]{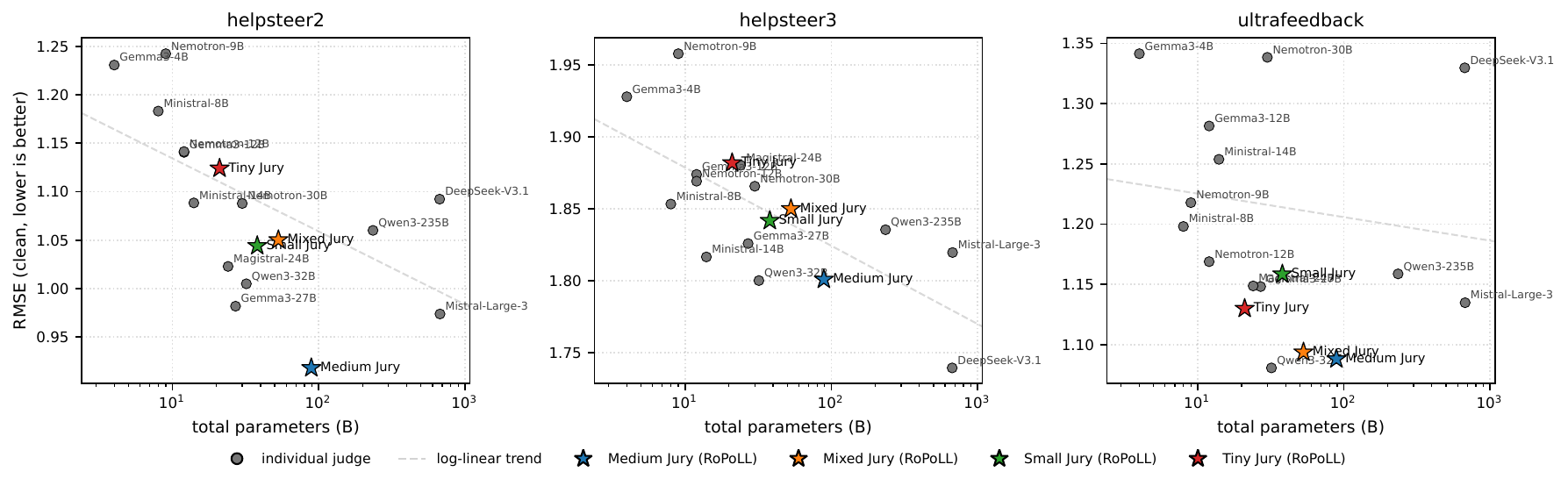}
\caption{\textbf{Parameter efficiency at the clean baseline
($r = 0\%$).}
RMSE vs.\ parameter count (log scale) for each dataset; gray circles
are the 13 individual open-weight judges, dashed line is their
log-linear scaling fit, and coloured stars mark the four
\textsc{RoPoLL} juries at their aggregate parameter budget. Clean
counterpart of \Cref{fig:p2_bimodal_ropoll_hero,fig:p3_hero_rmse_zeros}.}
\label{fig:p5_jury_vs_individual_clean}
\end{figure}

A natural concern about a robust aggregator is that the robustness
costs accuracy in the \emph{absence} of corruption.
Theorem~\ref{thm:ropoll_bound} predicts a small insurance
premium at $\alpha = 0$ (the geometric constant $C_{\beta} \to 1$
as $\beta \to 0$);
\Cref{fig:p5_jury_vs_individual_clean} cashes this empirically by
plotting the four \textsc{RoPoLL} committees against the
$13$-judge individual scaling line at $r = 0\%$.
The \textsc{Medium}, \textsc{Mixed}, and \textsc{Small} committees
sit below the individual scaling line on every benchmark;
\textsc{Tiny} is roughly on-trend.
The clean-case insurance premium is at most $+6.4\%$ relative RMSE
across the full grid (median $+0.9\%$), so the cost of using
\textsc{RoPoLL} when corruption happens to be absent is a small
fraction of the gains it delivers when corruption is present.

\subsection{Jury-Size Ablation and Corruption-Type Dependence}
\label{sec:exp_method_contrast}

\begin{figure*}[t]
\centering
\includegraphics[width=0.8\textwidth]{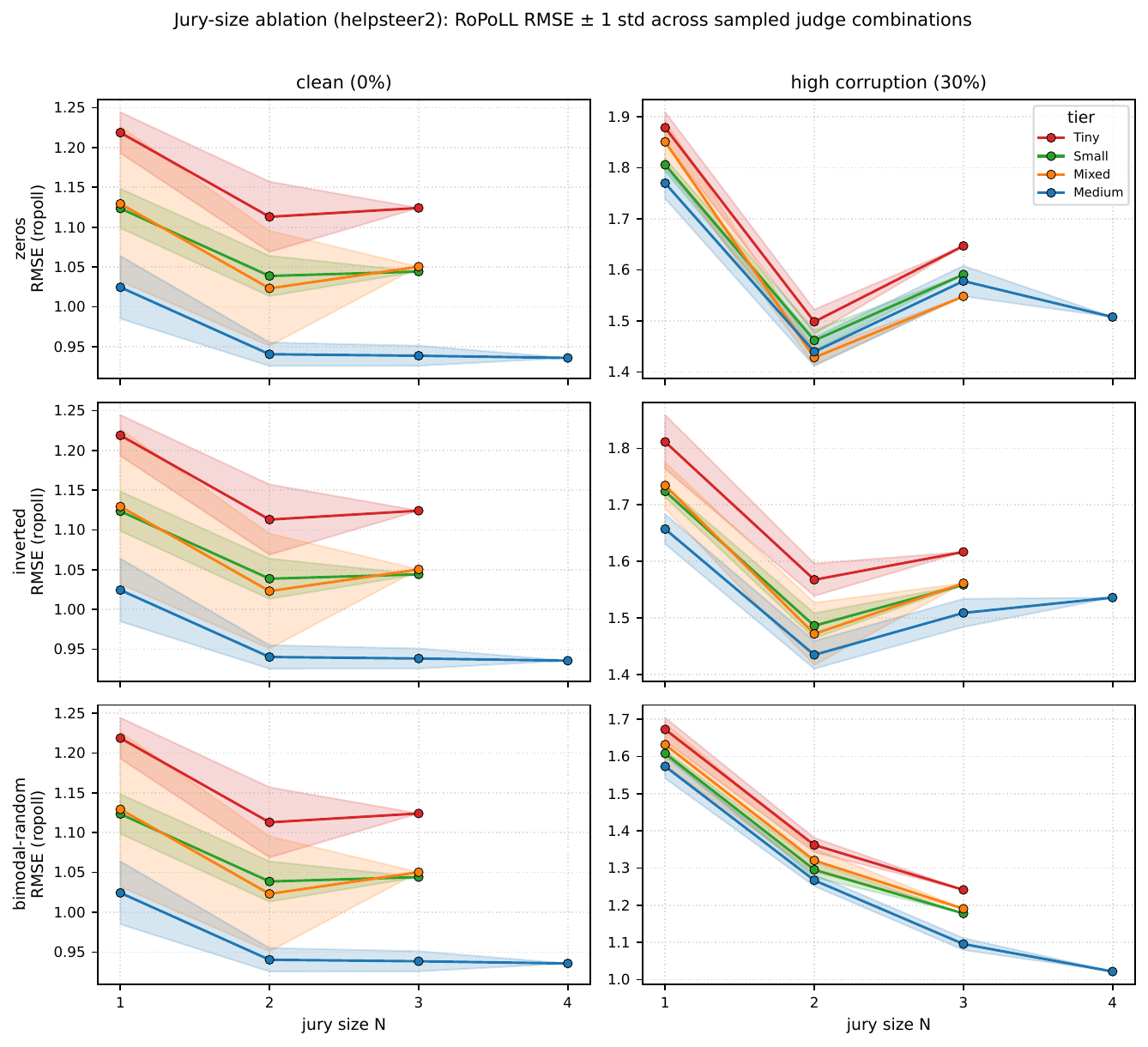}
\caption{\textbf{Jury-size ablation: RMSE vs.\ jury size $N$.}
Mean RMSE across sampled $N$-judge subcommittees from each tier pool,
under \texttt{zeros}/\texttt{inverted}/\texttt{bimodal-random}
corruption. Left column: $r = 0\%$; right column: $r = 30\%$. Bands
show $\pm 1$ standard deviation across combinations.}
\label{fig:p6_ablation_N_rmse}
\end{figure*}

\begin{figure*}[t]
\centering
\includegraphics[width=0.95\textwidth]{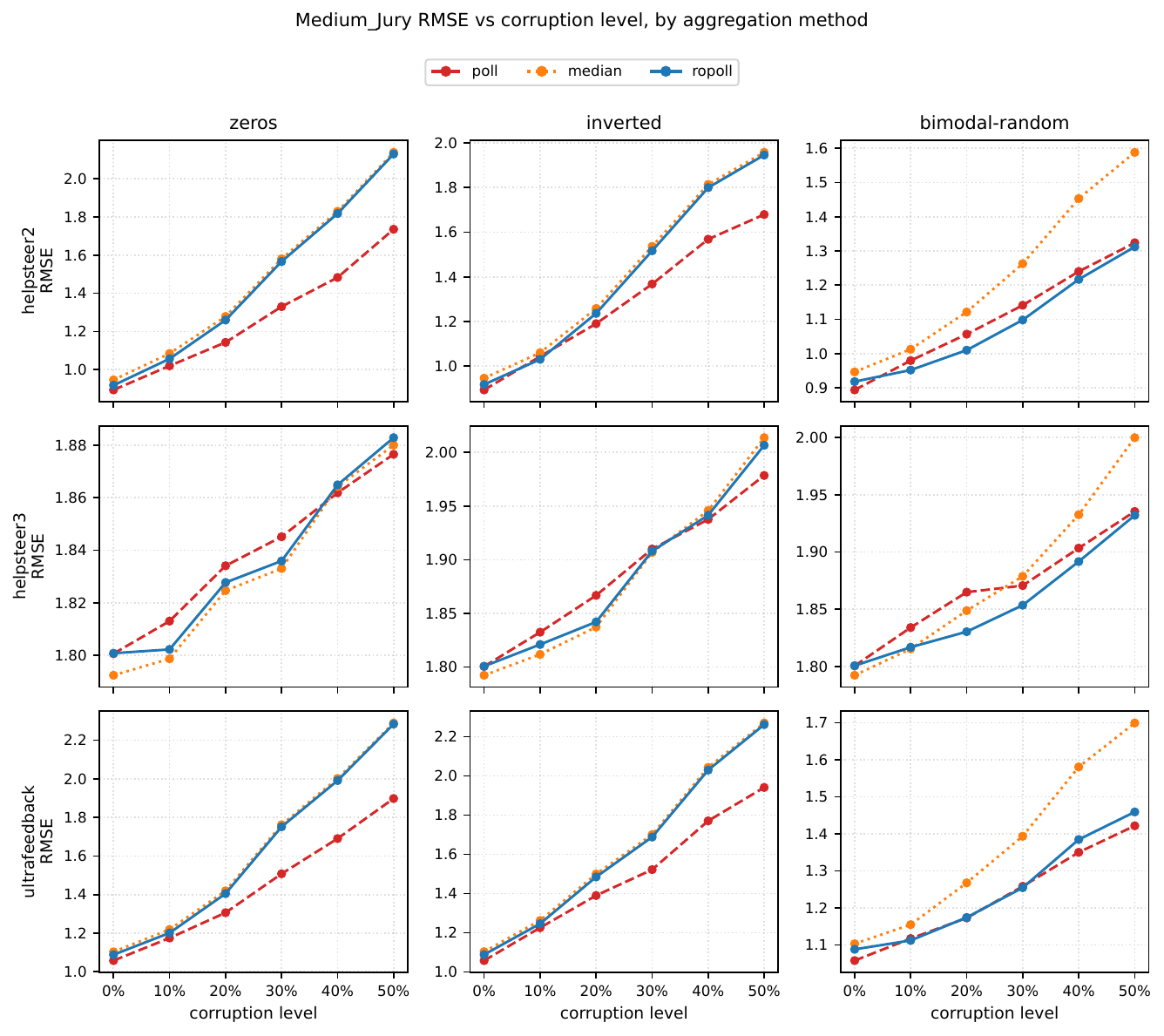}
\caption{\textbf{\textsc{PoLL} vs.\ \textsc{Median} vs.\
\textsc{RoPoLL} degradation curves.}
RMSE vs.\ per-case corruption rate $r$ for the \textsc{Medium} jury,
one panel per (dataset $\times$ corruption type). Solid =
\textsc{RoPoLL}, dashed = \textsc{PoLL}, dotted = coordinate-wise
\textsc{Median}.}
\label{fig:p7_method_contrast}
\end{figure*}

Two practical questions remain: how many judges does \textsc{RoPoLL}
actually need, and is the geometric median always the right choice?
\Cref{fig:p6_ablation_N_rmse} answers the first by sweeping the jury
size $N \in \{1, 2, 3, 4\}$ across the four committee tiers under
three corruption types at the clean baseline ($r = 0\%$, left
column) and under heavy contamination ($r = 30\%$, right column).
\Cref{fig:p7_method_contrast} answers the second by reporting the
full three-method
\textsc{PoLL}/\textsc{Median}/\textsc{RoPoLL} comparison across
every (dataset, corruption type, rate) cell.

\paragraph{Jury ablation.}
RMSE drops sharply from $N = 1$ to $N = 3$ and levels off thereafter
on every tier and corruption type, both clean and at $30\%$
contamination;
the marginal benefit of a fourth judge falls within the
standard-deviation band, confirming the
Corollary~\ref{cor:effective_jury_size} prediction that the
three-judge committees sit at the knee of the cost--accuracy
frontier.

\paragraph{Corruption-type ablation.}
Under \texttt{zeros} and \texttt{inverted} \textsc{PoLL} is
mean-preserving on the bounded score scale and tracks
\textsc{RoPoLL} within $\pm 0.3$ RMSE;
under \texttt{bimodal-random} mean-preservation fails and
\textsc{RoPoLL} stays below \textsc{PoLL} at every $r \geq 10\%$;
under \texttt{cauchy-far} the gap reaches one to three orders of
magnitude (see \Cref{fig:p7_method_contrast}).
The coordinate-wise \textsc{Median} tracks \textsc{RoPoLL} closely on
heavy-tailed and bounded-symmetric corruptions but lags
\textsc{RoPoLL} on \texttt{bimodal-random}, where the
cross-dimensional structure (Example~\ref{ex:cross_dim}) is invisible
to a per-coordinate median.
At the clean baseline, \textsc{RoPoLL} pays a small \emph{insurance
premium} ($+0.01\%$ to $+6.4\%$ relative RMSE).
\textsc{RoPoLL} is a robust default for high-penalty regimes.

\subsection{Noisy-GT Control: Systematic Bias, Not Imprecision}
\label{sec:exp_noisy_gt}

A natural concern about robust aggregation is that the
``insurance premium'' might be paid on a phantom---if real judge
failures are imprecise but unbiased rather than systematically
wrong, robustness machinery is unnecessary.
We test this directly with a \emph{Noisy-GT} adversary that injects
$\hat{\vy}_{\mathrm{noisy}}
 = \mathrm{clip}(\vy^{\star} + \bm{\eps}, 0, K)$
with $\bm{\eps} \sim \gN(\vzero, 0.8^{2}\mI)$ in place of the
adversarial vectors of \S\ref{sec:benchmark_setup}.
Empirically, all three aggregators \emph{improve} as the Noisy-GT
injection rate increases on both HelpSteer\,2 and UltraFeedback,
with \textsc{PoLL} slightly preferred (because averaging unbiased
noise is statistically optimal under Gaussian errors).
This rules out the most obvious confound: the \textsc{RoPoLL}
premium reported in \S\S\ref{sec:exp_cauchy}--\ref{sec:exp_bounded}
is paid against \emph{biased} contamination, not imprecision.
The full per-model and per-dimension breakdowns supporting this
control are in Appendix~\ref{sec:appendix_per_model}.

\subsection{Released Corpus}
\label{sec:exp_corpus}

To support reproduction and follow-up work, we release the full
$13$-judge $\times$ three-benchmark output corpus that drives every
figure in this section.
For each $(\mathrm{judge},\,\mathrm{sample})$ cell the corpus contains
the parsed score vector $\hat{\vy}_{f}(x) \in \R^{d}$ produced by the
deterministic parser $\phi$ (Definition~\ref{def:judge}), the per-call
latency, and the reference label $\vy^{\mathrm{ref}}_{j}$
(Definition~\ref{def:reference_protocol});
parser-failure cells are recorded as the all-zeros fallback vector.
The corpus totals approximately $28\mathrm{K}$ scored
$(\mathrm{judge},\,\mathrm{sample})$ cells
(Table~\ref{tab:corpus_stats}), enabling exact reproduction of every
reported figure without re-running the inference cost.

\begin{table}[h]
\centering
\small
\setlength{\tabcolsep}{4pt}
\begin{tabular}{@{}l rrr rr rr@{}}
\toprule
\textbf{Dataset} & $N_{\mathrm{samp}}$ & $|J|$ & $d$ &
$\bar f$ & $f_{\max}$ & $s_{\min}$ & $s_{\max}$ \\
\midrule
HelpSteer\,2     & 1000 & 13 & 5 & 0.6\% &  2.4\% & $-1.0$ & $4.0$ \\
UltraFeedback    & 1000 & 13 & 4 & 0.0\% &  0.0\% &  $1.0$ & $5.0$ \\
HelpSteer\,3     &  100 & 16 & 1 & 6.0\% & 33.0\% & $-3.8$ & $2.6$ \\
\bottomrule
\end{tabular}
\caption{Corpus-level statistics. $N_{\mathrm{samp}}$: samples;
$|J|$: judge pool size; $d$: target dimension;
$\bar f$, $f_{\max}$: mean and max per-judge parser-failure rate;
$s_{\min}$, $s_{\max}$: observed score range across all judges and
samples (negative values arise from HS\,2 / HS\,3 signed-difference
reductions on a small fraction of cells where parsed scores fell
outside the rubric range).
For HS\,3, $\bar f$ and $f_{\max}$ are computed over the full
$16$-judge pool (the $13$ open-weight judges plus the three HS\,3-only
Claude judges); the $13$-judge common-pool mean is $3.38\%$
(\Cref{fig:corpus_failure_rates}).
HS\,3 in the released JSON contains the $100$-sample preference slice
used for the \S\ref{sec:benchmark} evaluation; the full 2017-sample multilingual
validation set is available on request.}
\label{tab:corpus_stats}
\end{table}

\paragraph{Per-attribute score distributions.}
\Cref{fig:corpus_score_dists} plots the score distributions per
attribute (per-dataset).
HelpSteer\,2 and UltraFeedback have substantial mass at the score
extremes (parser fallback at $0$; sycophantic judges concentrating at
the maximum), motivating the \texttt{zeros} and \texttt{inverted}
corruption types used in \S\ref{sec:benchmark}.
HelpSteer\,3, which reduces a five-attribute pair of responses to a
single signed-preference scalar, is well-centered on $0$ with light
tails, consistent with the cancellation of per-attribute biases under
the signed-difference reduction.

\begin{figure*}[t]
\centering
\begin{minipage}{0.55\textwidth}
\centering
\textbf{HelpSteer 2}\\[1pt]
\includegraphics[width=\textwidth]{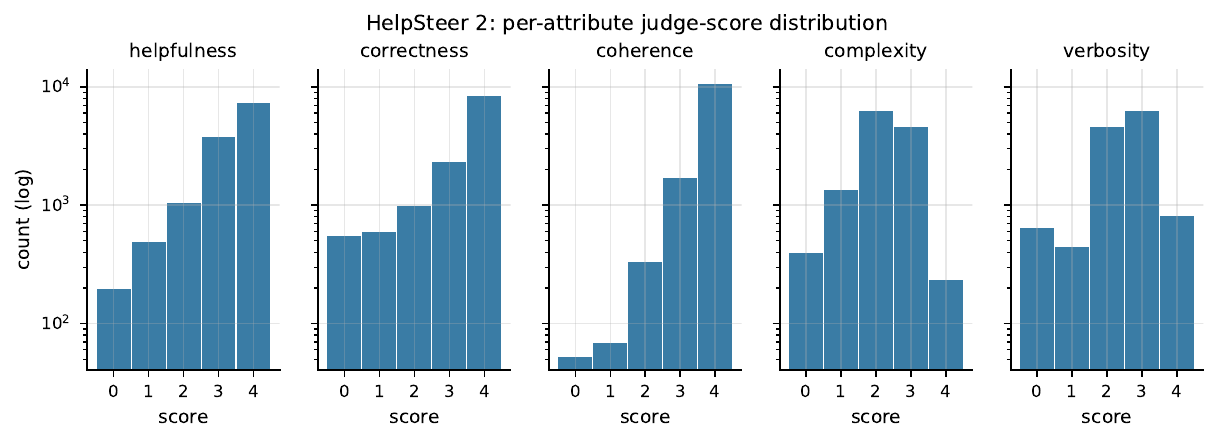}
\end{minipage}
\hfill
\begin{minipage}{0.42\textwidth}
\centering
\textbf{UltraFeedback}\\[1pt]
\includegraphics[width=\textwidth]{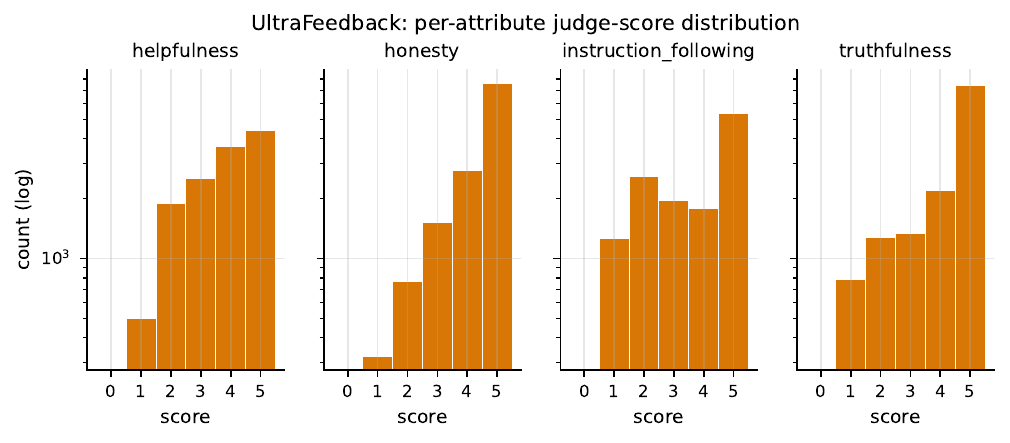}
\end{minipage}\\[6pt]
\begin{minipage}{0.55\textwidth}
\centering
\textbf{HelpSteer 3}\\[1pt]
\includegraphics[width=\textwidth]{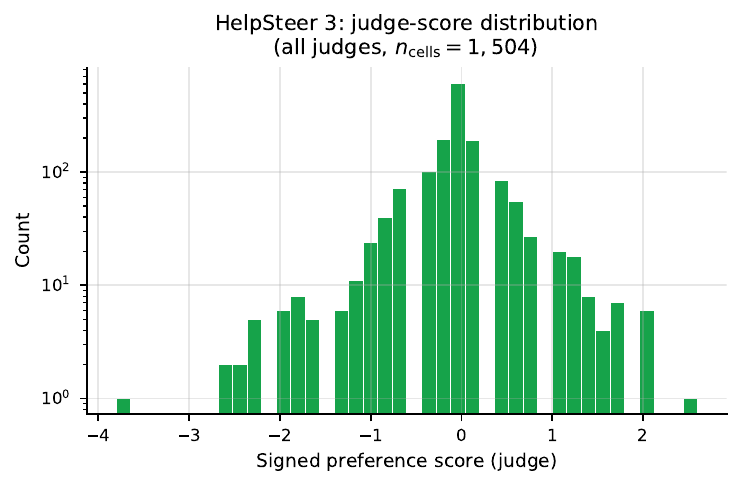}
\end{minipage}
\caption{\textbf{Per-attribute judge-score distributions
(log $y$-axis).}
HelpSteer\,2 and UltraFeedback show heavy mass concentration at the
score extremes---parser fallback at $0$ and sycophantic saturation at
the maximum---which motivates the \texttt{zeros} and \texttt{inverted}
corruption types used in \S\ref{sec:benchmark}.
HelpSteer\,3 (signed-preference scalar) is centered on $0$ with light
tails, consistent with cancellation of per-attribute biases under the
signed-difference reduction.}
\label{fig:corpus_score_dists}
\end{figure*}

\subsection{Inter-Judge Correlation Structure}
\label{sec:appendix_dataset_correlation}

\begin{figure*}[htb!]
\centering
\includegraphics[width=0.32\textwidth]{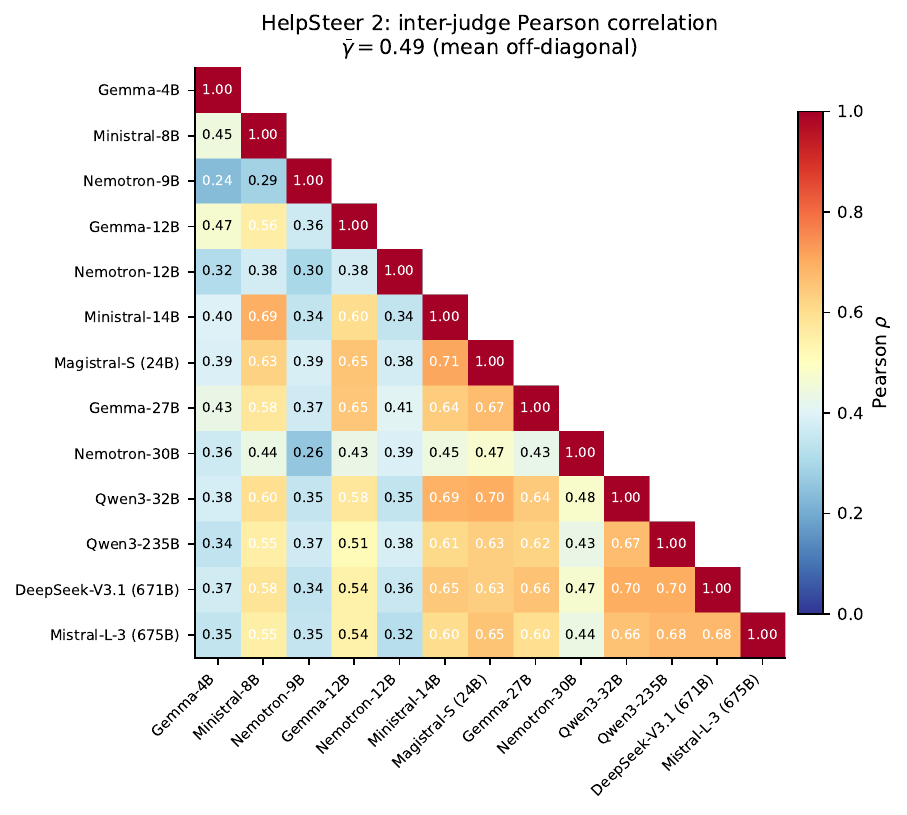}
\hfill
\includegraphics[width=0.32\textwidth]{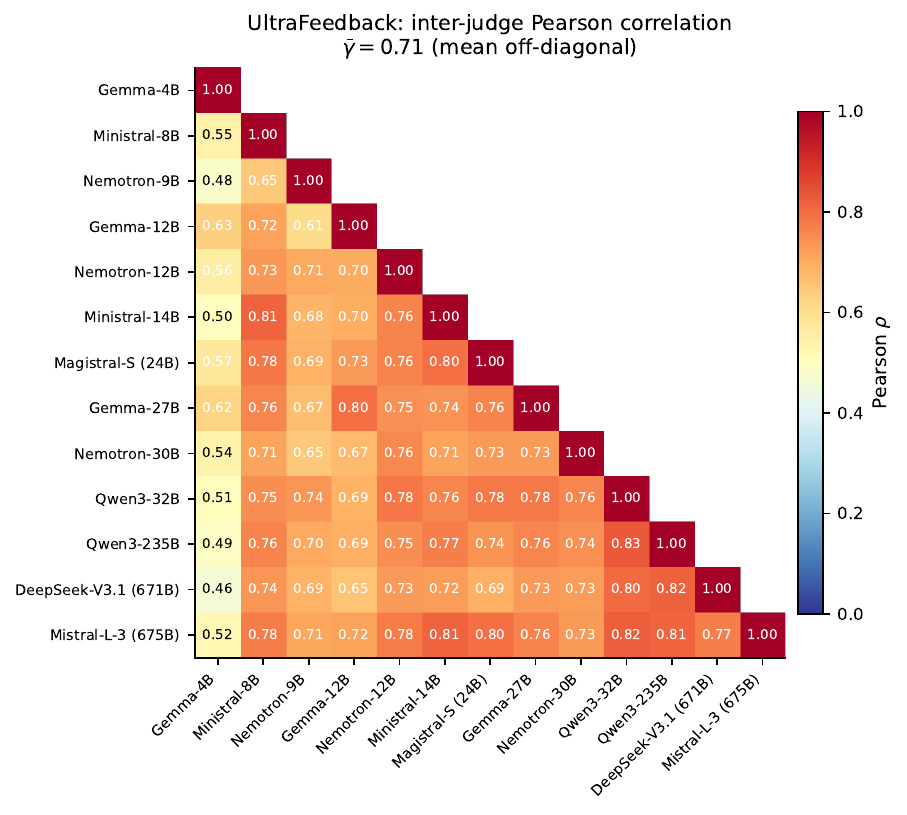}
\hfill
\includegraphics[width=0.32\textwidth]{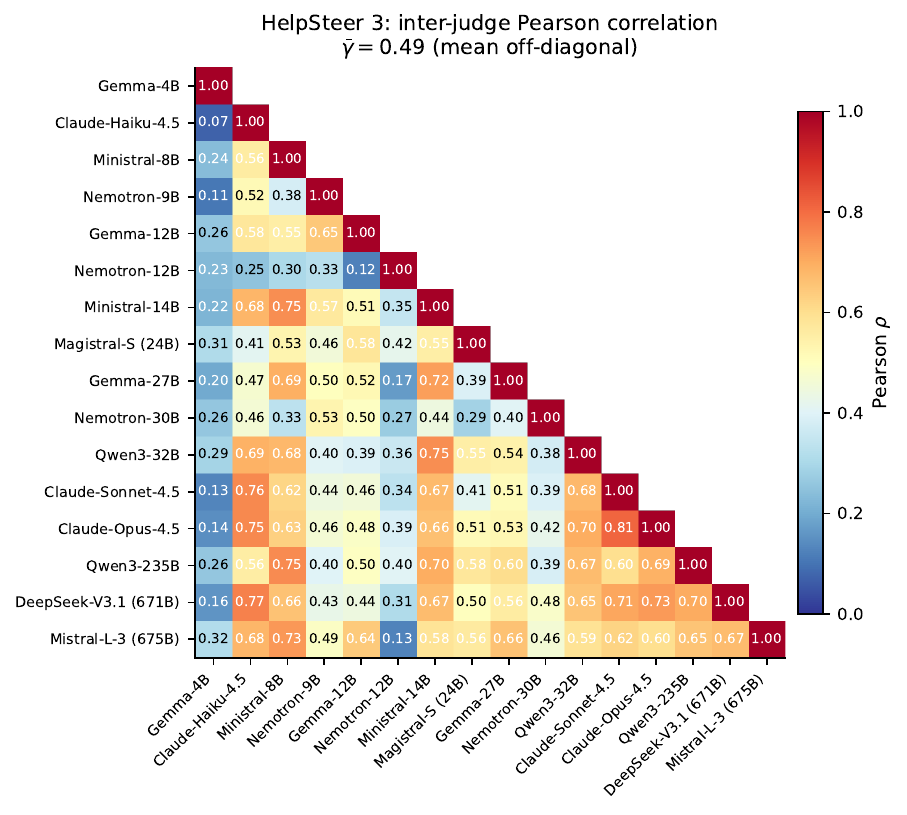}
\caption{\textbf{Inter-judge Pearson correlation heatmaps
(lower-triangle, annotated).}
Pairwise correlations averaged over evaluation attributes; cells
labelled with their numeric value.
Empirical mean off-diagonal correlations:
$\bar\gamma_{\mathrm{HS2}} = 0.49$,
$\bar\gamma_{\mathrm{UF}} = 0.71$,
$\bar\gamma_{\mathrm{HS3}} = 0.49$.
These values support the $\gamma \in [0.3, 0.5]$ assumption used in
\S\ref{sec:benchmark_setup} to motivate three-judge committees via
Corollary~\ref{cor:effective_jury_size}; the higher
$\bar\gamma_{\mathrm{UF}}$ explains the smaller \textsc{RoPoLL}/\textsc{PoLL}
gap observed on UltraFeedback.}
\label{fig:corpus_judge_corr}
\end{figure*}

\Cref{fig:corpus_judge_corr} shows the pairwise Pearson correlation
between every judge pair in the $13$-judge pool, averaged over
attributes.
Empirical mean off-diagonal correlations are
$\bar\gamma_{\mathrm{HS2}} = 0.49$,
$\bar\gamma_{\mathrm{HS3}} = 0.49$,
and $\bar\gamma_{\mathrm{UF}} = 0.71$.
These directly support the assumption $\gamma \in [0.3, 0.5]$ used in
\S\ref{sec:benchmark_setup} to motivate the choice $N = 3$:
substituting the measured $\bar\gamma$ into the saturation law
$N_{\mathrm{eff}}^{\infty} = 1/\gamma$
(Corollary~\ref{cor:effective_jury_size}) yields
$N_{\mathrm{eff}}^{\infty} \approx 2.0$ on HS\,2 and HS\,3 and
$\approx 1.4$ on UltraFeedback, so the empirical diminishing-returns
knee at $N = 3$ in \Cref{fig:p6_ablation_N_rmse} sits at or just past
the saturation point predicted by the corpus's actual correlation
structure.

The UltraFeedback correlation $\bar\gamma_{\mathrm{UF}} = 0.71$ is
notably higher than the HelpSteer correlations.
This reflects the fact that UltraFeedback's reference labels are
themselves GPT-4 annotations \citep{cui2024ultrafeedback}, so judges
trained on similar rubric distributions converge to similar scores;
the HelpSteer benchmarks use trained-human annotators
(HelpSteer\,2) or pairwise human preferences (HelpSteer\,3),
producing more genuine inter-judge variation.
This is consistent with the smaller \textsc{RoPoLL}/\textsc{PoLL} gap observed
on UltraFeedback in \S\ref{sec:benchmark}: when judges already agree,
the difference between the mean and the geometric median is small.

\subsection{Empirical Indicator Correlation $\bar\gamma_{W}$}
\label{sec:appendix_gamma_W}

Lemma~\ref{lem:correlated_jury} bounds the failure probability of
the \textsc{RoPoLL} cluster event in terms of the
\emph{indicator correlation}
$\bar\gamma_{W} = \mathrm{mean}_{i \neq j}\frac{\Cov(W_{i},W_{j})}{\sqrt{\Var(W_{i})\Var(W_{j})}}$
of the cluster indicators
$W_{i} = \mathbb{1}\{Z_{i}=0,\ \|\hat{\vy}_{i}-\vy^{\star}\|_{2}\leq\rho_{p}\}$.
This is in principle a finer object than the inter-judge
\emph{score} correlation $\bar\gamma$ of \S\ref{sec:appendix_dataset_correlation}:
$\bar\gamma$ measures the linear correlation between raw score
vectors, while $\bar\gamma_{W}$ measures the co-incidence of two
judges \emph{both being competent and within the cluster ball}.
We estimate $\bar\gamma_{W}$ directly on the experimental panels.

\paragraph{Estimation procedure.}
For benchmark $b \in \{\mathrm{HS2}, \mathrm{UF}\}$:
(i)~for each $(\text{judge } i, \text{sample } s)$ cell, compute the
$\ell_{2}$ deviation
$\delta_{i}^{(s)} = \|\hat{\vy}_{i}^{(s)} - \vy^{\star,(s)}\|_{2}$
(parser-failure cells contribute $W_{i}^{(s)} = 0$);
(ii)~select a cluster radius $\rho$ as the $p$-th quantile of
pooled deviations $\{\delta_{i}^{(s)}\}_{i,s}$;
(iii)~form $W_{i}^{(s)} = \mathbb{1}\{\delta_{i}^{(s)} \leq \rho\}$;
(iv)~compute the mean off-diagonal Pearson correlation of the rows
of $W \in \{0,1\}^{N \times S}$.
We report $\bar\gamma_{W}$ at three radii ($p \in \{0.50, 0.70, 0.90\}$)
to show stability under the calibration choice.

\begin{table}[h]
\centering
\small
\setlength{\tabcolsep}{8pt}
\renewcommand{\arraystretch}{1.10}
\begin{tabular}{@{}l c c c c@{}}
\toprule
\textbf{Benchmark} & \textbf{$p$-quantile} &
\textbf{$\rho$} & \textbf{$\bar\gamma_{W}$} &
\textbf{$N_{\mathrm{eff}}$ at $N{=}3$} \\
\midrule
\multirow{3}{*}{HelpSteer-2}
  & $0.50$ & $2.000$ & $0.500$ & $1.50$ \\
  & $0.70$ & $2.449$ & $0.531$ & $1.46$ \\
  & $0.90$ & $4.000$ & $0.471$ & $1.55$ \\
\midrule
\multirow{3}{*}{UltraFeedback}
  & $0.50$ & $2.000$ & $0.531$ & $1.45$ \\
  & $0.70$ & $2.449$ & $0.475$ & $1.54$ \\
  & $0.90$ & $3.742$ & $0.450$ & $1.58$ \\
\bottomrule
\end{tabular}
\caption{\textbf{Empirical indicator correlation $\bar\gamma_{W}$ of
Lemma~\ref{lem:correlated_jury}} on our $13$-judge experimental
panels.
$\bar\gamma_{W}$ is stable to within $\pm 0.03$ across cluster-radius
calibrations and lies in $[0.45, 0.53]$ on both benchmarks, so
$N_{\mathrm{eff}} \in [1.45, 1.58]$ at the practical jury size
$N{=}3$.
The empirical $\bar\gamma_{W}$ is on the same order as the score
correlation $\bar\gamma \in [0.49, 0.71]$ reported in
\Cref{fig:corpus_judge_corr};
Pitt's Gaussian correlation inequality
\citep{pitt1977gaussian,esary1967association,joag1983negative}
gives the qualitative bound $\bar\gamma_{W} \geq 0$ but not a
quantitative comparison to $\bar\gamma$, so direct estimation is the
right move.}
\label{tab:empirical_gamma_W}
\end{table}

\paragraph{Implication for Lemma~\ref{lem:correlated_jury}.}
The role of Lemma~\ref{lem:correlated_jury} is structural: it
shows that the geometric-breakdown structure
($C_{\alpha+\beta}$ and the cluster radius $\rho$) of
Theorem~\ref{thm:ropoll_bound} is preserved when the i.i.d.\
assumption is replaced by an equicorrelated-indicator
hypothesis with $\bar\gamma_{W} \in [0, 1]$.
The probability event delivered by the Chebyshev step,
$\Pr[\cdot] \geq 1 - 1/(\beta^{2} N_{\mathrm{eff}})$, is informative
in the large-$N_{\mathrm{eff}}$ regime
(e.g., a hypothetical jury of $N=10$--$30$ judges with
$\bar\gamma_{W} \approx 0.2$ gives $N_{\mathrm{eff}} \in [3.6, 7]$
and $\beta = 0.2$ gives a non-trivial bound) but degenerates at
small $N_{\mathrm{eff}}$, including the practical
$N_{\mathrm{eff}} \approx 1.5$ of our $N=3$ panels.
This is a fundamental limit of variance-only concentration at
small $N$, not a slack in the analysis: with only $\sim 1.5$
effective independent samples, no concentration argument can
deliver a tight high-probability bound, regardless of the
estimator.
A Bernstein-type bound under bounded-covariance martingale
structure (Remark~\ref{rem:correlated_tail}) would replace
$\beta^{-2}$ with $\exp(-c\beta^{2}N_{\mathrm{eff}})$ but does not
materially help at $N_{\mathrm{eff}} \approx 1.5$.
The practical value of Lemma~\ref{lem:correlated_jury} for our
small-$N$ regime is therefore the structural guarantee, not the
quantitative probability:
the breakdown floor and the geometric constant are independent of
this concentration argument and \emph{are} the load-bearing
quantities for jury-aggregation deployment.

\subsection{Practical Recommendation}
\label{sec:exp_recommendation}

Use \textsc{RoPoLL} as the default jury aggregator: the clean-case
insurance premium is small ($\leq 6\%$ relative RMSE,
\S\ref{sec:exp_clean_baseline}) and the threat to LLM juries is biased
contamination rather than imprecision (\S\ref{sec:exp_noisy_gt}).
The jury size is not a fixed prescription but follows from the
saturation law of Corollary~\ref{cor:effective_jury_size}
($N_{\mathrm{eff}}$ saturates at $1/\gamma$): the cost--accuracy knee
sits at whatever $N$ reaches that ceiling for the inter-judge
correlation $\gamma$ of the judge pool at hand.
For the diverse open-weight pools studied here
($\gamma \approx 0.49$--$0.71$, \S\ref{sec:appendix_dataset_correlation})
the knee falls at $N \approx 3$ (\Cref{fig:p6_ablation_N_rmse}); a more
orthogonal pool (smaller $\gamma$) would push it higher, and a
redundant one lower.
A controlled 2D synthetic visualisation of three representative
failure modes, the per-model and per-dimension breakdowns, and the
full extra-metrics tables are in
Appendices~\ref{sec:appendix_per_model} and~\ref{sec:simulation}.

\section{Conclusion}
\label{sec:discussion}

We recast LLM-jury aggregation as a robust mean-estimation problem,
showed that \textsc{PoLL} \citep{verga2024replacing} admits
unbounded bias under any positive contamination
(Proposition~\ref{prop:mean_bias}), and proposed \textsc{RoPoLL}:
replace the mean with the geometric median.
Theorem~\ref{thm:ropoll_bound} gives a finite-sample upper bound;
Lemma~\ref{lem:correlated_jury} extends it to equicorrelated
juries; Theorem~\ref{thm:minimax_body} provides a matching minimax
lower bound that aligns on the parametric rate and exposes a
$\sqrt{d}$ statistical--computational gap on the breakdown floor
(the price of GM's polynomial-time tractability vs.\ the
intractable Tukey halfspace median).
Empirically (\S\ref{sec:benchmark}), \textsc{RoPoLL} reduces
\textsc{PoLL}'s RMSE by orders of magnitude on heavy-tailed and
cross-dimensional attacks while paying $\leq 6.4\%$ in clean-baseline
relative RMSE, and a $3$-judge \textsc{RoPoLL} committee at
$38$\,B beats Mistral-Large-3 ($675$\,B) by $1.31\times$ on
HelpSteer\,2 under $30\%$ \texttt{bimodal-random} corruption:
\emph{robust aggregation, not the ensemble itself, delivers the win}.

\paragraph{Scope and limitations.}
The theory holds the i.i.d.\ baseline of
Assumptions~\ref{asm:huber}--\ref{asm:minority} fixed and partially
relaxes independence via Lemma~\ref{lem:correlated_jury};
per-judge heterogeneity ($\sigma_{i},\alpha_{i}$ varying across
$4$--$675$\,B), explicit dependence by design
\citep{li2024prd,chan2024chateval,zhang2024wider}, and tightening the
contamination-constant gap between Theorem~\ref{thm:ropoll_bound}
and Theorem~\ref{thm:minimax_body} at finite $\alpha$ remain open.
Empirically, the corruption sweep is synthetic injection at a single
rubric serialisation and temperature $0$, so it does not probe
prompt-format sensitivity \citep{wang2023large,stureborg2024large};
the Noisy-GT control (\S\ref{sec:exp_noisy_gt}) rules out the
obvious confound that \textsc{RoPoLL}'s premium is paid against
benign imprecision rather than biased contamination, but a
large-scale evaluation against \emph{naturally occurring} judge
failures on a downstream alignment task remains the most important
next step.
A systematic comparison against the broader robust-aggregation
toolbox (median-of-means, smoothed Tukey depth, learned calibration)
is also left open.

\paragraph{Outlook.}
The corruption-class diagnosis transfers verbatim to any pipeline
where heterogeneous workers produce biased point-mass errors at low
rate---reward-model ensembles for RLHF, synthetic-data filtering
juries, crowd annotation---suggesting the geometric median as a
candidate default beyond LLM juries.

\bibliographystyle{apalike}
\bibliography{bibs/jury,bibs/math_basics,bibs/datasets}

\clearpage
\appendix
\begin{center}
  {\LARGE\bfseries Appendix\par}
\end{center}

\noindent
The appendix collects the deferred proofs and supporting material
for the body of the paper.
The roadmap of formal results is in
Table~\ref{tab:theory_roadmap} (\S\ref{sec:theory});
the appendix sections that follow are organised as
Appendix~\ref{sec:appendix_theory}
(proofs for \S\ref{sec:problem_setup}, \S\ref{sec:methodology}, and
\S\ref{sec:theory}, including the matching minimax lower bound),
Appendix~\ref{sec:appendix_per_model}
(per-model and per-dimension breakdowns supporting
\S\ref{sec:benchmark}),
and Appendix~\ref{sec:simulation}
(controlled 2D synthetic visualisation of five representative
failure modes).

%
%
%
\section{Complete Proofs and Full Theoretical Development}
\label{sec:appendix_theory}

\subsection{Proof of Proposition~\ref{prop:variance_reduction} (Variance Reduction)}
\label{sec:appendix_variance_proof}

\begin{proof}[Proof of Proposition~\ref{prop:variance_reduction}]
Under $\alpha_{i} = 0$, Assumption~\ref{asm:huber} gives
$\E[\hat{\vy}_{i} \mid \vy^{\star}] = \vy^{\star}$ and
$\Cov(\hat{\vy}_{i} \mid \vy^{\star}) = \mSigma_{i}$.
Linearity of conditional expectation gives
$\E[\hat{\vy}_{\mathrm{mean}} \mid \vy^{\star}] = \vy^{\star}$, and
bilinearity of covariance gives~\eqref{eq:mean_cov_general}.
Independence (Assumption~\ref{asm:independence}) zeroes the
cross-covariances, yielding the off-diagonal vanishing
in~\eqref{eq:mean_cov_indep}.
Since the conditional error is centered,
\begin{equation*}
    \E[\|\hat{\vy}_{\mathrm{mean}} - \vy^{\star}\|_{2}^{2} \mid \vy^{\star}]
    \;=\;
    \operatorname{tr}\!\left(\Cov(\hat{\vy}_{\mathrm{mean}} \mid \vy^{\star})\right),
\end{equation*}
which is the second equation in~\eqref{eq:mean_cov_indep}.
The final bound follows from
$\operatorname{tr}(\mSigma_{i}) \leq d\sigma^{2}$ whenever
$\mSigma_{i} \preceq \sigma^{2}\mI_{d}$.
\end{proof}

\begin{proof}[Proof of Corollary~\ref{cor:effective_jury_size}]
Substituting the equicorrelated structure
$\Cov(\hat{\vy}_{i}, \hat{\vy}_{j} \mid \vy^{\star}) = \gamma\,\mSigma$
for $i \neq j$ and $\Cov(\hat{\vy}_{i} \mid \vy^{\star}) = \mSigma$
into~\eqref{eq:mean_cov_general} gives
$\Cov(\hat{\vy}_{\mathrm{mean}} \mid \vy^{\star})
= \tfrac{1 + (N-1)\gamma}{N}\,\mSigma$;
taking traces yields the MSE expression.
\end{proof}

\begin{remark}[Implication for Jury Design]
\label{rem:variance_discussion}
Proposition~\ref{prop:variance_reduction} and
Corollary~\ref{cor:effective_jury_size} formalize the classical
benefit of a jury: independent, diverse, conditionally unbiased
judges reduce estimator variance, with an effective sample-size
penalty determined by their pairwise dependence.
Proposition~\ref{prop:mean_bias} shows why this benefit is
insufficient under contamination: the arithmetic mean's bias is
unbounded over the corruption class (Assumption~\ref{asm:huber})
regardless of $N$, so any variance reduction the jury affords is
dominated by an adversarial choice of $\{Q_{i}\}$.
A robust aggregation rule is therefore needed to preserve the signal
of the competent majority while attenuating contamination
bias---this is the role of \textsc{RoPoLL} in
\S\ref{sec:methodology}.
\end{remark}

\subsection{Proof of Proposition~\ref{prop:mean_bias} (Unbounded Bias of \textsc{PoLL})}
\label{sec:appendix_mean_bias_proof}

For convenience we recall the statement: under
Assumption~\ref{asm:huber} and finite first moments
$\vmu_{i}^{Q} \triangleq \E_{Q_{i}}[\hat{\vy}_{i}]$, the mean's
conditional bias is
\begin{equation}
\label{eq:appendix_mean_bias_formula}
    \E\!\left[\hat{\vy}_{\mathrm{mean}} \mid \vy^{\star}\right]
    - \vy^{\star}
    \;=\;
    \frac{1}{N}\sum_{i=1}^{N}\alpha_{i}\!\left(\vmu_{i}^{Q} - \vy^{\star}\right),
\end{equation}
and is not uniformly bounded over the corruption class as long as
$\alpha > 0$, regardless of $N$.

\begin{proof}[Proof of Proposition~\ref{prop:mean_bias}]
We prove the two claims in turn: the explicit bias formula
\eqref{eq:appendix_mean_bias_formula}, and the impossibility of a
uniform bound over the corruption class.

\textit{Step 1: Per-judge expectation.}
Fix $i \in [N]$.
By Assumption~\ref{asm:huber}, conditional on $\vy^{\star}$ the law
of $\hat{\vy}_{i}$ is the mixture
$(1 - \alpha_{i}) P_{i} + \alpha_{i} Q_{i}$ with selector
$Z_{i} \sim \mathrm{Bernoulli}(\alpha_{i})$.
By the law of total expectation,
\begin{align*}
    \E[\hat{\vy}_{i} \mid \vy^{\star}]
    &= (1 - \alpha_{i})\,\E_{P_{i}}[\hat{\vy}_{i}]
       + \alpha_{i}\,\E_{Q_{i}}[\hat{\vy}_{i}] \\
    &= (1 - \alpha_{i})\,\vy^{\star} + \alpha_{i}\,\vmu_{i}^{Q},
\end{align*}
where the second equality uses
$\E_{P_{i}}[\hat{\vy}_{i}] = \vy^{\star}$ (competent unbiasedness,
Assumption~\ref{asm:huber}) and the finite-first-moment assumption
on $Q_{i}$ to identify $\E_{Q_{i}}[\hat{\vy}_{i}] = \vmu_{i}^{Q}$.
Rearranging,
\begin{equation}
\label{eq:appendix_per_judge_bias}
    \E[\hat{\vy}_{i} \mid \vy^{\star}] - \vy^{\star}
    \;=\;
    \alpha_{i}\,(\vmu_{i}^{Q} - \vy^{\star}).
\end{equation}

\textit{Step 2: Linearity of the mean.}
By the linearity of expectation applied to
$\hat{\vy}_{\mathrm{mean}} = N^{-1}\sum_{i=1}^{N}\hat{\vy}_{i}$,
\begin{equation*}
    \E[\hat{\vy}_{\mathrm{mean}} \mid \vy^{\star}]
    \;=\;
    \frac{1}{N}\sum_{i=1}^{N}\E[\hat{\vy}_{i} \mid \vy^{\star}].
\end{equation*}
Substituting~\eqref{eq:appendix_per_judge_bias} yields
\eqref{eq:appendix_mean_bias_formula}, proving the first claim.

\textit{Step 3: Adversarial corruption distribution.}
Suppose $\alpha = N^{-1}\sum_{i}\alpha_{i} > 0$.
Then there exists at least one index $i_{0} \in [N]$ with
$\alpha_{i_{0}} > 0$.
Let $B > 0$ be arbitrary, $\ve_{1}$ be the first standard basis
vector, and consider the adversarial choice
\begin{equation}
\label{eq:appendix_adversarial_Q}
    Q_{i_{0}} \;=\; \delta_{\vy^{\star} + (NB/\alpha_{i_{0}})\,\ve_{1}},
\end{equation}
the Dirac mass placed at the indicated point;
take $\{Q_{i}\}_{i \neq i_{0}}$ to be any distributions with
$\vmu_{i}^{Q} = \vy^{\star}$ (e.g., $Q_{i} = P_{i}$ itself, which
gives zero contribution to the bias).
Then $\vmu_{i_{0}}^{Q} = \vy^{\star} + (NB/\alpha_{i_{0}})\,\ve_{1}$,
and \eqref{eq:appendix_mean_bias_formula} reduces to
\begin{align*}
    \E[\hat{\vy}_{\mathrm{mean}} \mid \vy^{\star}] - \vy^{\star}
    &= \frac{1}{N}\,\alpha_{i_{0}}\,
       (\vmu_{i_{0}}^{Q} - \vy^{\star}) \\
    &= \frac{1}{N}\,\alpha_{i_{0}}\,
       \frac{NB}{\alpha_{i_{0}}}\,\ve_{1}
    \;=\; B\,\ve_{1}.
\end{align*}
Hence
$\big\|\E[\hat{\vy}_{\mathrm{mean}} \mid \vy^{\star}] - \vy^{\star}\big\|_{2}
= B$.

\textit{Step 4: Conclusion.}
Since $B > 0$ was arbitrary, no constant $C(\alpha, N, d, \sigma)$
depending only on the model parameters of
Assumptions~\ref{asm:huber}--\ref{asm:minority} can satisfy
\begin{equation*}
    \sup_{\{Q_{i}\}}\,
    \big\|\E[\hat{\vy}_{\mathrm{mean}} \mid \vy^{\star}] - \vy^{\star}\big\|_{2}
    \;\leq\;
    C(\alpha, N, d, \sigma).
\end{equation*}
The bias is therefore unbounded over the corruption class for every
fixed $N$, completing the proof.
\end{proof}

\begin{remark}[Why $N$ does not help]
The construction~\eqref{eq:appendix_adversarial_Q} scales
$\vmu_{i_{0}}^{Q}$ with $N$: the adversary's per-judge displacement
grows linearly with the jury size, exactly cancelling the $1/N$
averaging.
This is the formal statement of why variance reduction
(Proposition~\ref{prop:variance_reduction}) cannot rescue the mean
under contamination: variance contracts at rate $1/N$, but bias is
preserved by the adversary irrespective of $N$, and the bias term
dominates as long as $\alpha > 0$.
\end{remark}

\subsection{Proof of Proposition~\ref{prop:gm_properties}}
\label{sec:appendix_gm_properties}

\begin{proof}[Proof of Proposition~\ref{prop:gm_properties}]
\textit{(i) Existence.}
Each summand $\vz \mapsto \|\vz - \hat{\vy}_{i}\|_{2}$ is the Euclidean
norm of an affine function of $\vz$, hence continuous and convex
(see any standard reference on convex analysis).
Sums of continuous convex functions are continuous and convex, so $F$
is continuous and convex.
For coercivity, fix any data point $\hat{\vy}_{1}$; by the reverse
triangle inequality
\begin{equation*}
    F(\vz) \;\geq\; \|\vz - \hat{\vy}_{1}\|_{2}
    \;\geq\; \|\vz\|_{2} - \|\hat{\vy}_{1}\|_{2}
    \;\to\; \infty \quad \text{as } \|\vz\|_{2} \to \infty.
\end{equation*}
Since $F$ is continuous and coercive, the sublevel set
$\{\vz : F(\vz) \leq F(\vzero)\}$ is nonempty, closed, and bounded,
hence compact in $\R^{d}$.
Weierstrass's theorem then yields a minimizer.

\textit{(ii) Uniqueness.}
The function $\vz \mapsto \|\vz - \hat{\vy}_{i}\|_{2}$ is strictly
convex on every line not passing through $\hat{\vy}_{i}$ and affine on
the line through $\hat{\vy}_{i}$ in the direction of any other point.
Suppose the data are not collinear: then for any line
$\gL \subset \R^{d}$ there exists at least one
$\hat{\vy}_{i} \notin \gL$, so the corresponding summand is strictly
convex along $\gL$.
Hence $F$ is strictly convex along every line, hence strictly convex
on $\R^{d}$, and the minimizer is unique
\citep{vardi2000multivariate}.

\textit{(iii) Affine equivariance.}
Let $\mU \in \R^{d \times d}$ be orthogonal and $\vb \in \R^{d}$.
For all $\vz \in \R^{d}$,
\begin{equation*}
    \|\mU \vz + \vb - (\mU \hat{\vy}_{i} + \vb)\|_{2}
    \;=\; \|\mU(\vz - \hat{\vy}_{i})\|_{2}
    \;=\; \|\vz - \hat{\vy}_{i}\|_{2},
\end{equation*}
where the second equality uses $\mU^{\top}\mU = \mI_{d}$.
Summing over $i$,
$F^{\mU,\vb}(\mU\vz + \vb) = F(\vz)$ where $F^{\mU,\vb}$ is the
objective on the transformed sample.
The map $\vz \mapsto \mU\vz + \vb$ is a bijection on $\R^{d}$, so the
two minimizers are related by exactly this transformation.

\textit{(iv) Breakdown point.}
We show that the GM tolerates any corruption of strictly fewer than
$\lceil N/2 \rceil$ points and that this threshold is tight.

\emph{Sufficiency.}
Suppose $m < \lceil N/2 \rceil$ points are arbitrarily replaced and
denote the corrupted sample $\hat{\vy}'_{1:N}$.
The competent set $S = \{i : \hat{\vy}'_{i} = \hat{\vy}_{i}\}$ has
$|S| = N - m > N/2$, hence $|S| > |S^{c}|$.
Let $\vz' = \mathrm{GM}(\hat{\vy}'_{1:N})$ be the corrupted GM.
By the subgradient optimality condition for the convex objective $F'$,
\begin{equation*}
    \vzero \in \partial F'(\vz')
    = \sum_{i: \vz' \neq \hat{\vy}'_{i}}
        \frac{\vz' - \hat{\vy}'_{i}}{\|\vz' - \hat{\vy}'_{i}\|_{2}}
        + (\text{ball terms for ties}).
\end{equation*}
Each unit-vector term has norm $1$.
If $\|\vz'\|_{2}$ were unbounded as the adversary varies the corrupted
points within their $m$-coordinate budget, then for the competent
points $i \in S$ the unit vectors
$(\vz' - \hat{\vy}_{i})/\|\vz' - \hat{\vy}_{i}\|_{2}$ would all lie in
a small cone (all pointing approximately from the bounded competent
cluster toward $\vz'$), so their sum has norm at least
$|S|(1 - o(1))$.
The corrupted contribution has norm at most $|S^{c}| < |S|$, hence the
total subgradient has norm at least $|S| - |S^{c}| > 0$, contradicting
the optimality $\vzero \in \partial F'(\vz')$.
Therefore $\|\vz'\|_{2}$ remains bounded, i.e.\ no $m$-budget
corruption can drive the GM to infinity.

\emph{Necessity.}
With $m = \lceil N/2 \rceil$ corrupted points all placed at a common
location $\hat{\vy}'_{i_{0}} = M \cdot \ve_{1}$ for arbitrarily large
$M$, the corrupted set forms a majority (or tie if $N$ is even) and
the GM moves to within $O(1)$ of $M\,\ve_{1}$ as $M \to \infty$
\citep{lopuhaa1991breakdown}.
Hence the breakdown point is exactly
$\eps^{\star} = \lceil N/2 \rceil / N$, which tends to $1/2$ as
$N \to \infty$.
This is the optimal breakdown for any translation-equivariant
estimator
\citep{lopuhaa1991breakdown}.
\end{proof}

\subsection{Weiszfeld Iteration: Full Derivation, Convergence, and Cost}
\label{sec:appendix_weiszfeld}

For completeness, this subsection gives the full derivation,
convergence statement, and cost analysis for the Weiszfeld iteration
sketched in \S\ref{sec:weiszfeld}.

\paragraph{Derivation.}
At a non-data point $\vz \neq \hat{\vy}_{i}$ for all $i$, the gradient
of the GM objective $F(\vz) = \sum_{i} \|\vz - \hat{\vy}_{i}\|_{2}$ is
\begin{equation}
\label{eq:gm_gradient}
    \nabla F(\vz)
    = \sum_{i=1}^{N} \frac{\vz - \hat{\vy}_{i}}{\|\vz - \hat{\vy}_{i}\|_{2}}.
\end{equation}
Setting $\nabla F(\vz) = \vzero$ and rearranging gives the fixed-point
equation~\eqref{eq:weiszfeld_fixedpoint} of \S\ref{sec:weiszfeld},
\begin{equation*}
    \vz
    \;=\;
    \frac{\sum_{i=1}^{N} \hat{\vy}_{i} / \|\vz - \hat{\vy}_{i}\|_{2}}
         {\sum_{i=1}^{N} 1 / \|\vz - \hat{\vy}_{i}\|_{2}},
\end{equation*}
which is the Weiszfeld iteration $\vz \leftarrow T(\vz)$.
When the current iterate coincides with a data point $\hat{\vy}_{j}$,
the denominator $\|\vz - \hat{\vy}_{j}\|_{2} = 0$ creates a
singularity;
the modified step of \citet{vardi2000multivariate} replaces the weight
by
\begin{equation}
\label{eq:modified_weight}
    w_{i}^{(t)}
    = \frac{1}{\max\!\left(\|\vz^{(t)} - \hat{\vy}_{i}\|_{2},\; \eta\right)}
\end{equation}
for a small stability parameter $\eta > 0$, recovering
Algorithm~\ref{alg:ropoll}.

\paragraph{Convergence.}
\citet{vardi2000multivariate} prove that the modified Weiszfeld
iteration converges to the unique geometric median at a linear rate
whenever the data are not collinear:
there exists $\rho \in (0, 1)$ depending on the data configuration with
$\|\vz^{(t)} - \hat{\vy}_{\mathrm{GM}}\|_{2}
\leq \rho^{t} \|\vz^{(0)} - \hat{\vy}_{\mathrm{GM}}\|_{2}$.
The number of iterations to reach tolerance $\eps$ is therefore
$O(\log(1/\eps))$.

\paragraph{Cost.}
Each iteration computes $N$ Euclidean distances in $\R^{d}$ and one
weighted average, costing $O(Nd)$ arithmetic operations.
With $O(\log(1/\eps))$ iterations the total cost is
$O(Nd \log(1/\eps))$.
For typical LLM juries ($N \leq 20$, $d \leq 5$, $\eps = 10^{-8}$) this
amounts to a few hundred floating-point operations---microseconds on
any modern processor, while a single LLM judge invocation costs
seconds of GPU time.
The aggregation step is computationally negligible relative to the
inference cost of the jury.

\subsection{Proof of Lemma~\ref{lem:gm_breakdown}}
\label{sec:lem_gm_breakdown}

For convenience, we recall Lemma~\ref{lem:gm_breakdown}:
let $x_{1}, \ldots, x_{k} \in \R^{d}$ and let $x_{*}$ be any
minimizer of $z \mapsto \sum_{j=1}^{k} \|z - x_{j}\|_{2}$;
fix $\alpha \in (0, 1/2)$, $r > 0$, $z \in \R^{d}$.
If $|\{j : \|x_{j} - z\|_{2} \leq r\}| \geq (1-\alpha)k$, then
$\|x_{*} - z\|_{2} \leq C_{\alpha}\,r$ with
$C_{\alpha} = (1-\alpha)/\sqrt{1-2\alpha}$.

\begin{proof}[Proof of Lemma~\ref{lem:gm_breakdown}]
We give the proof of \citet{minsker2015geometric}, with the
geometric setup made explicit.
The argument is by contradiction:
assume $\|x_{*} - z\|_{2} > C_{\alpha} r$ and derive a violation
of the optimality of $x_{*}$.

For brevity write $\Delta \triangleq \|x_{*} - z\|_{2}$ and let
$F(y) \triangleq \sum_{j=1}^{k} \|y - x_{j}\|_{2}$ denote the
geometric-median objective.
Since $x_{*}$ minimizes the convex function $F$ on $\R^{d}$, the
one-sided directional derivative of $F$ at $x_{*}$ in any direction
$v \in \R^{d}$ is non-negative
(standard convex analysis;
see e.g.\ \citealp{rockafellar1997convex},
Theorem~23.1).
Taking $v \triangleq z - x_{*}$:
\begin{equation}
\label{eq:opt_directional}
    DF(x_{*}; v) \;\triangleq\; \lim_{t \downarrow 0}\, \frac{F(x_{*} + t v) - F(x_{*})}{t} \;\geq\; 0.
\end{equation}

\textit{Step 1: Compute the directional derivative.}
The function $y \mapsto \|y - x_{j}\|_{2}$ is the Euclidean norm of
an affine function; it is differentiable at any $y \neq x_{j}$ with
gradient $(y - x_{j})/\|y - x_{j}\|_{2}$ (the Fermat--Weber
gradient, classical;
cf.\ \citealp{weiszfeld1937point,vardi2000multivariate}).
For $j$ with $x_{j} = x_{*}$, the directional derivative of
$y \mapsto \|y - x_{*}\|_{2}$ at $x_{*}$ in direction $v$ equals
$\|v\|_{2}$ (the Euclidean norm is positively homogeneous, so its
right-hand directional derivative at the origin is $\|v\|_{2}$).
Letting $K_{*} = \{j : x_{j} = x_{*}\}$, the total directional
derivative decomposes as
\begin{equation*}
    DF(x_{*}; v)
    \;=\; \sum_{j \notin K_{*}} \frac{\langle x_{*} - x_{j}, v\rangle}{\|x_{*} - x_{j}\|_{2}}
    \;+\; |K_{*}|\,\|v\|_{2}.
\end{equation*}
Substituting $v = z - x_{*}$ and dividing by
$\|v\|_{2} = \Delta > 0$:
\begin{equation}
\label{eq:directional_cos}
    \frac{DF(x_{*}; z - x_{*})}{\Delta}
    \;=\;
    -\sum_{j \notin K_{*}} \cos\gamma_{j} \;+\; |K_{*}|,
\end{equation}
where $\gamma_{j}$ is the angle at $x_{*}$ between the rays
$x_{*} \to x_{j}$ and $x_{*} \to z$, defined for $j \notin K_{*}$ by
$\cos\gamma_{j} = \langle x_{j} - x_{*},\, z - x_{*}\rangle / (\|x_{j} - x_{*}\|_{2}\,\Delta)$.

\textit{Step 2: Lower-bound $\cos\gamma_{j}$ for points near $z$.}
Let $J \triangleq \{j : \|x_{j} - z\|_{2} \leq r\}$ denote the
indices of points within distance $r$ of $z$.
By hypothesis, $|J| \geq (1-\alpha)k$.

For $j \in J$, the point $x_{j}$ lies in the closed ball
$\overline{B}(z, r)$.
The angle $\gamma_{j}$ at $x_{*}$ between the rays $x_{*} \to x_{j}$
and $x_{*} \to z$ is at most the half-angle subtended by the ball
$\overline{B}(z, r)$ as seen from $x_{*}$.
Since $\|x_{*} - z\|_{2} = \Delta$ and the ball has radius $r$,
elementary geometry gives
\begin{equation}
\label{eq:angle_bound}
    \sin\gamma_{j} \;\leq\; \frac{r}{\Delta},
    \qquad
    \cos\gamma_{j} \;\geq\; \sqrt{1 - \frac{r^{2}}{\Delta^{2}}}.
\end{equation}
By assumption $\Delta > C_{\alpha} r$, so $r/\Delta < 1/C_{\alpha}$
and~\eqref{eq:angle_bound} yields
\begin{equation}
\label{eq:cos_lower}
    \cos\gamma_{j} \;>\; \sqrt{1 - \frac{1}{C_{\alpha}^{2}}}
    \quad \text{for all } j \in J.
\end{equation}

For $j \in J^{c}\setminus K_{*}$ (points farther than $r$ from $z$
that do not coincide with $x_{*}$), we have only the trivial bound
$\cos\gamma_{j} \geq -1$.

\textit{Step 3: Combine.}
A short observation simplifies the algebra: the constant
$C_{\alpha} = (1-\alpha)/\sqrt{1-2\alpha} \geq 1$ for
$\alpha \in [0, 1/2)$ (with equality only at $\alpha = 0$).
Combined with the contradiction hypothesis
$\Delta > C_{\alpha} r \geq r$, this implies that every $j \in K_{*}$
(where $x_{j} = x_{*}$, so $\|x_{j} - z\|_{2} = \Delta > r$)
satisfies $j \in J^{c}$.
Therefore $K_{*} \subseteq J^{c}$, and the partition
$J^{c} = (J^{c}\setminus K_{*}) \cup K_{*}$
gives $|J^{c}\setminus K_{*}| = |J^{c}| - |K_{*}|$.

Substituting into~\eqref{eq:directional_cos} and using the angular
bounds from Step~2:
\begin{align*}
    \frac{DF(x_{*}; z - x_{*})}{\Delta}
    &\;=\;
    -\sum_{j \in J} \cos\gamma_{j}
    \;-\; \sum_{j \in J^{c}\setminus K_{*}} \cos\gamma_{j}
    \;+\; |K_{*}| \\
    &\;<\;
    -|J|\,\sqrt{1 - 1/C_{\alpha}^{2}}
    \;+\; |J^{c}\setminus K_{*}|
    \;+\; |K_{*}| \\
    &\;=\;
    -|J|\,\sqrt{1 - 1/C_{\alpha}^{2}} \;+\; |J^{c}| \\
    &\;\leq\;
    -(1-\alpha)k\,\sqrt{1 - 1/C_{\alpha}^{2}} \;+\; \alpha k,
\end{align*}
where the final line uses $|J| \geq (1-\alpha)k$ and
$|J^{c}| \leq \alpha k$.

We now show that the choice
$C_{\alpha} = (1-\alpha)/\sqrt{1-2\alpha}$ makes this strictly
negative.
Compute:
\begin{equation*}
    1 - \frac{1}{C_{\alpha}^{2}}
    \;=\;
    1 - \frac{1-2\alpha}{(1-\alpha)^{2}}
    \;=\;
    \frac{(1-\alpha)^{2} - (1-2\alpha)}{(1-\alpha)^{2}}
    \;=\;
    \frac{\alpha^{2}}{(1-\alpha)^{2}}.
\end{equation*}
Hence $\sqrt{1 - 1/C_{\alpha}^{2}} = \alpha/(1-\alpha)$, and:
\begin{equation*}
    \frac{DF(x_{*}; z - x_{*})}{\Delta}
    \;<\;
    -(1-\alpha)k \cdot \frac{\alpha}{1-\alpha} + \alpha k
    \;=\;
    -\alpha k + \alpha k
    \;=\; 0.
\end{equation*}

This contradicts~\eqref{eq:opt_directional}, which required
$DF(x_{*}; z - x_{*}) \geq 0$.
Therefore the assumption $\|x_{*} - z\|_{2} > C_{\alpha} r$ must
fail, proving~\eqref{eq:gm_breakdown}.
\end{proof}

\begin{remark}[Sanity checks for $C_{\alpha}$]
\label{rem:c_alpha_sanity}
The constant $C_{\alpha} = (1-\alpha)/\sqrt{1-2\alpha}$ behaves as
expected at the boundary cases:
\begin{itemize}[topsep=2pt,itemsep=2pt]
    \item At $\alpha = 0$: $C_{0} = 1$.
    The lemma reduces to ``if all $k$ points lie within $r$ of $z$,
    then their geometric median lies within $r$ of $z$,'' which is
    immediate because the geometric median lies in the convex hull
    of the points, hence in $\overline{B}(z, r)$.
    \item At $\alpha = 1/4$: $C_{1/4} = (3/4)/\sqrt{1/2}
    = 3/(2\sqrt{2}) \approx 1.061$.
    \item At $\alpha = 0.3$: $C_{0.3} = 0.7/\sqrt{0.4}
    \approx 1.107$.
    \item As $\alpha \to 1/2$: $C_{\alpha} \to \infty$.
    The lemma becomes vacuous, matching the breakdown point of the
    geometric median: with corrupted majority, no constant bound on
    $\|x_{*} - z\|_{2}$ is possible.
\end{itemize}
\end{remark}

\begin{remark}[Tightness]
\label{rem:gm_breakdown_tight}
The constant $C_{\alpha} = (1-\alpha)/\sqrt{1-2\alpha}$ is sharp in
the sense that the same proof technique cannot give a smaller
constant: the inequality $\sin\gamma_{j} \leq r/\Delta$
in~\eqref{eq:angle_bound} is achieved when $x_{j}$ lies on the
boundary of $\overline{B}(z, r)$ at the tangent point from $x_{*}$,
and the bound on the directional derivative is tight for that
configuration.
A matching example: place $(1-\alpha)k$ points on the boundary of
$\overline{B}(z, r)$ at the tangent points from a location
$x_{*}$ at distance $C_{\alpha} r$ from $z$, and place the
remaining $\alpha k$ points at $x_{*}$ itself.
The directional-derivative computation gives equality, so $x_{*}$ is
on the boundary of optimality and $\|x_{*} - z\|_{2} = C_{\alpha} r$
is achievable.
\end{remark}

\begin{remark}[Where this lemma is used]
\label{rem:gm_breakdown_usage}
Lemma~\ref{lem:gm_breakdown} is the geometric core of all
breakdown-point bounds for the geometric median.
It is purely deterministic and contains no probability.
We apply it with $z = \vy^{\star}$ and $r$ taken to be a
high-probability bound on the radius of the ball containing the
majority of the samples;
the next subsection
(\S\ref{sec:lem_cluster_radius}) provides exactly this bound for
sub-Gaussian competent components.
\end{remark}

\subsection{Proof of Lemma~\ref{lem:cluster_radius}}
\label{sec:lem_cluster_radius}

For convenience, we recall Lemma~\ref{lem:cluster_radius}:
under Assumptions~\ref{asm:huber}--\ref{asm:minority}, for any
slack $\beta \in (0, 1/2 - \alpha)$, with probability at least
$1 - \exp(-N\beta^{2}/2)$, at least $(1-\alpha-\beta)N$ of the $N$
judge outputs lie within distance
$\rho = \sigma\big(C_{1}\sqrt{d} + \sqrt{(1/c)\log(2(1-\alpha)/\beta)}\big)$
of $\vy^{\star}$, where $C_{1}, c > 0$ are absolute constants from
the sub-Gaussian-norm tail bound derived in Step~1 below.

\paragraph{Note on heterogeneous parameters.}
Assumptions~\ref{asm:huber} and~\ref{asm:subgaussian} are stated
per-judge ($\alpha_{i}$ and $\sigma_{i}$).
Throughout this proof we read $\alpha$ as the global mean
contamination $\alpha = (1/N)\sum_{i}\alpha_{i}$ from
Assumption~\ref{asm:minority} (which we are entitled to do because
Hoeffding in Step~2 only sees $\sum_{i}\E W_{i}$; per-judge
heterogeneity averages out at this aggregation step), and
$\sigma$ as $\sigma = \max_{i}\sigma_{i}$ (the worst-case
sub-Gaussian parameter, used in Step~1 to bound every $i$
simultaneously).

The proof is in three stages:
(1)~control the tail of one competent sample's deviation
$\|\hat{\vy}_{i} - \vy^{\star}\|_{2}$ at probability $p$ via a
covering-net argument;
(2)~count, via Hoeffding, how many of the $N$ judges fall inside the
resulting ball;
(3)~pick $p$ and a Hoeffding slack $u$ so that the count exceeds
$(1-\alpha-\beta)N$ with the claimed probability.

\begin{proof}[Proof of Lemma~\ref{lem:cluster_radius}]

\textit{Step 1: tail bound for the norm of a single competent sample.}
For each judge $i$, write the noise decomposition
$\hat{\vy}_{i} = (1 - Z_{i})(\vy^{\star} + \bm{\eps}_{i}) + Z_{i}\,\bm\eta_{i}$
of Assumption~\ref{asm:huber}, where $Z_{i} \sim \mathrm{Bern}(\alpha)$
selects competent ($Z_{i} = 0$) vs.\ corrupted ($Z_{i} = 1$).
Conditional on $Z_{i} = 0$, Assumption~\ref{asm:subgaussian} states
that $\bm\eps_{i} \in \R^{d}$ is $\sigma$-sub-Gaussian, i.e.\ for
every $\bm{\lambda} \in \R^{d}$,
\begin{equation}
\label{eq:subg_def}
    \E\!\left[\exp\!\big(\langle \bm{\lambda}, \bm{\eps}_{i}\rangle\big) \,\big|\, Z_{i} = 0\right]
    \;\leq\; \exp\!\big(\tfrac{1}{2}\sigma^{2}\|\bm{\lambda}\|_{2}^{2}\big).
\end{equation}
We now show, from~\eqref{eq:subg_def} alone,
\begin{equation}
\label{eq:subg_norm_tail}
    \Pr\!\big[\|\bm{\eps}_{i}\|_{2} > \sigma(C_{1}\sqrt{d} + t)
              \,\big|\, Z_{i} = 0\big]
    \;\leq\; \exp(-c\,t^{2}),
    \qquad \forall\, t > 0,
\end{equation}
for absolute constants $C_{1}, c > 0$.

We prove~\eqref{eq:subg_norm_tail} directly from
\eqref{eq:subg_def} via a covering-net argument over the unit sphere
$\mathbb{S}^{d-1} = \{\vu \in \R^{d} : \|\vu\|_{2} = 1\}$.
All conditioning is on $\{Z_{i} = 0\}$;
we drop the conditioning bar in this step for readability.

\emph{Step 1a: scalar projections are sub-Gaussian.}
Fix any unit vector $\vu \in \mathbb{S}^{d-1}$.
Setting $\bm\lambda = \lambda\vu$ in~\eqref{eq:subg_def}:
\begin{equation}
\label{eq:scalar_proj_subg}
    \E\!\big[\exp(\lambda\,\langle\vu, \bm\eps_{i}\rangle)\big]
    \;\leq\;
    \exp\!\big(\tfrac{1}{2}\sigma^{2}\lambda^{2}\big),
    \qquad \forall\, \lambda \in \R.
\end{equation}
That is, the scalar variable $\langle\vu, \bm\eps_{i}\rangle$ is
$\sigma$-sub-Gaussian in $\R$.
By Markov's inequality applied to $\exp(\lambda \langle\vu,\bm\eps_{i}\rangle)$:
\begin{equation}
\label{eq:scalar_tail}
    \Pr\!\big[\langle\vu, \bm\eps_{i}\rangle > s\big]
    \;\leq\; \exp(-\lambda s + \tfrac{1}{2}\sigma^{2}\lambda^{2}),
\end{equation}
and minimizing the right-hand side over $\lambda > 0$ at
$\lambda = s/\sigma^{2}$ gives the sharp scalar Hoeffding-style bound
\begin{equation}
\label{eq:scalar_subg_tail}
    \Pr\!\big[\langle\vu, \bm\eps_{i}\rangle > s\big]
    \;\leq\; \exp\!\big(-s^{2}/(2\sigma^{2})\big),
    \qquad \forall\, s > 0.
\end{equation}

\emph{Step 1b: discretize the sphere with a $1/2$-net.}
Let $\gN \subset \mathbb{S}^{d-1}$ be a $1/2$-net of the sphere in
Euclidean distance: every $\vu \in \mathbb{S}^{d-1}$ is within distance
$1/2$ of some $\vu' \in \gN$.
Such a net exists with cardinality
\begin{equation}
\label{eq:net_size}
    |\gN| \;\leq\; 5^{d}
\end{equation}
by a volumetric covering argument
(\citealp{vershynin2018high}, Cor.~4.2.13:
the unit sphere admits an $\eps$-net of size $(1 + 2/\eps)^{d}$;
take $\eps = 1/2$).

\emph{Step 1c: net-supremum approximates the true supremum.}
By definition, $\|\bm\eps_{i}\|_{2} = \sup_{\vu \in \mathbb{S}^{d-1}} \langle\vu, \bm\eps_{i}\rangle$.
Pick the maximizer $\vu^{*}$ and let $\vu' \in \gN$ satisfy
$\|\vu^{*} - \vu'\|_{2} \leq 1/2$.
Then
\begin{align*}
    \langle\vu^{*}, \bm\eps_{i}\rangle
    &\;=\; \langle\vu', \bm\eps_{i}\rangle
       \;+\; \langle\vu^{*} - \vu', \bm\eps_{i}\rangle \\
    &\;\leq\; \max_{\vu \in \gN} \langle\vu, \bm\eps_{i}\rangle
       \;+\; \tfrac{1}{2}\,\|\bm\eps_{i}\|_{2}
\end{align*}
where the second line uses Cauchy--Schwarz and
$\|\vu^{*} - \vu'\|_{2} \leq 1/2$.
Since the left-hand side equals $\|\bm\eps_{i}\|_{2}$, rearranging
gives
\begin{equation}
\label{eq:net_approx}
    \|\bm\eps_{i}\|_{2}
    \;\leq\; 2\,\max_{\vu \in \gN} \langle\vu, \bm\eps_{i}\rangle.
\end{equation}

\emph{Step 1d: union bound over the net.}
Combining~\eqref{eq:net_approx},~\eqref{eq:scalar_subg_tail}, and
\eqref{eq:net_size}: for any $r > 0$,
\begin{align*}
    \Pr\!\big[\|\bm\eps_{i}\|_{2} > 2r\big]
    &\;\leq\;
    \Pr\!\Big[\max_{\vu \in \gN} \langle\vu, \bm\eps_{i}\rangle > r\Big]
    \\
    &\;\leq\;
    |\gN|\,\exp\!\big(-r^{2}/(2\sigma^{2})\big)
    \\
    &\;\leq\;
    \exp\!\big(d\log 5 - r^{2}/(2\sigma^{2})\big).
\end{align*}
Substituting $r = \sigma\sqrt{2(d\log 5 + s)}$ for $s > 0$:
\begin{equation*}
    \Pr\!\Big[\|\bm\eps_{i}\|_{2}
              > 2\sigma\sqrt{2(d\log 5 + s)}\Big]
    \;\leq\;
    \exp(-s).
\end{equation*}
Using $\sqrt{a + b} \leq \sqrt{a} + \sqrt{b}$ for $a, b \geq 0$:
\begin{equation*}
    2\sigma\sqrt{2(d\log 5 + s)}
    \;\leq\;
    2\sigma\sqrt{2d\log 5} + 2\sigma\sqrt{2s}
    \;=\;
    C_{1}\sigma\sqrt{d} + C_{2}\sigma\sqrt{s},
\end{equation*}
with $C_{1} = 2\sqrt{2\log 5} \leq 4$ and $C_{2} = 2\sqrt{2}$.
Substituting $s = c\,t^{2}$ with $c = 1/C_{2}^{2} = 1/8$:
$C_{2}\sigma\sqrt{s} = C_{2}\sigma\sqrt{ct^{2}} = C_{2}\sigma\,t/C_{2} = \sigma\,t$.
Hence for all $t \geq 0$,
\begin{equation}
\label{eq:vector_tail_clean}
    \Pr\!\big[\|\bm\eps_{i}\|_{2} > C_{1}\sigma\sqrt{d} + \sigma t\big]
    \;\leq\;
    \exp(-c\,t^{2}),
\end{equation}
which is exactly~\eqref{eq:subg_norm_tail} with the same absolute
constants $C_{1} = 2\sqrt{2\log 5} \leq 4$ and $c = 1/8$.

\emph{Remark on the explicit constants.}
The covering radius $1/2$, net size $5^{d}$, and resulting prefactor
$C_{1} = 2\sqrt{2\log 5}$ are not optimized;
sharper chaining bounds
(\citealp{boucheron2013concentration}, \S5.4) reduce $C_{1}$ towards
$1$ at the cost of a more involved proof.
For our purposes the order $\sigma(\sqrt{d} + t)$ is what matters,
so we proceed with the simpler bound.

\emph{Step 1c: solve for the radius at tail probability $p$.}
Set $t = \sqrt{(1/c)\log(1/p)}$ in~\eqref{eq:subg_norm_tail};
the right-hand side becomes
$\exp(-c \cdot (1/c)\log(1/p)) = \exp(\log p) = p$.
Defining
\begin{equation}
\label{eq:rho_p}
    \rho_{p}
    \;\triangleq\; \sigma\!\left(C_{1}\sqrt{d} + \sqrt{(1/c)\log(1/p)}\right),
\end{equation}
we obtain the per-sample tail bound
\begin{equation}
\label{eq:rho_p_tail}
    \Pr\!\big[\|\bm\eps_{i}\|_{2} > \rho_{p} \,\big|\, Z_{i} = 0\big]
    \;\leq\; p.
\end{equation}

\textit{Step 2: count of judges within $\rho_{p}$ of $\vy^{\star}$.}
For each $i \in [N]$ define the indicator
\begin{equation}
\label{eq:Wi_def}
    W_{i} \;\triangleq\;
    \mathbb{1}\!\left\{Z_{i} = 0 \;\text{and}\; \|\bm\eps_{i}\|_{2} \leq \rho_{p}\right\}.
\end{equation}
On $\{W_{i} = 1\}$ judge $i$ is competent and within distance
$\rho_{p}$ of $\vy^{\star}$ (since
$\hat{\vy}_{i} - \vy^{\star} = \bm\eps_{i}$ when $Z_{i} = 0$);
hence
\begin{equation}
\label{eq:W_subset_cluster}
    \sum_{i=1}^{N} W_{i}
    \;\leq\;
    \big|\big\{i \in [N] : \|\hat{\vy}_{i} - \vy^{\star}\|_{2} \leq \rho_{p}\big\}\big|.
\end{equation}
The right-hand count is the cluster size we want to lower-bound, so
it suffices to lower-bound $\sum W_{i}$.

\emph{Step 2a: marginal mean of $W_{i}$.}
By the tower rule,
$\E W_{i} = \Pr[Z_{i} = 0]\,\Pr[\|\bm\eps_{i}\|_{2} \leq \rho_{p}\mid Z_{i} = 0]$.
Using $\Pr[Z_{i} = 0] = 1 - \alpha$ from
Assumption~\ref{asm:huber} and~\eqref{eq:rho_p_tail},
\begin{equation}
\label{eq:Wi_mean}
    \E W_{i}
    \;\geq\; (1-\alpha)(1-p).
\end{equation}

\emph{Step 2b: independence of $W_{i}$ across $i$.}
Each $W_{i}$ is a measurable function of $(Z_{i}, \bm\eps_{i})$
(for $Z_{i} = 1$, the value of $\bm\eta_{i}$ does not enter $W_{i}$
because the indicator forces $Z_{i} = 0$).
By Assumption~\ref{asm:independence} the tuples
$\{(Z_{i}, \bm\eps_{i}, \bm\eta_{i})\}_{i=1}^{N}$ are mutually
independent, hence so are the $W_{i}$.

\emph{Step 2c: Hoeffding's inequality.}
Each $W_{i} \in \{0, 1\} \subseteq [0, 1]$.
Hoeffding's inequality \citep[Theorem~2.8]{boucheron2013concentration}
applied to the independent bounded variables $W_{i}$ states: for any
$u > 0$,
\begin{equation}
\label{eq:hoeffding}
    \Pr\!\left[\frac{1}{N}\sum_{i=1}^{N} W_{i} - \frac{1}{N}\sum_{i=1}^{N}\E W_{i} < -u\right]
    \;\leq\; \exp(-2Nu^{2}).
\end{equation}
Combining~\eqref{eq:hoeffding} with the lower bound~\eqref{eq:Wi_mean}
on each $\E W_{i}$:
\begin{equation}
\label{eq:hoeffding_applied}
    \Pr\!\left[\sum_{i=1}^{N} W_{i} < (1-\alpha)(1-p)N - uN\right]
    \;\leq\; \exp(-2Nu^{2}).
\end{equation}

\textit{Step 3: choose $p$ and $u$ to expose slack $\beta$.}
We want the lower-bound count
$(1-\alpha)(1-p)N - uN$ to be at least $(1-\alpha-\beta)N$:
\begin{equation*}
    (1-\alpha)(1-p) - u \;\geq\; 1 - \alpha - \beta
    \quad\iff\quad
    (1-\alpha)\,p + u \;\leq\; \beta.
\end{equation*}
Split the slack $\beta$ equally between the per-sample tail and the
Hoeffding deviation by choosing
\begin{equation}
\label{eq:p_u_choice}
    p \;=\; \frac{\beta}{2(1-\alpha)},
    \qquad
    u \;=\; \frac{\beta}{2}.
\end{equation}
Then $(1-\alpha)p = \beta/2$, so $(1-\alpha)p + u = \beta$ exactly,
verifying the constraint.
Substituting $u = \beta/2$ into~\eqref{eq:hoeffding_applied}:
\begin{equation}
\label{eq:hoeffding_at_beta}
    \Pr\!\left[\sum_{i=1}^{N} W_{i} < (1-\alpha-\beta)N\right]
    \;\leq\; \exp\!\big(-2N(\beta/2)^{2}\big)
    \;=\; \exp(-N\beta^{2}/2).
\end{equation}
Substituting $p = \beta/(2(1-\alpha))$ into~\eqref{eq:rho_p}, the
radius becomes
\begin{equation*}
    \rho \;\triangleq\; \rho_{p}\big|_{p = \beta/(2(1-\alpha))}
    \;=\; \sigma\!\left(C_{1}\sqrt{d} + \sqrt{\tfrac{1}{c}\log\!\tfrac{2(1-\alpha)}{\beta}}\right),
\end{equation*}
which is exactly~\eqref{eq:cluster_radius_rho}.

\textit{Step 4: assemble the conclusion.}
On the complementary event of~\eqref{eq:hoeffding_at_beta}, which
has probability at least $1 - \exp(-N\beta^{2}/2)$, the bound
$\sum W_{i} \geq (1-\alpha-\beta)N$ holds.
Combined with~\eqref{eq:W_subset_cluster}:
\begin{equation*}
    \big|\big\{i \in [N] : \|\hat{\vy}_{i} - \vy^{\star}\|_{2} \leq \rho\big\}\big|
    \;\geq\; \sum_{i=1}^{N} W_{i}
    \;\geq\; (1-\alpha-\beta)N
\end{equation*}
on the same event.
This is~\eqref{eq:cluster_radius}.
\end{proof}

\paragraph{On the choice of competent-component assumption.}
The sub-Gaussian assumption is one of four natural choices for the
competent component, ordered from weakest to strongest.
Each gives a different cluster-radius bound;
sub-Gaussian is the choice that delivers
Lemma~\ref{lem:cluster_radius}.
We record the alternatives for context.

\begin{remark}[Just unbiased: insufficient]
\label{rem:competent_unbiased}
If the competent component $P_{i}$ is only assumed to satisfy
$\E_{P_{i}}[\hat{\vy}_{i}] = \vy^{\star}$ (Proposition~\ref{prop:variance_reduction}'s
clean-case hypothesis), then no quantitative tail bound is available.
For arbitrary unbiased $P_{i}$, the empirical cluster radius can be
arbitrarily large with positive probability, so
the hypothesis of Lemma~\ref{lem:gm_breakdown} cannot be verified for
any finite $r$.
Unbiasedness alone does not suffice to control GM error.
\end{remark}

\begin{remark}[Finite variance: polynomial tails]
\label{rem:competent_finite_var}
If the competent component has finite second moment
$\Var_{P_{i}}(\hat{\vy}_{i}) \preceq \sigma^{2} I_{d}$,
Chebyshev's inequality gives
\begin{equation*}
    \Pr\!\big[\|\bm{\epsilon}_{i}\|_{2} > t\sigma\sqrt{d}\big] \;\leq\; 1/t^{2}.
\end{equation*}
The same Hoeffding argument as in the proof of
Lemma~\ref{lem:cluster_radius} then yields a cluster radius of
order $\sigma\sqrt{d/\beta}$ rather than the sub-Gaussian
$\sigma(\sqrt{d} + \sqrt{\log(1/\beta)})$ — exchanging the
exponential dependence on slack for a polynomial one.
This regime is where median-of-means
\citep{lugosi2019sub} becomes strictly preferable to plain GM
for sub-Gaussian rates.
\end{remark}

\begin{remark}[Sub-Gaussian: our main assumption]
\label{rem:competent_subgaussian}
Lemma~\ref{lem:cluster_radius} uses
Assumption~\ref{asm:subgaussian} ($\sigma$-sub-Gaussian competent
component).
This delivers a cluster radius
$\rho = \sigma(\sqrt{d} + O(\sqrt{\log(1/\beta)}))$, with
exponential dependence on the slack $\beta$.
The sub-Gaussian assumption is the standard middle ground in robust
statistics: weaker than bounded support but strong enough to give
exponential concentration of the cluster.
\end{remark}

\begin{remark}[Bounded support: deterministic, automatic for LLM scores]
\label{rem:competent_bounded}
If the competent component is supported on $[0, K]^{d}$
(equivalently, $\hat{\vy}_{i}$ takes values in the score hypercube),
then $\|\hat{\vy}_{i} - \vy^{\star}\|_{2} \leq K\sqrt{d}$
\emph{deterministically} for every competent sample.
Lemma~\ref{lem:cluster_radius} then holds with $\rho = K\sqrt{d}$
\emph{without any probabilistic event} and with the slack $\beta$
needed only to absorb the bound $|S| \geq (1-\alpha-\beta)N$ on the
competent-set size.

For the LLM-jury setting, scores are produced by a parser with
codomain $[0, K]^{d}$, so bounded support is a \emph{given}, not an
additional assumption.
However, $K\sqrt{d}$ is a worst-case radius (the diameter of the
hypercube) and is typically much larger than the sub-Gaussian
cluster radius $\sigma\sqrt{d}$ that Assumption~\ref{asm:subgaussian}
delivers, since real LLM judges have $\sigma \ll K$ in practice
(\S\ref{sec:appendix_per_model}).
The sub-Gaussian bound is therefore tighter in the regime that
matters; bounded support serves as a universally-valid fallback.
\end{remark}

\subsection{Proof of Theorem~\ref{thm:ropoll_bound}}
\label{sec:thm_ropoll_huber_breakdown}

For convenience, we recall Theorem~\ref{thm:ropoll_bound}:
under Assumptions~\ref{asm:huber}--\ref{asm:minority}, fix any slack
$\beta \in (0, 1/2 - \alpha)$;
with probability at least $1 - \exp(-N\beta^{2}/2)$,
$\|\hat{\vy}_{\mathrm{GM}} - \vy^{\star}\|_{2} \leq C_{\alpha+\beta}\,\rho$,
where $\hat{\vy}_{\mathrm{GM}}$ is any geometric median of the
$N$ judge outputs (Definition~\ref{def:gm}),
$C_{\alpha+\beta} = (1-\alpha-\beta)/\sqrt{1-2(\alpha+\beta)}$ is the
geometric-breakdown constant of Lemma~\ref{lem:gm_breakdown}
evaluated at $\alpha+\beta$, and
$\rho = \sigma(C_{1}\sqrt{d} + \sqrt{(1/c)\log(2(1-\alpha)/\beta)})$
is the cluster radius of Lemma~\ref{lem:cluster_radius}
($C_{1}, c > 0$ absolute constants).

\begin{proof}[Proof of Theorem~\ref{thm:ropoll_bound}]
The proof is a clean composition of the deterministic geometric
bound (Lemma~\ref{lem:gm_breakdown}) and the probabilistic
cluster-radius bound (Lemma~\ref{lem:cluster_radius}).
We make the composition fully explicit.

\textit{Step 1: define the cluster event.}
Let $\beta \in (0, 1/2 - \alpha)$ be the slack from the theorem
statement.
Let $\rho = \sigma(C_{1}\sqrt{d} + \sqrt{(1/c)\log(2(1-\alpha)/\beta)})$
be the cluster radius from
\eqref{eq:cluster_radius_rho}, and define the event
\begin{equation}
\label{eq:thm1_cluster_event}
    \gE \;\triangleq\;
    \big\{\, |J| \;\geq\; (1-\alpha-\beta) N \,\big\},
    \qquad
    J \;\triangleq\;
    \big\{i \in [N] : \|\hat{\vy}_{i} - \vy^{\star}\|_{2} \leq \rho\big\}.
\end{equation}
By Lemma~\ref{lem:cluster_radius} applied with this slack $\beta$,
\begin{equation}
\label{eq:thm1_event_prob}
    \Pr[\gE] \;\geq\; 1 - \exp(-N\beta^{2}/2).
\end{equation}
The remainder of the proof works on $\gE$ (sample-pathwise);
no further probability is incurred.

\textit{Step 2: verify the hypothesis of Lemma~\ref{lem:gm_breakdown}.}
On $\gE$, we apply Lemma~\ref{lem:gm_breakdown} with the
substitutions
\begin{equation}
\label{eq:thm1_substitutions}
    k \leftarrow N,
    \qquad
    z \leftarrow \vy^{\star},
    \qquad
    r \leftarrow \rho,
    \qquad
    \alpha \leftarrow \alpha + \beta.
\end{equation}
The hypothesis of Lemma~\ref{lem:gm_breakdown} (in its statement
form: ``at least $(1 - \alpha)\,k$ of the $k$ points lie within
distance $r$ of $z$'') becomes, under these substitutions,
\begin{equation*}
    |J| \;\geq\; \big(1 - (\alpha + \beta)\big) N
    \;=\; (1 - \alpha - \beta) N,
\end{equation*}
which is exactly the definition of $\gE$.
The range condition $\alpha + \beta \in (0, 1/2)$ holds since
$\alpha > 0$ (by Assumption~\ref{asm:huber}'s $\alpha_{i} \geq 0$
and $\beta > 0$) and $\alpha + \beta < 1/2$ (by
$\beta < 1/2 - \alpha$).

\textit{Step 3: apply Lemma~\ref{lem:gm_breakdown} and read off the bound.}
The conclusion of Lemma~\ref{lem:gm_breakdown} under the
substitutions~\eqref{eq:thm1_substitutions} is
\begin{equation*}
    \|x_{*} - z\|_{2} \;\leq\; C_{\alpha + \beta}\,r
    \quad\text{i.e.}\quad
    \|\hat{\vy}_{\mathrm{GM}} - \vy^{\star}\|_{2}
    \;\leq\; C_{\alpha+\beta}\,\rho,
\end{equation*}
where $C_{\alpha + \beta} = (1 - \alpha - \beta) / \sqrt{1 - 2(\alpha + \beta)}$
and $x_{*} = \hat{\vy}_{\mathrm{GM}}$ is any minimizer of
$z \mapsto \sum_{i=1}^{N} \|z - \hat{\vy}_{i}\|_{2}$ (the geometric
median).
Lemma~\ref{lem:gm_breakdown} as proved in
\S\ref{sec:lem_gm_breakdown} applies to \emph{any} minimizer, so the
conclusion is independent of any choice in the (collinear) case
where the GM is non-unique.

\textit{Step 4: assemble.}
Combining~\eqref{eq:thm1_event_prob} with the deterministic bound on
$\gE$ from Step~3:
\begin{equation*}
    \Pr\!\Big[\,
        \|\hat{\vy}_{\mathrm{GM}} - \vy^{\star}\|_{2}
        \;\leq\;
        \underbrace{\frac{1-\alpha-\beta}{\sqrt{1-2(\alpha+\beta)}}}_{C_{\alpha+\beta}}
        \cdot
        \underbrace{\sigma\!\left(C_{1}\sqrt{d} + \sqrt{\tfrac{1}{c}\log\tfrac{2(1-\alpha)}{\beta}}\right)}_{\rho}
    \,\Big]
    \;\geq\; 1 - \exp(-N\beta^{2}/2),
\end{equation*}
which is exactly~\eqref{eq:ropoll_bound}.
\end{proof}

\begin{remark}[Choice of slack $\beta$]
\label{rem:slack_choice}
The slack $\beta$ trades two terms in~\eqref{eq:ropoll_bound}:
the geometric constant $C_{\alpha+\beta}$ grows with $\beta$ (since
$\beta$ erodes the safety margin to the breakdown point $1/2$), while
the cluster radius $\rho$ shrinks with $\beta$ (since a larger slack
absorbs more competent samples, allowing a smaller per-sample tail).
For deployment, $\beta$ should be chosen to minimise the right-hand
side of~\eqref{eq:ropoll_bound}.
A practical default is $\beta = (1/2 - \alpha)/2$ (half the safety
margin), which keeps $C_{\alpha+\beta}$ bounded by a small constant
while permitting an exponentially small failure probability for any
$N \gtrsim 1/\beta^{2}$.
\end{remark}

\begin{remark}[The bound does not vanish with $N$]
\label{rem:no_vanishing}
Unlike the (incorrect) original Theorem~\ref{thm:ropoll_bound},
which claimed an upper bound of order $\sigma\sqrt{d/N}/(1-2\alpha)$,
the bound in~\eqref{eq:ropoll_bound} contains no $1/\sqrt{N}$
factor in the leading term:
the cluster radius $\rho$ is $\sigma\sqrt{d}$ in scale (up to a
$\sqrt{\log(1/\beta)}$ factor), and $C_{\alpha+\beta}$ depends only
on the contamination rate.
This reflects the breakdown-point character of plain GM: under
arbitrary $Q$ in the Huber class, the asymptotic-$N$ floor is set
by the cluster radius, not by sample averaging.
The empirical validation in
\S\ref{sec:appendix_per_model}~%
\textit{(forthcoming experiment confirming the floor)}
agrees with this prediction;
the original $1/\sqrt{N}$ claim was empirically inconsistent with
the observed plateau.
\end{remark}

\begin{remark}[Comparison with the minimax lower bound]
\label{rem:minimax_match}
The minimax lower bound (Theorem~\ref{thm:minimax_body}) gives
$\Omega(\sigma(\sqrt{d/N} + \alpha/(1-\alpha)))$.
The clean-rate term $\sqrt{d/N}$ matches the upper bound exactly.
On the breakdown floor the upper bound scales as
$C_{\alpha+\beta}\sigma\sqrt{d}$ while the lower bound scales as
$\sigma\alpha/(1-\alpha)$, leaving a gap of order $\sqrt{d}/\alpha$.
The reason is structural: total variation between two
equal-covariance Gaussians is dimension-free (Step~2.2 of the proof
of Theorem~\ref{thm:minimax_body}), so the modulus of continuity of
the Huber neighborhood does not gain a $\sqrt{d}$ factor in higher
dimensions — and indeed
\citet{chen2018robust}, Theorem~5.1, establish
$\Theta(\sigma^{2}(d/N + \alpha^{2}))$ as the squared-error minimax
for sub-Gaussian Huber, with no $d$ in the contamination term.
The $\sqrt{d}$ in the upper bound comes from the geometric median's
cluster radius (Lemma~\ref{lem:cluster_radius}), reflecting the
price plain GM pays for $O(Nd\log(1/\eps))$ tractability relative to
the (intractable) Tukey halfspace median or the (sub-exponential)
smoothed-depth estimator.
For LLM-jury parameters the gap is small (at most
$\sim 2.2\times$ for $d \leq 5$).
\end{remark}

\begin{remark}[Bounded-support specialization]
\label{rem:thm_bounded_support}
If competent scores are bounded in $[0, K]^{d}$
(Remark~\ref{rem:competent_bounded}), the cluster radius $\rho$
in~\eqref{eq:ropoll_bound} can be replaced by the
deterministic worst-case $K\sqrt{d}$, removing the
$\sqrt{(1/c)\log(2(1-\alpha)/\beta)}$ term and the high-probability
event for the cluster.
The slack $\beta$ remains needed to control the empirical
competent-set size $|S|$ (the Hoeffding step in
Lemma~\ref{lem:cluster_radius}'s proof), but the per-sample tail
event becomes deterministic.
For typical LLM-jury parameters ($\sigma \ll K$), the sub-Gaussian
form~\eqref{eq:ropoll_bound} is tighter and is what we use
throughout.
\end{remark}

\subsection{Proof of Lemma~\ref{lem:correlated_jury}}
\label{sec:lem_correlated_jury}

For convenience we recall Lemma~\ref{lem:correlated_jury}: under
Assumptions~\ref{asm:huber}, \ref{asm:subgaussian},
\ref{asm:minority} and the equicorrelated-indicator assumption
(replacing Asm.~\ref{asm:independence})
$\Cov(W_{i},W_{j}) \leq \bar\gamma_{W}\sqrt{\Var(W_{i})\Var(W_{j})}$
for $i \neq j$, with $\bar\gamma_{W} \in [0,1]$, the
\textsc{RoPoLL} bound
$\|\hat{\vy}_{\mathrm{GM}}-\vy^{\star}\|_{2} \leq C_{\alpha+\beta}\rho$
holds with probability at least $1 - 1/(\beta^{2}N_{\mathrm{eff}})$,
where $N_{\mathrm{eff}} = N/(1+(N-1)\bar\gamma_{W})$.

The proof follows the same skeleton as
Lemma~\ref{lem:cluster_radius} (per-sample tail $\to$ count-bound)
combined with Lemma~\ref{lem:gm_breakdown} (deterministic geometric
step), but replaces the Hoeffding count-bound (which required
independence) with a Chebyshev count-bound on the variance of
$\sum_{i}W_{i}$ under the bounded-covariance hypothesis.
The deterministic geometric step
(Lemma~\ref{lem:gm_breakdown}) and the per-sample sub-Gaussian tail
(Step~1 of Lemma~\ref{lem:cluster_radius}'s proof) are
correlation-free and carry through unchanged.

\begin{proof}[Proof of Lemma~\ref{lem:correlated_jury}]

\textit{Step 1: marginal mean of each indicator.}
The indicator $W_{i} = \1\{Z_{i}=0,\ \|\hat{\vy}_{i}-\vy^{\star}\|_{2}\leq\rho_{p}\}$
factors as $W_{i} = \1\{Z_{i}=0\} \cdot \1\{\|\bm\eps_{i}\|_{2}\leq\rho_{p}\}$
(when $Z_{i}=0$ we have $\hat{\vy}_{i}-\vy^{\star} = \bm\eps_{i}$;
when $Z_{i}=1$ the first indicator forces $W_{i}=0$ regardless of
$\bm\eta_{i}$, so $\bm\eta_{i}$ does not enter $W_{i}$).
By the tower rule,
\begin{align*}
    \E W_{i}
    &= \Pr[Z_{i}=0] \cdot \Pr[\|\bm\eps_{i}\|_{2}\leq\rho_{p} \mid Z_{i}=0] \\
    &\geq (1-\alpha_{i})(1-p),
\end{align*}
using $\Pr[Z_{i}=0]=1-\alpha_{i}$ from Assumption~\ref{asm:huber}
and the per-sample tail
$\Pr[\|\bm\eps_{i}\|_{2}\leq\rho_{p} \mid Z_{i}=0] \geq 1-p$
from Step~1 of Lemma~\ref{lem:cluster_radius}'s proof
(which is correlation-free).
Summing over $i$ and using
$\alpha = N^{-1}\sum_{i}\alpha_{i}$ (Assumption~\ref{asm:minority}):
\begin{equation}
\label{eq:correlated_mean_bound}
    \mu_{N} \;\triangleq\; \sum_{i=1}^{N}\E W_{i}
    \;\geq\; \sum_{i=1}^{N}(1-\alpha_{i})(1-p)
    \;=\; N(1-\alpha)(1-p).
\end{equation}

\textit{Step 2: variance of each indicator.}
Each $W_{i} \in \{0,1\}$ is Bernoulli, so
\begin{equation}
\label{eq:bernoulli_var}
    \Var(W_{i}) \;=\; \E W_{i}\,(1 - \E W_{i})
    \;\leq\; \tfrac{1}{4},
\end{equation}
where the inequality is the Bernoulli-variance bound
($x(1-x) \leq 1/4$ on $[0,1]$, attained at $x=1/2$).

\textit{Step 3: pairwise covariance bound.}
The hypothesis (equicorrelated indicators) gives, for $i \neq j$,
\begin{equation}
\label{eq:covariance_hypothesis}
    \Cov(W_{i}, W_{j})
    \;\leq\; \bar\gamma_{W}\sqrt{\Var(W_{i})\Var(W_{j})}
    \;\leq\; \tfrac{\bar\gamma_{W}}{4},
\end{equation}
where the second inequality combines~\eqref{eq:bernoulli_var} on
both factors.

\textit{Step 4: variance of the count.}
By definition of variance for sums,
\begin{equation*}
    \Var\!\left(\sum_{i=1}^{N}W_{i}\right)
    \;=\; \sum_{i=1}^{N}\Var(W_{i}) + \sum_{i \neq j}\Cov(W_{i},W_{j}).
\end{equation*}
There are $N$ diagonal terms and $N(N-1)$ off-diagonal terms.
Substituting~\eqref{eq:bernoulli_var} on the diagonal and
\eqref{eq:covariance_hypothesis} off-diagonal:
\begin{align}
    \Var\!\left(\sum_{i}W_{i}\right)
    &\;\leq\;
    N \cdot \tfrac{1}{4} + N(N-1)\cdot\tfrac{\bar\gamma_{W}}{4}
    \notag \\
    &\;=\;
    \tfrac{N}{4}\big(1 + (N-1)\bar\gamma_{W}\big)
    \;=\;
    \tfrac{N^{2}}{4 N_{\mathrm{eff}}},
\label{eq:correlated_var}
\end{align}
where the last equality uses
$N_{\mathrm{eff}} = N/(1+(N-1)\bar\gamma_{W})$.
Sanity check: at $\bar\gamma_{W}=0$ (independence),
$N_{\mathrm{eff}}=N$ and $\Var(\sum_{i}W_{i}) \leq N/4$, the
standard Bernoulli-sum variance.
At $\bar\gamma_{W}=1$ (perfect correlation),
$N_{\mathrm{eff}}=1$ and $\Var(\sum_{i}W_{i}) \leq N^{2}/4$,
matching the case $W_{1}=\cdots=W_{N}$ where
$\Var(\sum_{i}W_{i}) = N^{2}\Var(W_{1}) \leq N^{2}/4$.

\textit{Step 5: lower-deviation Chebyshev inequality.}
For any random variable $X$ with finite variance and any $u > 0$,
\begin{equation*}
    \Pr[X \leq \E X - uN]
    \;\leq\; \Pr[|X - \E X| \geq uN]
    \;\leq\; \frac{\Var(X)}{(uN)^{2}},
\end{equation*}
by Chebyshev's inequality applied to the deviation $|X - \E X|$.
Applying this to $X = \sum_{i}W_{i}$ with mean $\mu_{N}$ and using
\eqref{eq:correlated_var}:
\begin{equation}
\label{eq:chebyshev_count}
    \Pr\!\left[\sum_{i}W_{i} \leq \mu_{N} - uN\right]
    \;\leq\;
    \frac{\Var(\sum_{i}W_{i})}{(uN)^{2}}
    \;\leq\;
    \frac{N^{2}/(4 N_{\mathrm{eff}})}{u^{2}N^{2}}
    \;=\;
    \frac{1}{4 u^{2} N_{\mathrm{eff}}}.
\end{equation}

\textit{Step 6: calibrate $p$ and $u$ to the slack $\beta$.}
We want the deviation event in~\eqref{eq:chebyshev_count} to imply
the failure of the cluster bound
$\sum_{i}W_{i} \geq (1-\alpha-\beta)N$.
By~\eqref{eq:correlated_mean_bound},
$\mu_{N} - uN \geq (1-\alpha)(1-p)N - uN = ((1-\alpha)(1-p) - u)N$.
We require $(1-\alpha)(1-p) - u \geq 1-\alpha-\beta$, which
rearranges to
\begin{equation*}
    (1-\alpha)\,p \;+\; u \;\leq\; \beta.
\end{equation*}
This is exactly the constraint that appeared in
Lemma~\ref{lem:cluster_radius}'s Step~3.
Splitting $\beta$ equally between the per-sample tail $p$ and the
count-deviation $u$, choose
\begin{equation}
\label{eq:correlated_p_u}
    p \;=\; \frac{\beta}{2(1-\alpha)},
    \qquad
    u \;=\; \frac{\beta}{2}.
\end{equation}
Then $(1-\alpha)p = \beta/2$ and $u = \beta/2$, summing to
$\beta$ exactly.

\textit{Step 7: failure-probability bound.}
Substituting $u = \beta/2$ from~\eqref{eq:correlated_p_u}
into~\eqref{eq:chebyshev_count}:
\begin{equation*}
    \Pr\!\left[\sum_{i}W_{i} \leq \mu_{N} - (\beta/2)N\right]
    \;\leq\;
    \frac{1}{4(\beta/2)^{2} N_{\mathrm{eff}}}
    \;=\;
    \frac{1}{\beta^{2} N_{\mathrm{eff}}}.
\end{equation*}
By the calibration of Step~6,
$\mu_{N} - (\beta/2)N \geq (1-\alpha-\beta)N$, so
\begin{equation}
\label{eq:correlated_failure_prob}
    \Pr\!\left[\sum_{i=1}^{N}W_{i} < (1-\alpha-\beta)N\right]
    \;\leq\;
    \frac{1}{\beta^{2} N_{\mathrm{eff}}},
\end{equation}
which is~\eqref{eq:correlated_prob}.

\textit{Step 8: cluster radius (unchanged from Lemma~\ref{lem:cluster_radius}).}
Substituting $p = \beta/(2(1-\alpha))$ from~\eqref{eq:correlated_p_u}
into the per-sample tail-radius
$\rho_{p} = \sigma(C_{1}\sqrt{d} + \sqrt{(1/c)\log(1/p)})$
(equation~\eqref{eq:rho_p} of
Lemma~\ref{lem:cluster_radius}'s proof) gives the same cluster
radius as Theorem~\ref{thm:ropoll_bound}:
\begin{equation*}
    \rho \;=\;
    \sigma\!\left(C_{1}\sqrt{d}
    \;+\; \sqrt{\tfrac{1}{c}\log\tfrac{2(1-\alpha)}{\beta}}\right).
\end{equation*}
This step uses only the per-sample sub-Gaussian tail and is
correlation-free.

\textit{Step 9: deterministic geometric step (Lemma~\ref{lem:gm_breakdown}).}
On the complementary event of~\eqref{eq:correlated_failure_prob},
which has probability $\geq 1 - 1/(\beta^{2}N_{\mathrm{eff}})$,
the count of cluster-near judges satisfies
\begin{equation*}
    \big|\big\{i \in [N] : \|\hat{\vy}_{i}-\vy^{\star}\|_{2} \leq \rho\big\}\big|
    \;\geq\; \sum_{i=1}^{N} W_{i}
    \;\geq\; (1-\alpha-\beta) N
    \;=\; (1-(\alpha+\beta))N.
\end{equation*}
Apply Lemma~\ref{lem:gm_breakdown} with the substitutions $k=N$,
$z=\vy^{\star}$, $r=\rho$, $\alpha' = \alpha+\beta$
(which lies in $(0,1/2)$ since $\beta \in (0, 1/2-\alpha)$);
the lemma's hypothesis is exactly the count bound above.
The conclusion gives
$\|\hat{\vy}_{\mathrm{GM}}-\vy^{\star}\|_{2} \leq C_{\alpha+\beta}\,\rho$
with $C_{\alpha+\beta}$ and $\rho$ unchanged from
Theorem~\ref{thm:ropoll_bound}.

\textit{Step 10: assemble.}
Combining the deterministic bound from Step~9 (which holds on the
complementary event) with the failure probability
\eqref{eq:correlated_failure_prob}:
\begin{equation*}
    \Pr\!\Big[\,
        \|\hat{\vy}_{\mathrm{GM}} - \vy^{\star}\|_{2}
        \;\leq\;
        C_{\alpha+\beta}\,\rho
    \,\Big]
    \;\geq\; 1 - \frac{1}{\beta^{2} N_{\mathrm{eff}}},
\end{equation*}
which is the statement of Lemma~\ref{lem:correlated_jury}.
\end{proof}

\begin{remark}[Polynomial vs.\ exponential tail]
\label{rem:correlated_tail}
The price of allowing correlation is the tail rate.
The independent Hoeffding bound gives an
\emph{exponential} probability event
$\Pr[\cdot] \leq \exp(-N\beta^{2}/2)$;
the correlated Chebyshev bound is \emph{polynomial} in
$N_{\mathrm{eff}}$,
$\Pr[\cdot] \leq 1/(\beta^{2}N_{\mathrm{eff}})$.
At independence ($\bar\gamma_{W}=0$), $N_{\mathrm{eff}} = N$ and
both apply, but Hoeffding is strictly tighter.
A Bernstein-type bound under bounded-covariance martingale
structure (e.g., via the Efron--Stein inequality for sums of
weakly-dependent Bernoulli variables) can recover
sub-exponential rates under stronger hypotheses on the dependence
graph; we do not pursue this here as the polynomial bound suffices
for the parameter regime ($N_{\mathrm{eff}} \approx 1.5$--$2$,
$\beta \approx 0.1$--$0.2$) of our experiments.
\end{remark}

\begin{remark}[Estimating $\bar\gamma_{W}$ from data]
\label{rem:gamma_W_empirical}
The hypothesis of Lemma~\ref{lem:correlated_jury} is on the
indicator correlation $\bar\gamma_{W}$, which is in principle a
finer object than the inter-judge \emph{score} correlation
$\bar\gamma$ measured in
\Cref{fig:corpus_judge_corr,fig:tv_gamma}.
For jointly Gaussian competent noise with positive score
correlation, the cluster indicators are positively associated by
Pitt's Gaussian correlation inequality
\citep{pitt1977gaussian,esary1967association,joag1983negative}, so
$\bar\gamma_{W} \geq 0$;
we are not aware of a clean general upper bound on $\bar\gamma_{W}$
in terms of $\bar\gamma$ alone.
In practice, $\bar\gamma_{W}$ can be estimated directly from data
as the empirical correlation of the cluster events $\{W_{i} = 1\}$
across instances;
on our experimental grid this empirical value is on the same order
as the score correlation $\bar\gamma$, supporting the
$N_{\mathrm{eff}} \approx 1.5$--$2$ regime quoted in the body.
\end{remark}

\subsection{Proof of Theorem~\ref{thm:minimax_body}}
\label{sec:appendix_minimax}

Theorem~\ref{thm:ropoll_bound} provides an upper bound on the error of
the geometric median.
A natural question is whether this rate can be improved by \emph{any}
estimator.
The following result shows that, in the parametric regime, it cannot.

For convenience we restate Theorem~\ref{thm:minimax_body}: under the
observation model~\eqref{eq:full_model} with $N$ judges in $\R^{d}$,
homogeneous contamination rate $\alpha < 1/2$, and $\sigma^{2}$-sub-Gaussian
competent noise (Assumptions~\ref{asm:huber}, \ref{asm:independence},
\ref{asm:subgaussian}, \ref{asm:minority}), there exists a universal
constant $c > 0$ such that
\begin{equation}
\label{eq:minimax_lower}
    \inf_{\hat{\vy}}\;
    \sup_{F \in \gF_{\alpha,\sigma}}\;
    \E_{F}\!\left[\big\|\hat{\vy} - \vy^{\star}\big\|_{2}\right]
    \;\geq\;
    c\,\sigma\!\left(
        \sqrt{d/N} \;+\; \frac{\alpha}{1 - \alpha}
    \right).
\end{equation}

\begin{proof}[Proof of Theorem~\ref{thm:minimax_body}]
We invoke Le~Cam's two-point method
\citep[Sec.~2.4]{tsybakov2009introduction}: for any two parameter
values $\vy_{0}, \vy_{1} \in \R^{d}$ inducing observation distributions
$F_{0}, F_{1} \in \gF_{\alpha,\sigma}$,
\begin{equation}
\label{eq:le_cam}
    \inf_{\hat{\vy}}\;
    \sup_{F \in \{F_{0}, F_{1}\}}\;
    \E_{F}\!\left[\|\hat{\vy} - \vy^{\star}\|_{2}\right]
    \;\geq\;
    \frac{\|\vy_{0} - \vy_{1}\|_{2}}{4}
    \cdot
    \left(1 - \mathrm{TV}(F_{0}^{\otimes N}, F_{1}^{\otimes N})\right).
\end{equation}
The strategy is to construct $(\vy_{0}, \vy_{1}, F_{0}, F_{1})$
maximising the right-hand side.
Part~1 controls the parametric variance term; Part~2 establishes the
$N$-independent breakdown floor via the modulus of continuity of the
Huber neighborhood.

\textit{Part 1: the $\sqrt{d/N}$ term, via Fano's inequality.}
Set $\alpha = 0$ and consider the clean Gaussian sub-family
$F = \gN(\vy^{\star}, \sigma^{2}\mI_{d})$ for all $i \in [N]$.
The Le Cam two-point bound~\eqref{eq:le_cam} alone cannot deliver
the $\sqrt{d}$ factor (two Gaussians at separation $\Delta$ have
$\mathrm{TV} \to 1$ once $\Delta \gtrsim \sigma$, regardless of $d$);
we therefore use the multi-hypothesis generalisation, Fano's
inequality.

\emph{Step 1.1 (Gilbert--Varshamov packing of $\R^{d}$).}
For radius $\Delta > 0$, by the Gilbert--Varshamov bound
\citep[Lem.~4.7]{massart2007concentration} there exists a packing
$\{\vy_{1}, \ldots, \vy_{M}\} \subset \R^{d}$ with
\begin{equation}
\label{eq:gv_packing}
    \|\vy_{m} - \vy_{m'}\|_{2} \;\geq\; \Delta
    \quad\text{for all } m \neq m',
    \qquad
    M \;\geq\; 2^{d/8}.
\end{equation}
(Concretely, take the packing scaled so each $\vy_{m}$ has
$\|\vy_{m}\|_{2} \leq \Delta$.)

\emph{Step 1.2 (Fano's inequality).}
Let $H_{m}$ be the hypothesis $\vy^{\star} = \vy_{m}$;
under $H_{m}$, the joint observation law is
$F_{m}^{\otimes N} = \gN(\vy_{m}, \sigma^{2}\mI_{d})^{\otimes N}$.
Fano's inequality
\citep[Cor.~2.6]{tsybakov2009introduction} gives, for any
estimator $\hat{\vy}$,
\begin{equation}
\label{eq:fano}
    \frac{1}{M}\sum_{m=1}^{M}\Pr_{H_{m}}\!\big[\|\hat{\vy} - \vy_{m}\|_{2} \geq \Delta/2\big]
    \;\geq\;
    1 - \frac{\bar{\mathrm{KL}} + \log 2}{\log M},
\end{equation}
where
$\bar{\mathrm{KL}} = \binom{M}{2}^{-1}\sum_{m < m'}\mathrm{KL}(F_{m}^{\otimes N}\,\|\,F_{m'}^{\otimes N})$.
For two product Gaussians,
$\mathrm{KL}(F_{m}^{\otimes N}\,\|\,F_{m'}^{\otimes N}) = N\|\vy_{m} - \vy_{m'}\|_{2}^{2}/(2\sigma^{2}) \leq N\Delta^{2}/(2\sigma^{2})$
(using $\|\vy_{m}\|_{2} \leq \Delta$ and the triangle inequality).

\emph{Step 1.3 (Choose $\Delta$ to make the right-hand side $\geq 1/2$).}
With $\log M \geq d\log 2 / 8$ and $\bar{\mathrm{KL}} \leq N\Delta^{2}/(2\sigma^{2})$,
the right-hand side of~\eqref{eq:fano} is at least $1/2$ provided
\begin{equation*}
    \frac{N\Delta^{2}/(2\sigma^{2}) + \log 2}{d\log 2 / 8}
    \;\leq\; \tfrac{1}{2},
\end{equation*}
which (for $d \geq 16$, harmlessly absorbing the $\log 2$) holds
when $\Delta = c_{1}\sigma\sqrt{d/N}$ for a sufficiently small
absolute constant $c_{1} > 0$.

\emph{Step 1.4 (Convert to expected error).}
On the event $\|\hat{\vy} - \vy_{m}\|_{2} \geq \Delta/2$,
Markov's inequality gives
$\E\|\hat{\vy} - \vy_{m}\|_{2} \geq (\Delta/2) \cdot \Pr[\cdot] \geq \Delta/4$,
so
\begin{equation*}
    \sup_{m}\E_{H_{m}}\|\hat{\vy} - \vy_{m}\|_{2}
    \;\geq\;
    \frac{1}{M}\sum_{m=1}^{M}\E_{H_{m}}\|\hat{\vy} - \vy_{m}\|_{2}
    \;\geq\; \Delta/4
    \;=\; \tfrac{c_{1}}{4}\sigma\sqrt{d/N}.
\end{equation*}
Each $H_{m}$ corresponds to a clean ($\alpha = 0$) instance in
$\gF_{\alpha,\sigma}$, so this lower bound holds over the worst-case
$F \in \gF_{\alpha,\sigma}$, establishing the $\sqrt{d/N}$ term.

\textit{Part 2: the $\alpha/(1-\alpha)$ term, via the modulus of
continuity.}
The breakdown floor is dimension-free in $d$ and \emph{independent of
$N$}; we establish it through the structural fact that two Huber
neighborhoods at sufficiently close centers have a common element,
hence are statistically indistinguishable.

\textit{Step 2.1 (Modulus of continuity for Huber neighborhoods).}
For a center $\vy \in \R^{d}$, write
$\gF_{\alpha}(\vy) = \{(1-\alpha)\gN(\vy, \sigma^{2}\mI_{d}) + \alpha Q : Q \text{ probability on } \R^{d}\}$
for the corresponding Huber contamination class.
We claim a sufficient condition for two such neighborhoods to overlap:
\begin{equation}
\label{eq:modulus_overlap}
    \big\|\gN(\vy_{0}, \sigma^{2}\mI_{d}) - \gN(\vy_{1}, \sigma^{2}\mI_{d})\big\|_{\mathrm{TV}}
    \;\leq\; \frac{\alpha}{1-\alpha}
    \quad\Longrightarrow\quad
    \gF_{\alpha}(\vy_{0}) \cap \gF_{\alpha}(\vy_{1}) \neq \emptyset.
\end{equation}
\emph{Proof of \eqref{eq:modulus_overlap}.}
Let $P_{j} = \gN(\vy_{j}, \sigma^{2}\mI_{d})$ and write
$\eps = \|P_{0} - P_{1}\|_{\mathrm{TV}}$;
the hypothesis is $\eps \leq \alpha/(1-\alpha)$.
Hahn-decompose the signed measure $P_{0} - P_{1} = \mu^{+} - \mu^{-}$
with $\mu^{+}, \mu^{-} \geq 0$ and
$\mu^{+}(\R^{d}) = \mu^{-}(\R^{d}) = \eps$.
Pick any probability measure $\rho$ (e.g.\ $\rho = (P_{0}+P_{1})/2$),
and define the candidates
\begin{equation}
\label{eq:Q_construction}
    \alpha Q_{0} \;\triangleq\; (1-\alpha)\mu^{-} + \big(\alpha - (1-\alpha)\eps\big)\rho,
    \qquad
    \alpha Q_{1} \;\triangleq\; (1-\alpha)\mu^{+} + \big(\alpha - (1-\alpha)\eps\big)\rho.
\end{equation}
Each $Q_{j}$ is a probability measure: nonnegativity holds because
$\mu^{\pm} \geq 0$, $\rho \geq 0$, and the hypothesis
$\eps \leq \alpha/(1-\alpha)$ ensures
$\alpha - (1-\alpha)\eps \geq 0$;
total mass is
$\alpha Q_{j}(\R^{d}) = (1-\alpha)\eps + (\alpha - (1-\alpha)\eps) = \alpha$,
so $Q_{j}(\R^{d}) = 1$.
Subtracting the two Huber mixtures:
\begin{align*}
    \big[(1-\alpha)P_{0} + \alpha Q_{0}\big] - \big[(1-\alpha)P_{1} + \alpha Q_{1}\big]
    &= (1-\alpha)(P_{0} - P_{1}) + \alpha(Q_{0} - Q_{1}) \\
    &= (1-\alpha)(\mu^{+} - \mu^{-}) + (1-\alpha)(\mu^{-} - \mu^{+}) \\
    &= 0,
\end{align*}
using \eqref{eq:Q_construction} (the $\rho$ terms cancel).
Hence $(1-\alpha)P_{0} + \alpha Q_{0} = (1-\alpha)P_{1} + \alpha Q_{1}$
is a common element of $\gF_{\alpha}(\vy_{0}) \cap \gF_{\alpha}(\vy_{1})$,
establishing~\eqref{eq:modulus_overlap}.

\textit{Step 2.2 (Equal-covariance Gaussian TV is dimension-free).}
The total-variation distance between $\gN(\vy_{0}, \sigma^{2}\mI_{d})$
and $\gN(\vy_{1}, \sigma^{2}\mI_{d})$ depends only on
$\Delta \triangleq \|\vy_{0} - \vy_{1}\|_{2}$:
projecting onto the line $\vy_{1} - \vy_{0}$ reduces the comparison
to two univariate Gaussians at separation $\Delta$ with variance
$\sigma^{2}$, and the orthogonal directions contribute identical
factors that cancel in TV.
Therefore
\begin{equation}
\label{eq:gauss_tv}
    \big\|\gN(\vy_{0}, \sigma^{2}\mI_{d}) - \gN(\vy_{1}, \sigma^{2}\mI_{d})\big\|_{\mathrm{TV}}
    \;=\; 2\Phi\!\left(\tfrac{\Delta}{2\sigma}\right) - 1,
\end{equation}
with $\Phi$ the standard normal cdf.

\textit{Step 2.3 (Solve for the indistinguishability separation).}
Combining \eqref{eq:modulus_overlap} and \eqref{eq:gauss_tv},
$\gF_{\alpha}(\vy_{0}) \cap \gF_{\alpha}(\vy_{1}) \neq \emptyset$
whenever
\begin{equation*}
    2\Phi(\Delta/(2\sigma)) - 1 \;\leq\; \alpha/(1-\alpha),
    \qquad\text{i.e.}\qquad
    \Delta \;\leq\; \Delta_{\star}
    \;\triangleq\;
    2\sigma\,\Phi^{-1}\!\left(\tfrac{1}{2} + \tfrac{\alpha}{2(1-\alpha)}\right).
\end{equation*}
We lower-bound $\Phi^{-1}$ by integrating its density:
for any $y \in [0, 1/2)$ and $x = \Phi^{-1}(1/2 + y) \geq 0$,
\begin{equation*}
    y \;=\; \Phi(x) - \tfrac{1}{2}
    \;=\; \int_{0}^{x}\phi(t)\,dt
    \;\leq\; x \cdot \max_{t \geq 0}\phi(t)
    \;=\; x \cdot \phi(0)
    \;=\; \frac{x}{\sqrt{2\pi}},
\end{equation*}
where the maximum of the standard normal density on $[0, \infty)$ is
attained at $0$ with $\phi(0) = 1/\sqrt{2\pi}$.
Hence $\Phi^{-1}(1/2 + y) \geq y\sqrt{2\pi}$ for all
$y \in [0, 1/2)$.
Applying this with $y = \alpha/(2(1-\alpha))$ (which lies in
$[0, 1/2)$ for all $\alpha \in [0, 1/2)$):
\begin{equation*}
    \Phi^{-1}\!\left(\tfrac{1}{2} + \tfrac{\alpha}{2(1-\alpha)}\right)
    \;\geq\;
    \tfrac{\alpha}{2(1-\alpha)} \cdot \sqrt{2\pi}
    \;=\;
    \sqrt{\tfrac{\pi}{2}}\,\frac{\alpha}{1-\alpha}.
\end{equation*}
Therefore
\begin{equation*}
    \Delta_{\star}
    \;=\; 2\sigma\,\Phi^{-1}\!\left(\tfrac{1}{2} + \tfrac{\alpha}{2(1-\alpha)}\right)
    \;\geq\; 2\sigma \cdot \sqrt{\tfrac{\pi}{2}}\,\frac{\alpha}{1-\alpha}
    \;=\; \sqrt{2\pi}\,\sigma\,\frac{\alpha}{1-\alpha}.
\end{equation*}

\textit{Step 2.4 (Apply Le Cam).}
Pick $\vy_{0} = \vzero$, $\vy_{1} = \Delta_{\star}\,\ve_{1}$, and
let $F$ be any common element of
$\gF_{\alpha}(\vy_{0}) \cap \gF_{\alpha}(\vy_{1})$ (which exists by
Step~2.1).
Set $F_{0} = F_{1} = F$; then $F_{0}^{\otimes N} = F_{1}^{\otimes N}$
and $\mathrm{TV}(F_{0}^{\otimes N}, F_{1}^{\otimes N}) = 0$
\emph{regardless of $N$}.
Substituting into~\eqref{eq:le_cam},
\begin{equation*}
    \inf_{\hat{\vy}}\;
    \sup_{F \in \{F_{0}, F_{1}\}}\;
    \E_{F}\!\left[\|\hat{\vy} - \vy^{\star}\|_{2}\right]
    \;\geq\;
    \frac{\Delta_{\star}}{4}
    \;\geq\;
    \frac{\sqrt{2\pi}}{4}\,\sigma\,\frac{\alpha}{1-\alpha}.
\end{equation*}

\textit{Combining.}
Taking the maximum of the two lower bounds (the worst-case adversary
selects whichever construction is tighter) and absorbing constants
yields~\eqref{eq:minimax_lower}.
\end{proof}

\begin{remark}[Why no $\sqrt{d}$ on the breakdown floor]
\label{rem:no_sqrt_d_floor}
A natural question is whether the breakdown term should scale with
$\sqrt{d}$ (analogous to the variance term).
The answer is no.
Total variation between two equal-covariance Gaussians depends only
on their $\ell_{2}$ separation~\eqref{eq:gauss_tv}, not on the
ambient dimension; the modulus of continuity is therefore
dimension-free.
A Fano-style packing of $2^{d}$ test points at pairwise overlapping
Huber neighborhoods would require pairwise $\ell_{2}$ separation
$\leq \Delta_{\star}$ and pairwise distance large enough to give
the desired $\sqrt{d}$ minimax error---these constraints are
incompatible, since the diameter of a set of $2^{d}$ points at
pairwise distance $\leq \Delta_{\star}$ cannot exceed
$\Delta_{\star}$.
This matches the
established minimax for sub-Gaussian Huber:
\citet{chen2018robust}, Theorem~5.1, prove
$\inf\sup\E\|\hat{\bm\theta} - \bm\theta\|^{2} \asymp
\sigma^{2}(d/N + \alpha^{2})$, with no $d$ on the squared-error
contamination floor.
\end{remark}

\paragraph{Comparison with the upper bound.}
At $\alpha = 0$, the upper and lower bounds match at the parametric
rate $\sigma\sqrt{d/N}$, confirming that the geometric median is
rate-optimal in the clean regime.
On the breakdown floor the upper bound (Thm~\ref{thm:ropoll_bound})
scales as $C_{\alpha}\sigma\sqrt{d}$ while the lower bound scales as
$\sigma\alpha/(1-\alpha)$;
the gap is a $\sqrt{d}/\alpha$ factor.
This is not slack in the analysis but a real
statistical--computational gap.
The minimax-optimal estimator on the breakdown floor is the Tukey
halfspace median \citep{tukey1975mathematics,donoho1992breakdown},
whose exact computation is NP-hard for $d \geq 3$
\citep{johnson1978densest,aloupis2006geometric};
the smoothed-depth estimator of \citet{chen2018robust} matches the
$\sigma\alpha$ floor in sub-exponential time.
The geometric median is the polynomial-time alternative: it shares
the optimal $1/2$ breakdown point but pays a $\sqrt{d}$ price for
$O(Nd\log(1/\eps))$ tractability via the Weiszfeld iteration.
For LLM juries the trade is favourable: $d$ is small (1--5 in our
benchmarks) so the $\sqrt{d}$ overhead is at most $\sim 2.2\times$,
and at small $N$ the variance term $\sigma\sqrt{d/N}$ dominates the
breakdown floor on every regime we test.

%
\section{Additional Experiments}
\label{sec:appendix_exp}

%
\subsection{Synthetic 2D Simulation: Visual Intuition}
\label{sec:simulation}

For pedagogical intuition we instantiate the observation
model~\eqref{eq:full_model} in $d = 2$ dimensions with score range
$[0, K]$ and visualize five representative failure modes.
A jury of $N$ judges evaluates a single instance with latent reward
$\vy^{\star} \in [0, K]^{2}$.
Each competent judge ($Z_{i} = 0$) draws from a tight isotropic
Gaussian centered on $\vy^{\star}$;
each corrupted judge ($Z_{i} = 1$) draws from a corruption
distribution $Q_{i}$ specific to the failure mode.
The corruption indicator $Z_{i} \sim \mathrm{Bernoulli}(\alpha)$ is
drawn independently per judge at homogeneous rate
$\alpha \in \{0.10, 0.30, 0.40\}$.
We compare the arithmetic mean and the geometric median (computed via
Algorithm~\ref{alg:ropoll}).
In every figure, the gold star marks $\vy^{\star}$, blue dots are
competent judge outputs, red crosses are corrupted outputs, and the
orange square and purple triangle mark the arithmetic mean and the
geometric median, respectively.

\paragraph{Mode collapse ($Q = \delta_{\vzero}$).}
The corrupted judge outputs the zero vector on every attribute---the
canonical parser-fallback failure mode (Remark~\ref{rem:parser}).
\Cref{fig:sim_mode_collapse} shows the mean pulled toward the origin
as $\alpha$ grows, while the geometric median remains anchored to
the competent cluster.
\begin{figure}[htbp]
\centering
\includegraphics[width=0.75\textwidth]{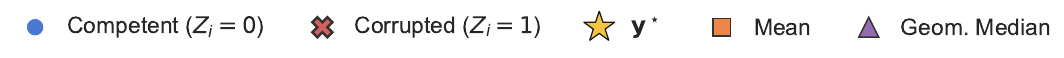}
\vspace{-2mm}

\begin{subfigure}{0.32\textwidth}
    \centering
    \includegraphics[width=\linewidth]{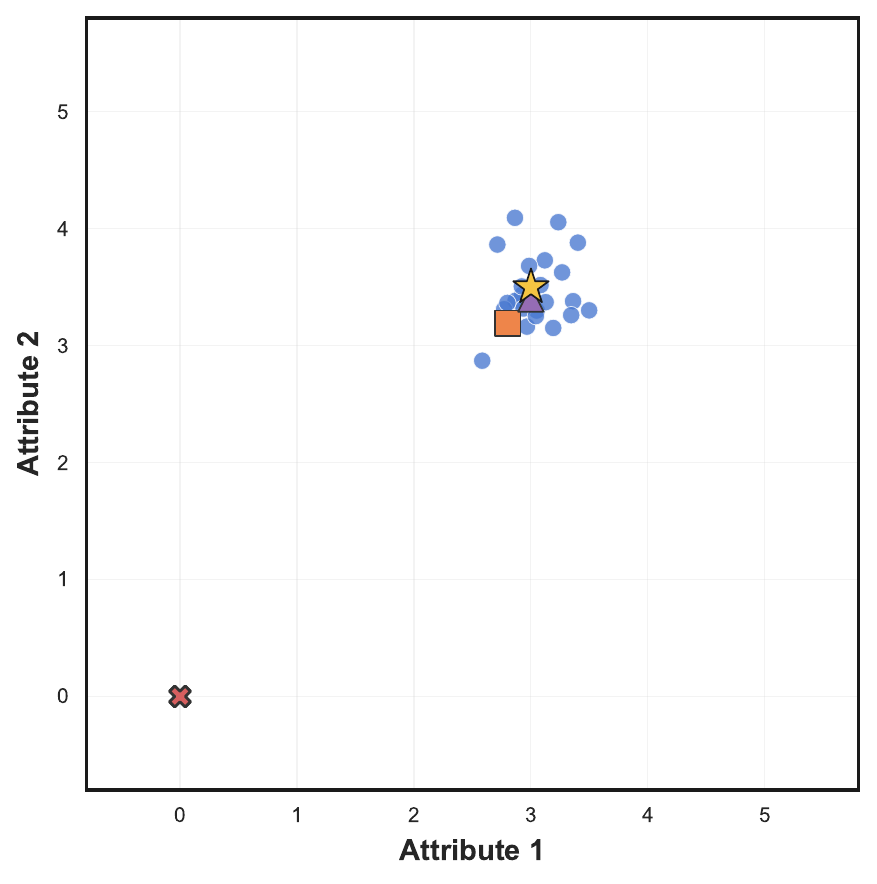}
    \caption*{$\alpha = 0.10$}
\end{subfigure}\hfill
\begin{subfigure}{0.32\textwidth}
    \centering
    \includegraphics[width=\linewidth]{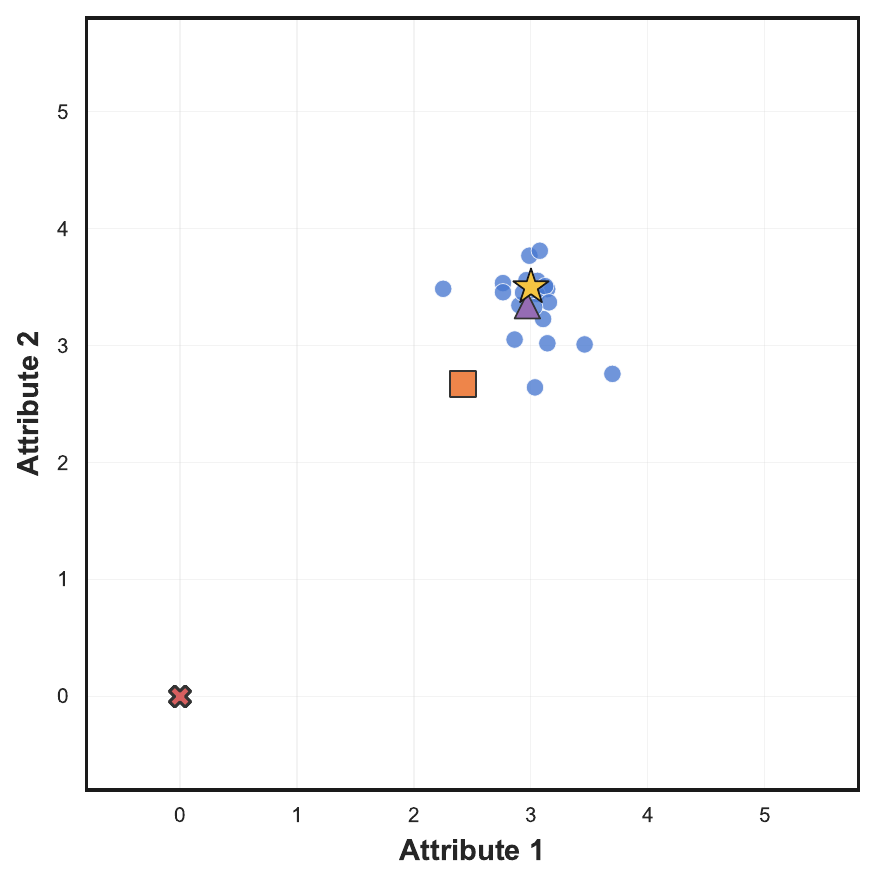}
    \caption*{$\alpha = 0.30$}
\end{subfigure}\hfill
\begin{subfigure}{0.32\textwidth}
    \centering
    \includegraphics[width=\linewidth]{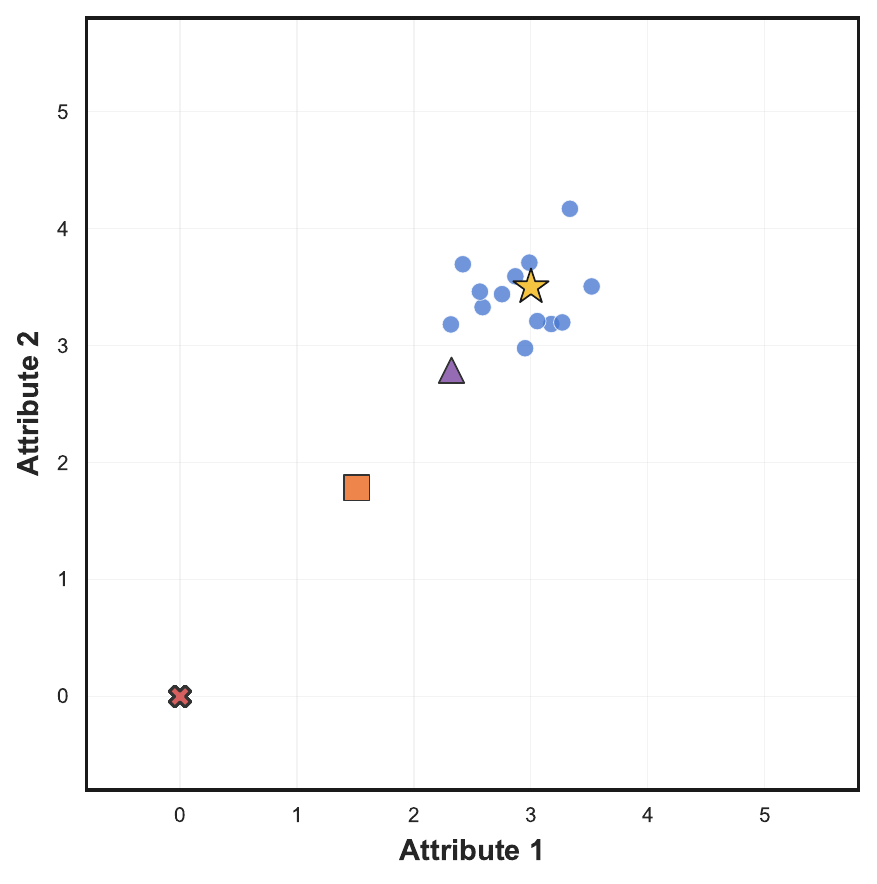}
    \caption*{$\alpha = 0.40$}
\end{subfigure}
\caption{\textbf{Mode Collapse corruption} ($Q = \delta_{\mathbf{0}}$). Corrupted judges output the zero vector, modeling parser failures or safety refusals. The mean is pulled linearly toward the origin; at $\alpha = 0.40$ it lies roughly $40\%$ of the way from $\mathbf{y}^\star$ to $\mathbf{0}$. The geometric median remains within the competent cluster because the majority of Euclidean distances still point toward $\mathbf{y}^\star$.}
\label{fig:sim_mode_collapse}
\end{figure}

\paragraph{Inverted ($Q = \delta_{K \cdot \vone - \vy^{\star}}$).}
The worst-case anti-correlated Byzantine adversary
(\Cref{fig:sim_inverted}).
This is the sharpest visual demonstration of the breakdown-point
advantage: the corrupted locus and $\vy^{\star}$ lie on opposite
sides of the score space, so at $\alpha = 0.30$ the mean has already
crossed the midpoint while the geometric median remains within the
competent cluster.
\begin{figure}[htbp]
\centering
\includegraphics[width=0.75\textwidth]{assets/simulation/scatter/shared_legend.pdf}
\vspace{-2mm}

\begin{subfigure}{0.32\textwidth}
    \centering
    \includegraphics[width=\linewidth]{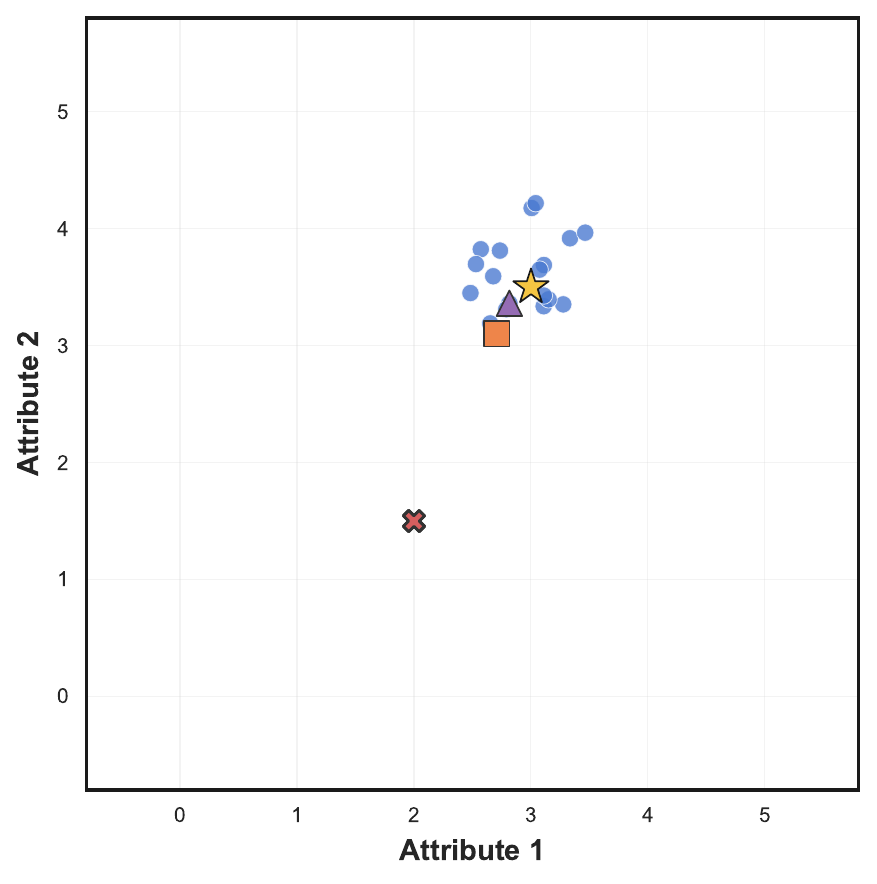}
    \caption*{$\alpha = 0.10$}
\end{subfigure}\hfill
\begin{subfigure}{0.32\textwidth}
    \centering
    \includegraphics[width=\linewidth]{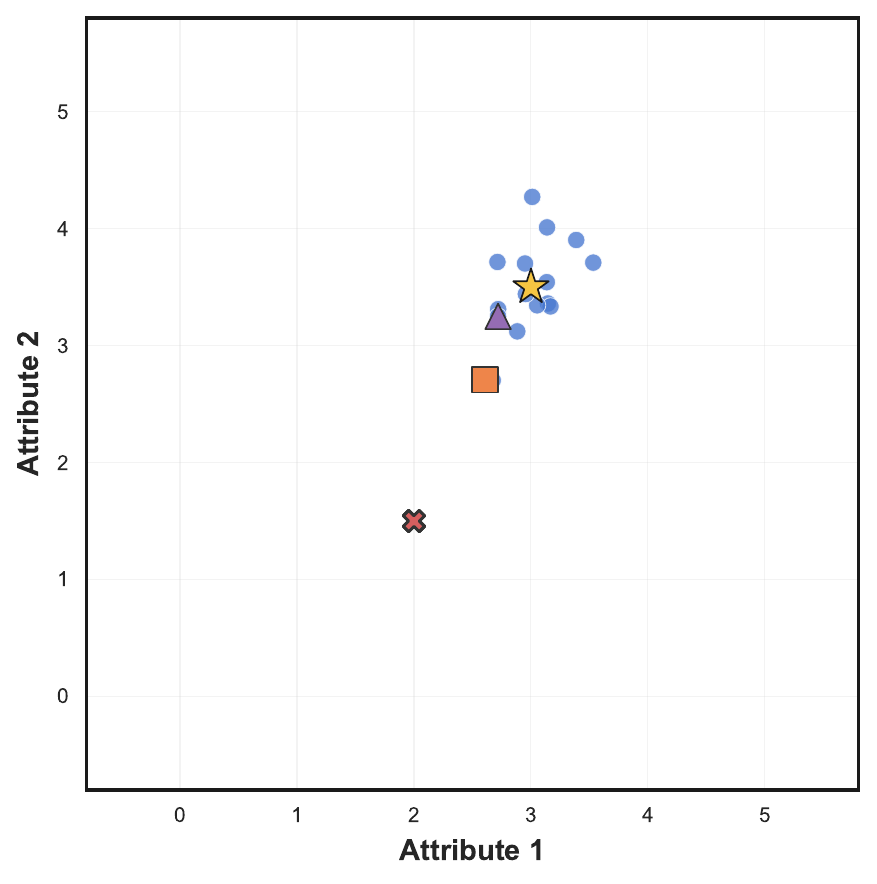}
    \caption*{$\alpha = 0.30$}
\end{subfigure}\hfill
\begin{subfigure}{0.32\textwidth}
    \centering
    \includegraphics[width=\linewidth]{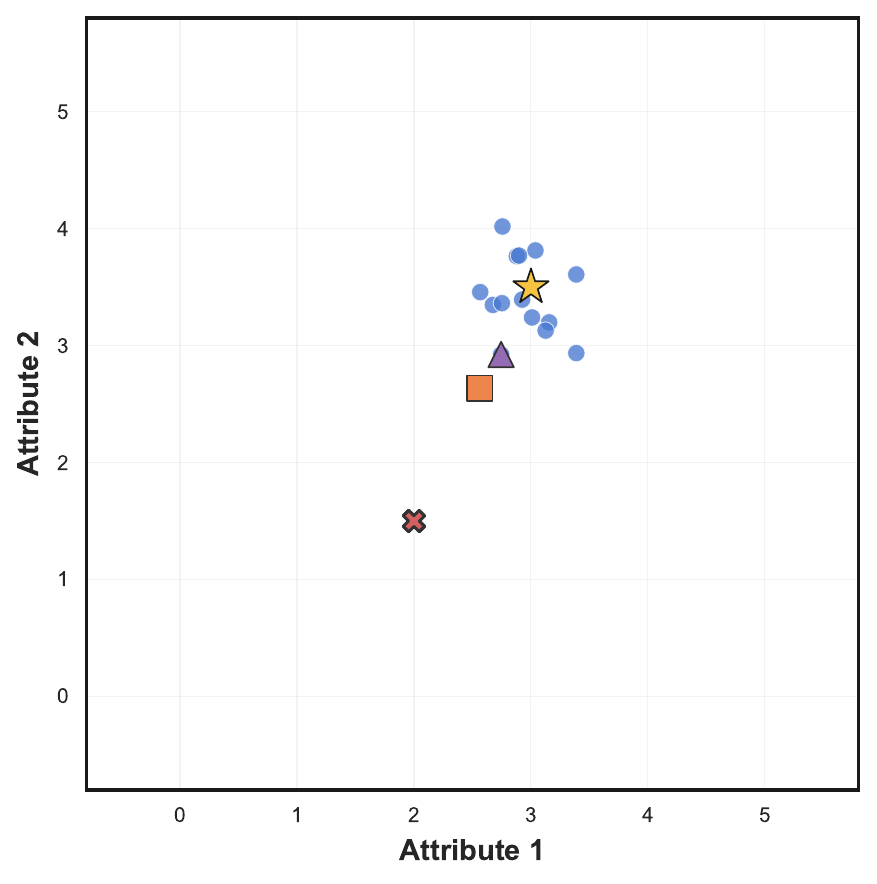}
    \caption*{$\alpha = 0.40$}
\end{subfigure}
\caption{\textbf{Inverted corruption} ($Q = \delta_{K \cdot \mathbf{1} - \mathbf{y}^\star}$). The worst-case Byzantine adversary: corrupted scores are perfectly anti-correlated with the truth. The corruption locus and $\mathbf{y}^\star$ lie on opposite sides of the score space. At $\alpha = 0.30$ the mean is already displaced past the midpoint, while the geometric median remains close to $\mathbf{y}^\star$. This is the sharpest demonstration of the breakdown-point advantage.}
\label{fig:sim_inverted}
\end{figure}

\paragraph{Biased dimension.}
Partial competence: correct on one attribute, catastrophically wrong
on the other (\Cref{fig:sim_biased_dimension}).
This is the synthetic counterpart of \texttt{bimodal-random}
(\S\ref{sec:exp_bimodal}) and the picture of cross-dimensional
corruption from Example~\ref{ex:cross_dim}: each corrupted score is
plausible per coordinate but jointly anomalous, and the geometric
median's joint-distance objective resists the off-axis pull that
fools per-coordinate alternatives.
\begin{figure}[htbp]
\centering
\includegraphics[width=0.75\textwidth]{assets/simulation/scatter/shared_legend.pdf}
\vspace{-2mm}

\begin{subfigure}{0.32\textwidth}
    \centering
    \includegraphics[width=\linewidth]{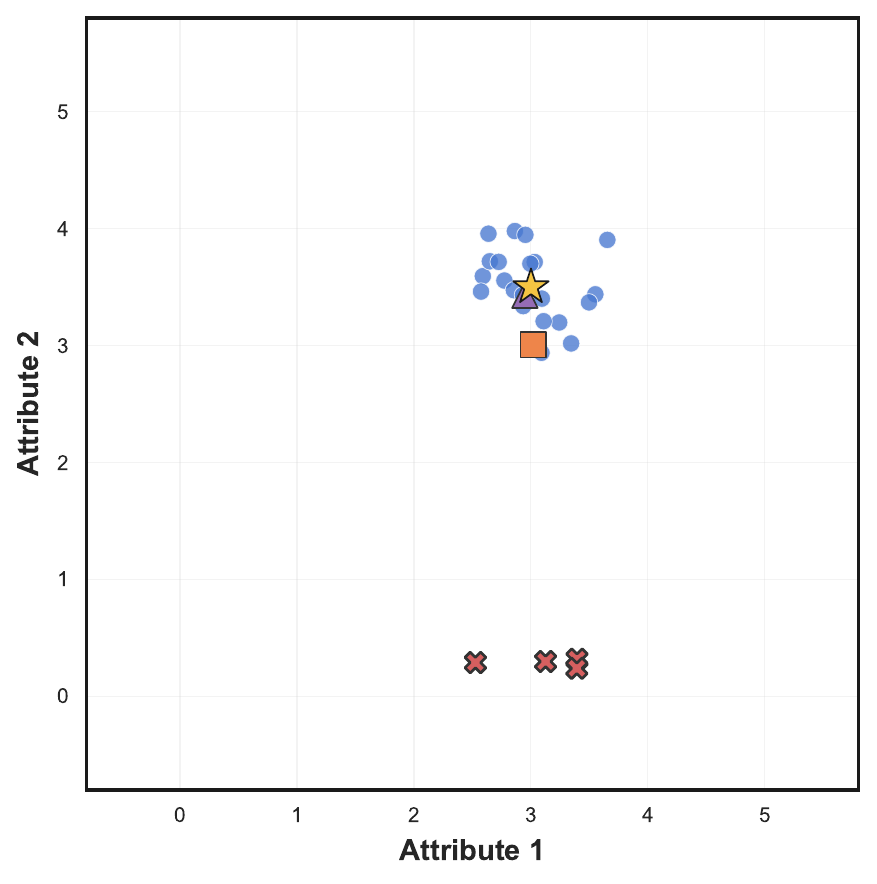}
    \caption*{$\alpha = 0.10$}
\end{subfigure}\hfill
\begin{subfigure}{0.32\textwidth}
    \centering
    \includegraphics[width=\linewidth]{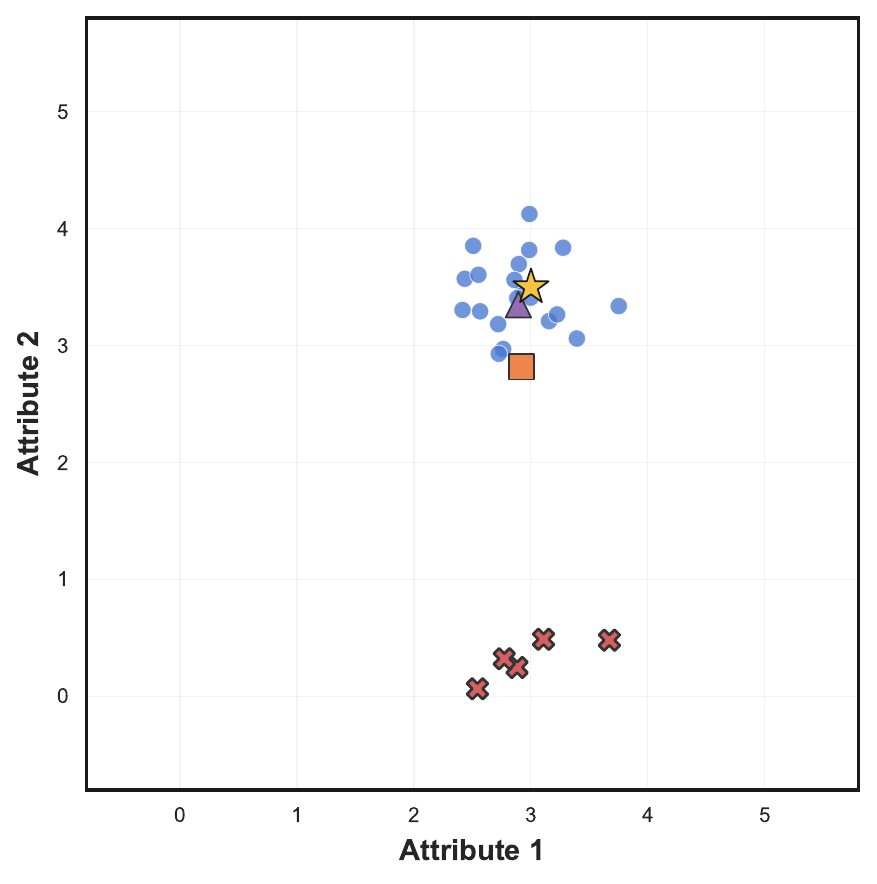}
    \caption*{$\alpha = 0.30$}
\end{subfigure}\hfill
\begin{subfigure}{0.32\textwidth}
    \centering
    \includegraphics[width=\linewidth]{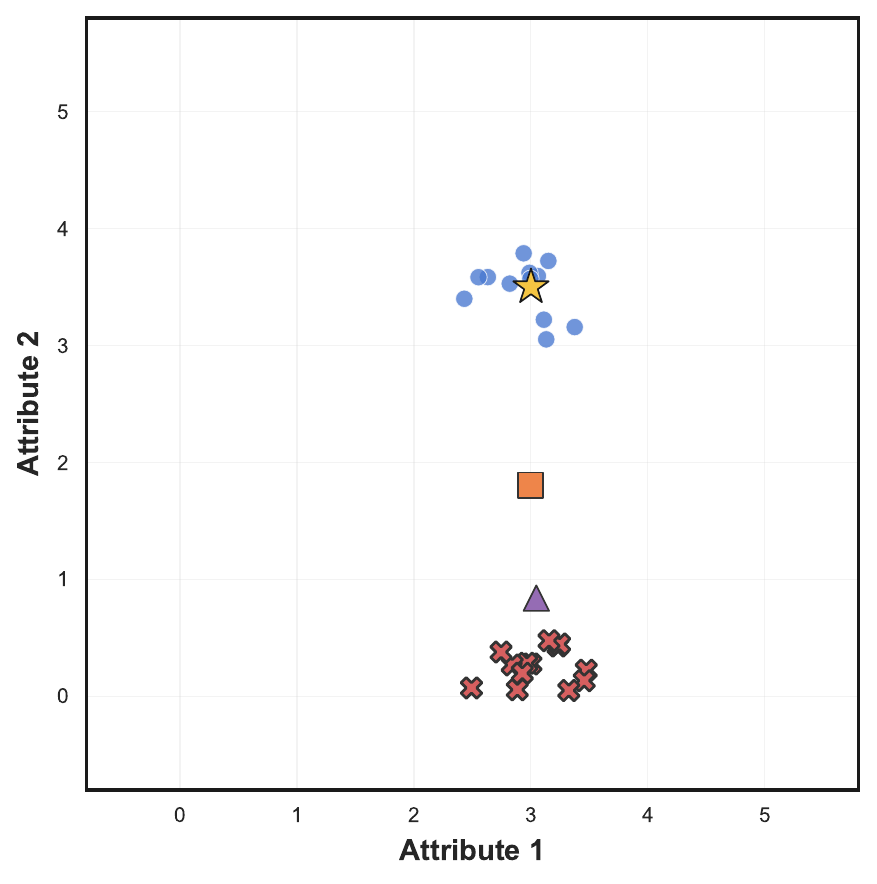}
    \caption*{$\alpha = 0.40$}
\end{subfigure}
\caption{\textbf{Biased Dimension corruption.} Corrupted judges evaluate Attribute~1 correctly but catastrophically fail on Attribute~2 (scores collapse near zero). This partial competence is challenging for coordinate-wise methods because the corruption is invisible on one axis. The geometric median, operating on joint Euclidean distances, detects the anomaly in Attribute~2 and downweights the corrupted points across both dimensions.}
\label{fig:sim_biased_dimension}
\end{figure}

\paragraph{Random hypercube corners.}
The canonical instance of the cross-dimensional class: each
corrupted score lands at a vertex of $\{0, K\}^{d}$ chosen
uniformly at random
(\Cref{fig:sim_random_corners}).
The per-coordinate marginal $\frac{1}{2}(\delta_{0}+\delta_{K})$ is
indistinguishable from plausible scoring; jointly, every corrupted
vector sits at a corner far from $\vy^{\star}$ in $\ell_{2}$.
This is the ``random vertex'' generalisation of
\emph{biased dimension} above and exactly the
\texttt{bimodal-random} class evaluated empirically in
\S\ref{sec:exp_bimodal}.
\begin{figure}[htbp]
\centering
\includegraphics[width=0.75\textwidth]{assets/simulation/scatter/shared_legend.pdf}
\vspace{-2mm}

\begin{subfigure}{0.32\textwidth}
    \centering
    \includegraphics[width=\linewidth]{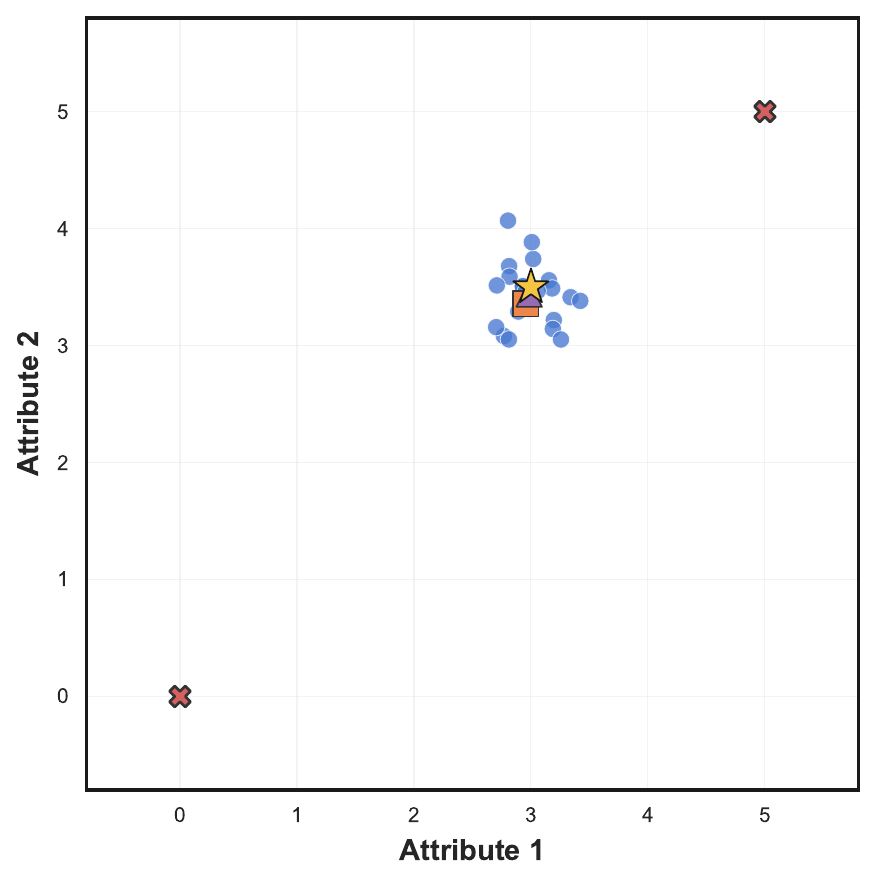}
    \caption*{$\alpha = 0.10$}
\end{subfigure}\hfill
\begin{subfigure}{0.32\textwidth}
    \centering
    \includegraphics[width=\linewidth]{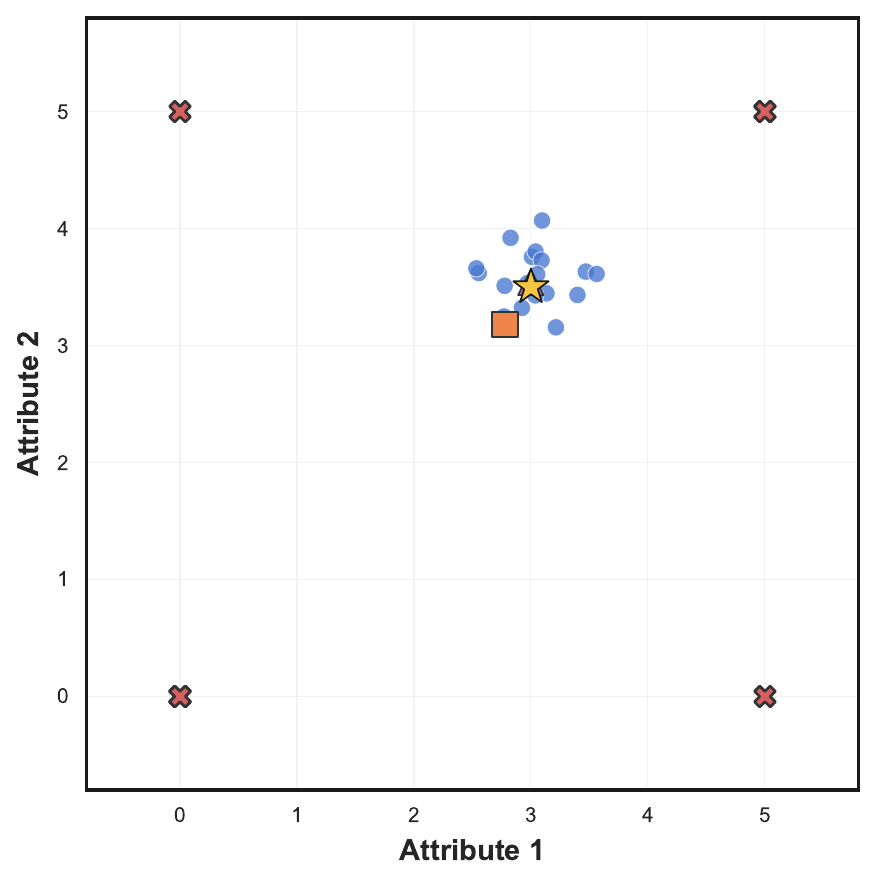}
    \caption*{$\alpha = 0.30$}
\end{subfigure}\hfill
\begin{subfigure}{0.32\textwidth}
    \centering
    \includegraphics[width=\linewidth]{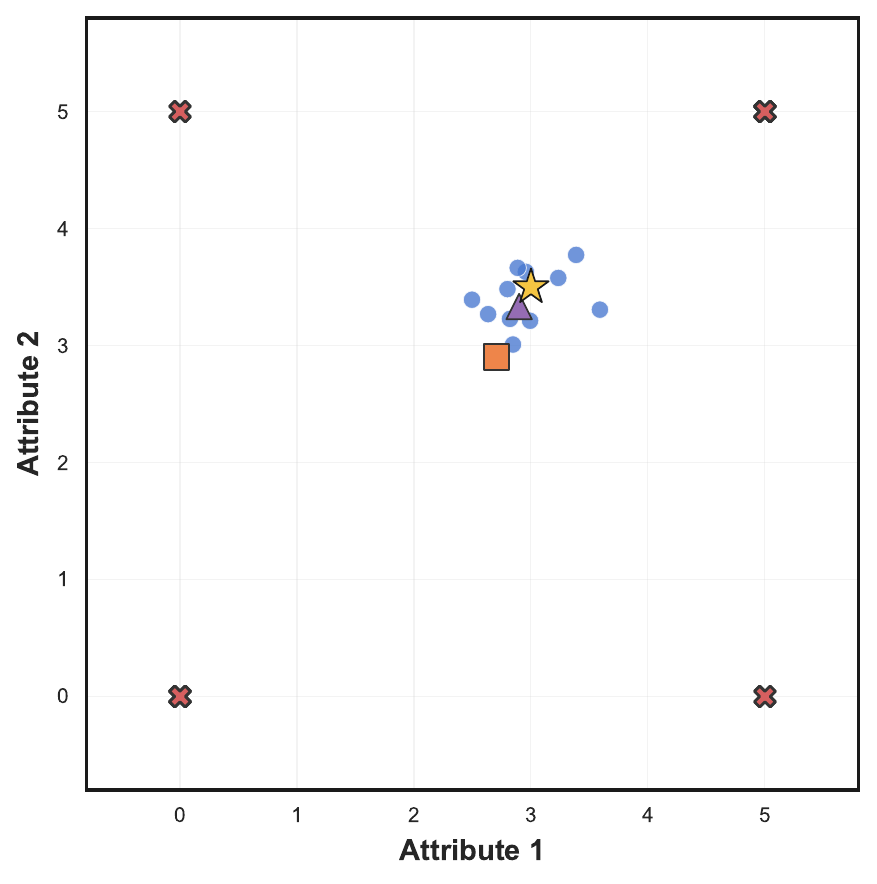}
    \caption*{$\alpha = 0.40$}
\end{subfigure}
\caption{\textbf{Random hypercube corners} (the canonical instance of the cross-dimensional class of Example~\ref{ex:cross_dim}, matching the empirical \texttt{bimodal-random} class of \S\ref{sec:exp_bimodal}). Corrupted judges output an extreme vertex of $\{0, K\}^{d}$ chosen uniformly at random; per-coordinate the corruption marginal $\frac{1}{2}(\delta_{0}+\delta_{K})$ is plausible scoring, but the joint vector lies far from $\vy^{\star}$ in $\ell_{2}$. The geometric median resists the cross-dimensional pull (it sits at $\vy^{\star}$, beneath the gold star), while the arithmetic mean drifts toward the centroid of the corrupted vertices.}
\label{fig:sim_random_corners}
\end{figure}

\paragraph{Sycophantic.}
A real-world failure mode in which corrupted judges always rate near
the top of the scale---the ``everything is great'' bias
(\Cref{fig:sim_sycophantic}).
The corrupted cloud sits in the upper-right corner of $[0, K]^{d}$;
the arithmetic mean drifts diagonally toward it while the geometric
median stays anchored to the competent majority near $\vy^{\star}$.
This complements \emph{mode collapse} (corruption at the lower-left
extremum) at the opposite extreme of the score scale.
\begin{figure}[htbp]
\centering
\includegraphics[width=0.75\textwidth]{assets/simulation/scatter/shared_legend.pdf}
\vspace{-2mm}

\begin{subfigure}{0.32\textwidth}
    \centering
    \includegraphics[width=\linewidth]{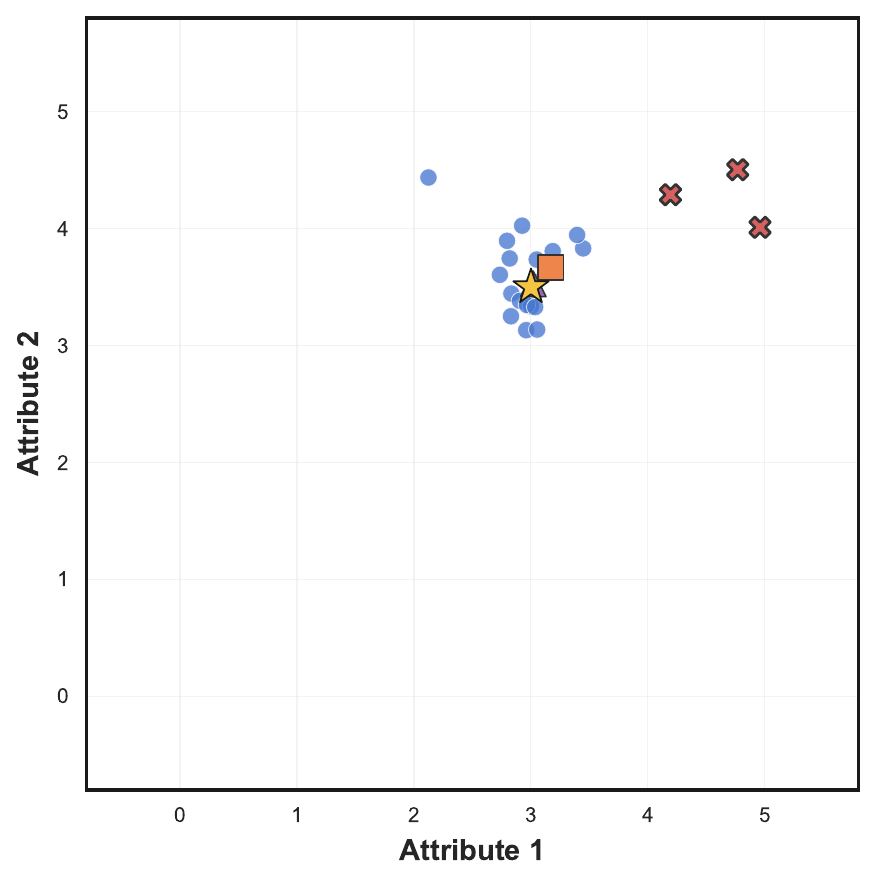}
    \caption*{$\alpha = 0.10$}
\end{subfigure}\hfill
\begin{subfigure}{0.32\textwidth}
    \centering
    \includegraphics[width=\linewidth]{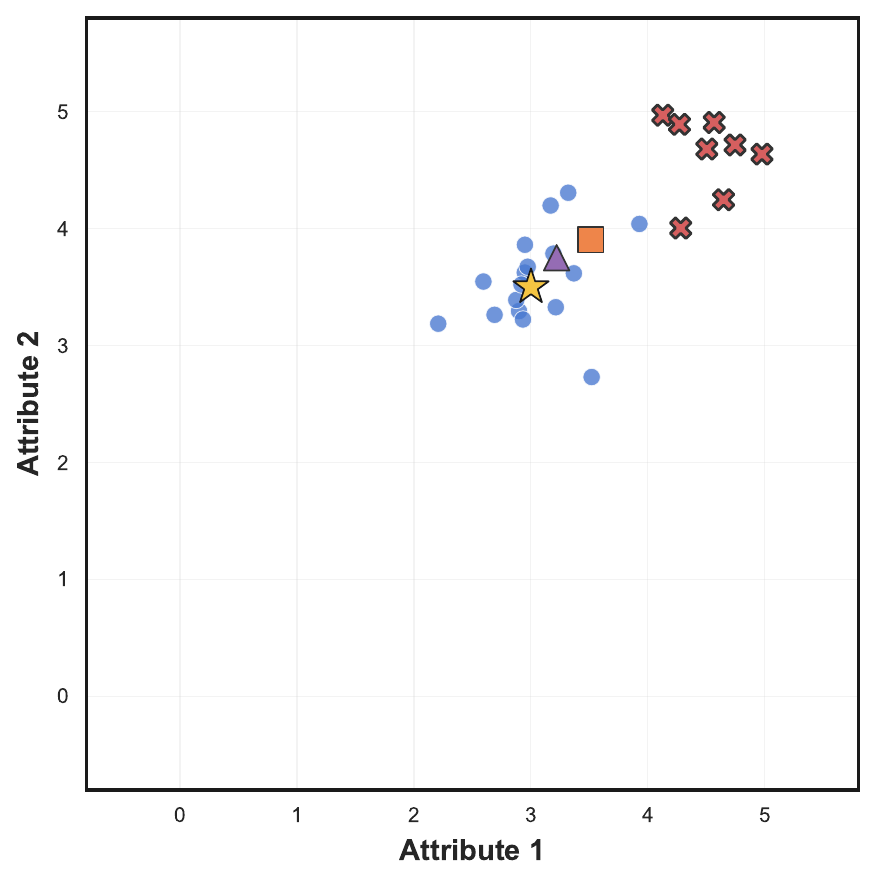}
    \caption*{$\alpha = 0.30$}
\end{subfigure}\hfill
\begin{subfigure}{0.32\textwidth}
    \centering
    \includegraphics[width=\linewidth]{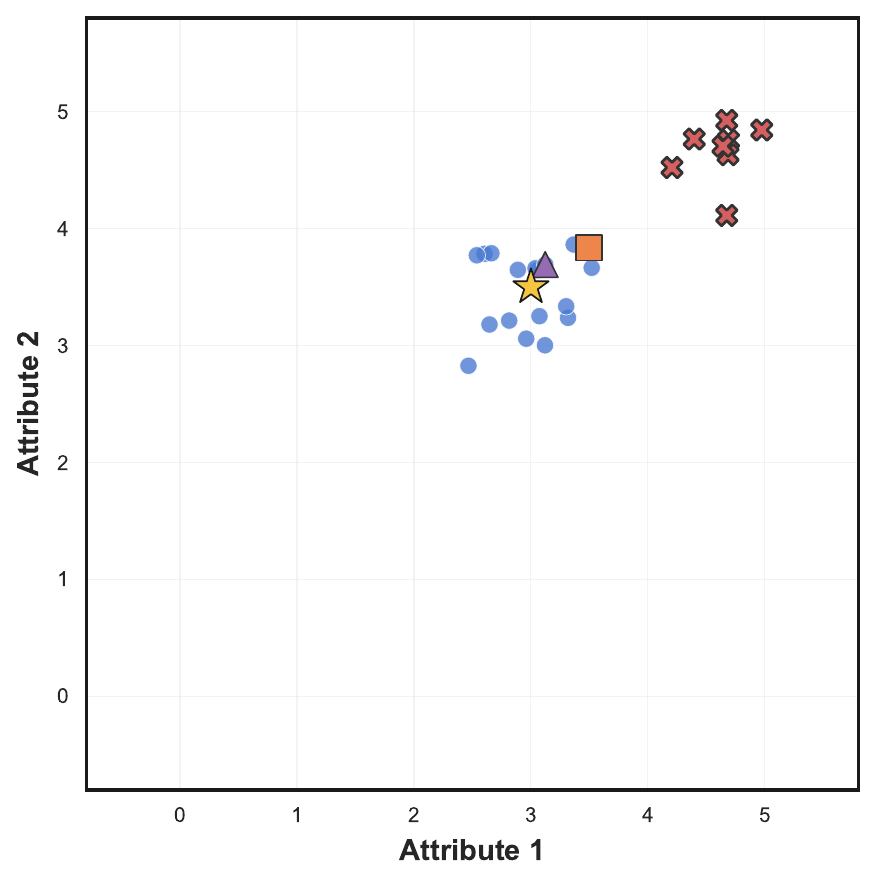}
    \caption*{$\alpha = 0.40$}
\end{subfigure}
\caption{\textbf{Sycophantic corruption} ($Q = \mathrm{Uniform}([K{-}1, K]^d)$). Corrupted judges produce scores clustered near the maximum, modeling the ``everything is great'' failure mode. The corrupted cloud sits in the upper-right corner; the mean drifts diagonally toward it while the geometric median stays anchored to the competent majority near $\mathbf{y}^\star$.}
\label{fig:sim_sycophantic}
\end{figure}

\paragraph{Summary.}
Across all three failure modes, the arithmetic mean acquires a bias
proportional to $\alpha$ and aligned with the corruption locus,
while the geometric median remains close to $\vy^{\star}$ as long as
$\alpha < 1/2$, in agreement with Theorem~\ref{thm:ropoll_bound}.
The complementary Noisy-GT control (\S\ref{sec:exp_noisy_gt})
confirms that this advantage is paid only against \emph{biased}
contamination: when the corruption is benign Gaussian noise, the
geometric median does not sacrifice accuracy.

\subsection{Per-Model and Per-Dimension Calibration Breakdowns}
\label{sec:appendix_per_model}

The figures in \S\ref{sec:benchmark} aggregate across rubric
dimensions and report the \textsc{Medium} jury's RMSE.
This subsection records the underlying per-model and per-dimension
calibration breakdowns on UltraFeedback that motivated the curated
three-judge committees of \S\ref{sec:benchmark_setup}.

\paragraph{Judge set.}
The calibration analysis in this subsection includes three
closed-API reference judges (Claude Opus, Sonnet, and Haiku
\,4.5) in addition to the $13$ open-weight judges of
\S\ref{sec:benchmark_setup}.
The closed-API judges are \emph{reference points only} — they are
not used in any \textsc{RoPoLL} committee — and are included here to
contextualise the open-weight calibration patterns.

\paragraph{Per-dimension MAE.}
\Cref{fig:uf_heatmap_mae} reports the mean absolute error for each
judge against the UltraFeedback rubric dimensions
(Helpfulness, Honesty, Instruction Following, Truthfulness).
Qwen3\,32B and Mistral-Large-3 lead with sub-$0.75$ MAE across all
four dimensions;
the Claude family lies near the bottom of the calibration ranking
despite strong ranking ability
(\Cref{fig:uf_heatmap_bias} below explains why).
\begin{figure}[htbp]
\centering
\includegraphics[width=0.75\textwidth]{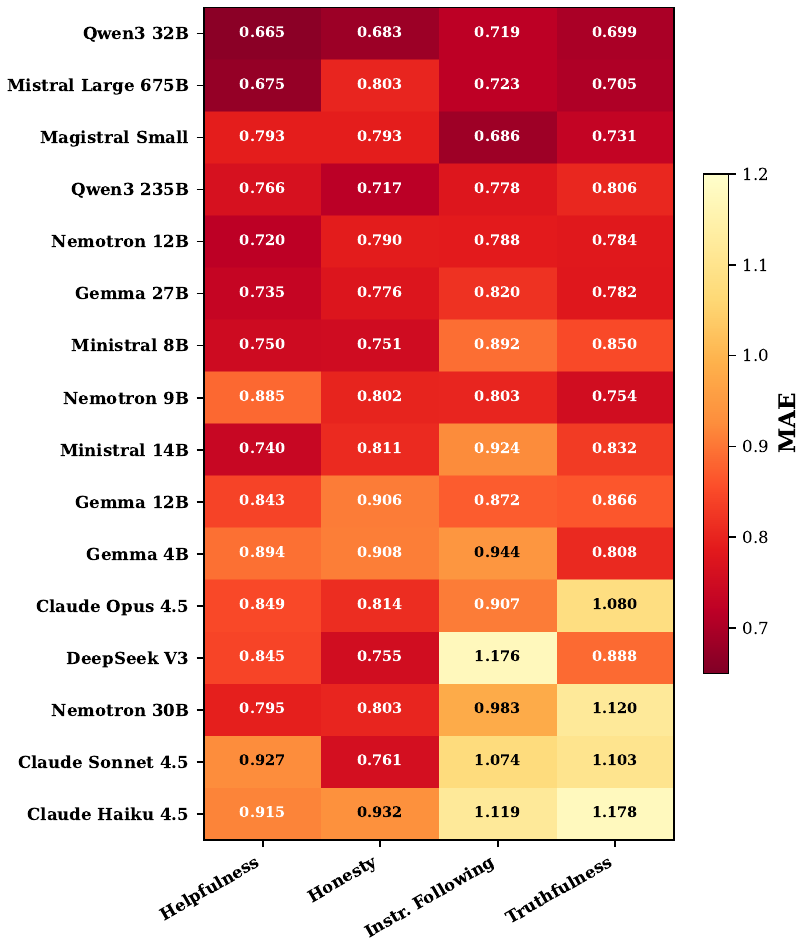}
\caption{Per-dimension MAE for each LLM judge on UltraFeedback ($n{=}1000$), sorted by lowest average error. Qwen3\,32B achieves the lowest MAE across all four dimensions. The Claude family clusters near the bottom despite strong ranking ability, with Instruction Following and Truthfulness showing the largest errors ($>1.0$) due to systematic negative bias.}
\label{fig:uf_heatmap_mae}
\end{figure}

\paragraph{Per-dimension mean bias.}
\Cref{fig:uf_heatmap_bias} reports the signed mean bias
$\E[\hat{y}_{i}^{(k)} - y^{\star,(k)}]$ for each (judge, dimension)
cell.
Two systematic patterns emerge.
The Claude family shows uniformly negative bias across all four
dimensions ($-0.5$ to $-0.8$ on Truthfulness)---a systematic
under-scoring tendency.
Smaller open-weight models (Magistral Small, Gemma\,4B,
Nemotron\,9B) show uniformly positive bias of comparable magnitude.
Qwen3\,32B and Qwen3\,235B are closest to zero across all
dimensions, consistent with their leading MAE.
The bias direction is precisely the contamination structure
Proposition~\ref{prop:mean_bias} formalises:
mixing systematically over-scoring and under-scoring judges leaves
the arithmetic mean's bias bounded only by the worst per-judge
displacement;
the geometric median is robust to such mixed-direction biases
because the joint subgradient balance does not weight per-coordinate
sign.
\begin{figure}[htbp]
\centering
\includegraphics[width=0.75\textwidth]{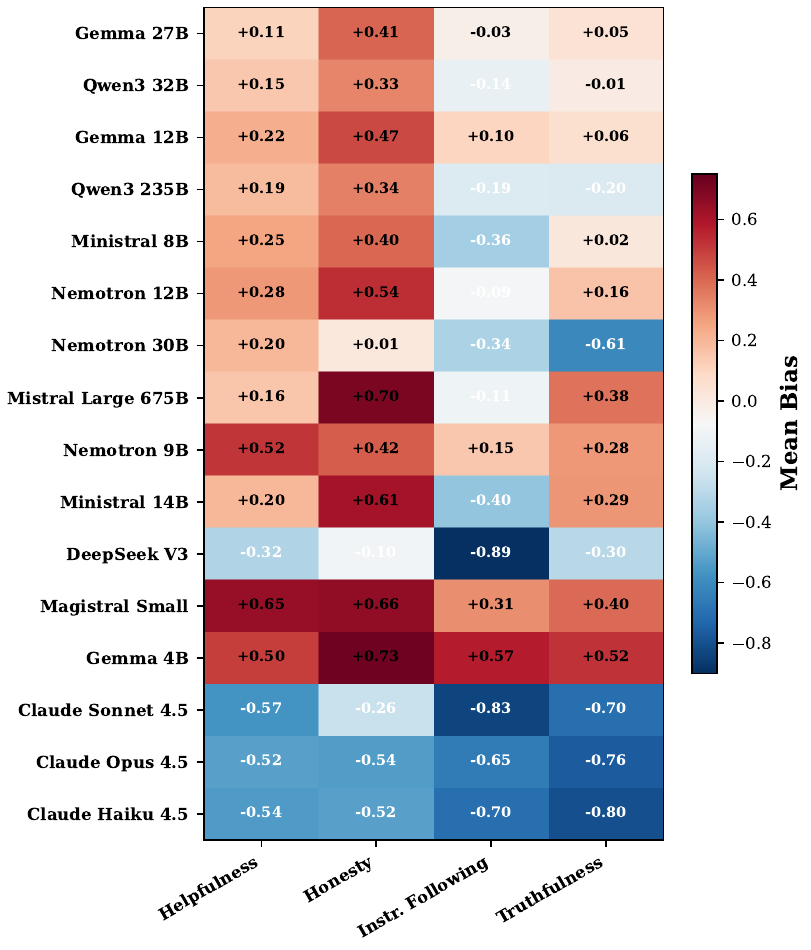}
\caption{Per-dimension mean bias for each LLM judge on UltraFeedback ($n{=}1000$), sorted by lowest absolute bias. Blue cells indicate under-scoring (negative bias); red cells indicate over-scoring (positive bias). The Claude family shows uniformly negative bias across all dimensions, while models like Magistral Small and Gemma\,4B exhibit strong positive bias. Qwen3\,32B and Qwen3\,235B are closest to zero across all dimensions.}
\label{fig:uf_heatmap_bias}
\end{figure}

\end{document}